%% file: main.tex
\documentclass{article}

\PassOptionsToPackage{numbers, compress}{natbib}

 \usepackage[preprint]{neurips_2026}

\usepackage[utf8]{inputenc} 
\usepackage[T1]{fontenc}    
\usepackage[hidelinks]{hyperref}       
\usepackage{url}            
\usepackage{booktabs}       
\usepackage{amsfonts}       
\usepackage{nicefrac}       
\usepackage{microtype}      
\usepackage{xcolor}        
\usepackage{mymacros}   
\usepackage{colortbl}       
\usepackage{multicol}
\usepackage{makecell}
\usepackage{pifont}
\usepackage{enumitem}
\usepackage{algorithm2e}

\usepackage{pgfplots}
\DeclareUnicodeCharacter{2212}{−}
\usepgfplotslibrary{groupplots,dateplot}
\usetikzlibrary{patterns,shapes.arrows}
\usepackage{tikz}
\usepackage{multirow}
\usepackage{graphicx}
\usepackage{subcaption}
\usepackage{float}
\usepackage{wrapfig}
\usepackage{cancel}
\usepackage{color,soul}
\usepackage{comment}
\usepackage{diagbox}
\usepackage{dsfont}
\usepackage{placeins}

\title{Reusing Past Samples in Proximal Policy Optimization: When and How Does It Help?}

\author{%
  Alessandro Montenegro \\
  Politecnico di Milano, Milan, Italy\\
  \texttt{alessandro.montenegro@polimi.it} \\
  \And
  Riccardo Venturelli \\
  Politecnico di Milano, Milan, Italy\\
  \texttt{riccardo.venturelli@polimi.it} \\
  \AND
  Marco Mussi \\
  Politecnico di Milano, Milan, Italy\\
  \texttt{marco.mussi@polimi.it} \\
  \And
  Matteo Papini \\
  Universit\`{a} degli Studi di Milano, Milan, Italy\\
  \texttt{matteo.papini@unimi.it} \\
  \And
  Alberto Maria Metelli \\
  Politecnico di Milano, Milan, Italy\\
  \texttt{albertomaria.metelli@polimi.it} \\
}

\begin{document}

\setlength{\abovedisplayskip}{2.5pt}
\setlength{\belowdisplayskip}{2.5pt}
\setlength{\textfloatsep}{8pt}

\maketitle

\input{content/00_abstract}
\input{content/01_intro}
\input{content/02_preliminaries}
\input{content/03_reusing}
\input{content/04_exp}
\input{content/05_conclusion}



\bibliography{biblio}
\bibliographystyle{plainnat}


\newpage

\appendix

\setlength{\abovedisplayskip}{5pt}
\setlength{\belowdisplayskip}{5pt}
\setlength{\textfloatsep}{16pt}

\input{appendix/apx_ppo_details}

\input{appendix/apx_proofs}
\input{appendix/apx_exp_details}
\input{appendix/apx_exp_additional}



\end{document}

%% file: content/00_abstract.tex
\begin{abstract}
Among on-policy \emph{deep reinforcement learning} methods, Proximal Policy Optimization (\ppo) has become the de facto standard, due to its consistently strong empirical performance across diverse application domains. However, on-policy methods are inherently sample inefficient: fresh data collected under the current policy is used for just a \emph{few updates} before being discarded. Off-policy methods avoid this inefficiency via experience replay, achieving notable sample efficiency gains, but at the cost of training instabilities or extensive tuning. This motivated the rise of \emph{hybrid strategies} that augment \ppo with off-policy data reuse. 
Existing sample-reuse variants of \ppo demonstrated improved sample efficiency over \emph{vanilla} \ppo, yet a systematic study of when reuse helps, in which scenarios, and to what extent remains missing.
In this work, we study the effectiveness of sample reuse in \ppo by instantiating two variants within a \emph{multiple importance weighting} framework. 
Both retain the core \ppo mechanics, reusing only samples from a \emph{window} of recent iterations, thereby isolating the effect of data reuse from other factors. 
The variants, termed \uppo and \bhppo, employ \emph{vanilla importance weights} or \emph{balance-heuristic-corrected ones}, respectively. For both, we derive policy improvement lower bounds providing theoretical grounding for their respective losses. 
We use them to empirically study when and how data reuse improves sample efficiency or final performance of \ppo across 
continuous control tasks.
\end{abstract}

%% file: content/01_intro.tex
\section{Introduction} \label{sec:introduction}
Within \emph{reinforcement learning}~\citep[RL,][]{sutton1998reinforcement} approaches, \emph{policy optimization}~\citep[PO,][]{deisenroth2013survey} methods have proved to be highly effective for solving complex, large-scale sequential decision-making problems, both in simulated~\citep{duan2016benchmarking} and real-world~\citep{buchler2022learning,margolis2024rapid,wang2025beamdojo} environments. Their success is attributable to several desirable properties: applicability to continuous settings~\citep{peters2006policy}, robustness to noise~\citep{gravell2020learning}, resilience to partial observability~\citep{azizzadenesheli2018policy}, and possibility of incorporating the expert domain knowledge during the policy design phase~\citep{ghavamzadeh2006bayesian}. Taken together, these properties have established PO as the dominant paradigm for continuous control, robot locomotion, and complex simulated tasks.

A key distinction among PO methods is between \emph{off-policy} and \emph{on-policy} approaches. 
Off-policy methods require an RL agent to update its \emph{target policy} using interaction data collected under a different \emph{behavioral policy}, including data gathered across multiple past policies. 
Notable representatives of this family include \sac~\citep{haarnoja2018soft}, \ddpg~\citep{lillicrap2020continuous}, and \td~\citep{fujimoto2018addressing}. 
These methods are known for their sample efficiency, as they can store and reuse past experience via a replay buffer, and often achieve strong final performance. However, they tend to be computationally intensive, frequently requiring careful implementation choices and extensive tuning.
On-policy methods, by contrast, restrict updates to fresh interaction data collected under the current target policy, discarding those samples after one or a few gradient steps. This family includes \texttt{REINFORCE}~\citep{williams1992simple}, \texttt{GPOMDP}~\citep{baxter2001infinite}, \trpo~\citep{schulman2015trust}, \ppo~\citep{schulman2017proximal}, \pois~\citep{metelli2018policy}, and \grpo~\citep{shao2024deepseekmath}. 
By construction, on-policy methods offer stable training dynamics and admit meaningful theoretical guarantees, ranging from convergence results~\citep{yuan2022general,fatkhullin2023stochastic} to policy improvement bounds~\citep{kakade2002approximately,pirotta2013safe,achiam2017constrained}. 
Their main drawback, however, is poor sample efficiency: each collected trajectory contributes to only a single round of updates before being discarded.\footnote{The distinction between on-policy and off-policy methods is a spectrum rather than strictly binary. Methods such as \ppo and \pois perform multiple gradient steps on a fixed batch, introducing a mild distributional shift between the behavioral and target policies within each update cycle. Nevertheless, these methods are conventionally classified as on-policy, since the collected data is discarded at the end of each iteration and never reused across policy updates.}

Given the sample efficiency of off-policy methods and the stability and theoretical foundations of on-policy ones, an active research direction pursues methods that exhibit the best of both worlds. Most efforts have focused on augmenting on-policy methods with off-policy data~\citep{gu2017interpolated,papini2018stochastic,fakoor2019p3o}, enabling the reuse of samples collected under past policies. 
This is commonly achieved via \emph{importance weighting}~\citep{owen2000safe}, a technique that corrects for the distributional mismatch between the behavioral and target policies by reweighting samples by the ratio of their likelihoods under the two policies. 
Since these weights can exhibit high variance when the two policies are dissimilar, some corrections have been proposed to make their use more reliable, including the \emph{balance heuristic}~\citep[BH,][]{veach1995optimally,metelli2020importance} estimator.
In the context of policy gradient methods~\citep{peters2006policy}, substantial work has analyzed whether reusing either past gradients~\citep{papini2018stochastic,xu2019sample,yuan2020stochastic,paczolay2024sample} or past trajectories~\citep{papini2024policy,lin2025reusing,montenegro2025reusing} can yield faster convergence rates, with affirmative results in both cases, that, however, remain tied to theoretical settings. 

These findings have motivated the investigation of similar gains in practical deep on-policy algorithms~\citep{wang2016sample,gu2016q,gu2017interpolated,fakoor2019p3o}. 
Much of this effort has concentrated on \ppo, which has become the de facto standard deep on-policy method thanks to its strong empirical performance across application domains ranging from robot locomotion~\citep{lee2020learning,buchler2022learning,margolis2024rapid,wang2025beamdojo} to LLM fine-tuning~\citep{wang2024comprehensive,xu2024dpo}. 
This success has, in turn, motivated deep investigation of \ppo from multiple standpoints, including 
theoretical convergence guarantees~\citep{jin2023stationary,doering2026approximate}, implementation best practices~\citep{engstrom2019implementation,andrychowicz2021matters,kallel2025revisiting}, and algorithmic variants~\citep{wang2019trust,wang2020truly,sun2022you,cheng2024proximal}.
Among the latter, several works have specifically targeted the incorporation of off-policy data into \ppo. \citet{queeney2021generalized} derive an off-policy policy improvement lower bound from the \emph{performance difference lemma}~\citep{kakade2002approximately}, proposing \texttt{GePPO}, a \ppo variant that reuses data from a window of recent iterations. 
\texttt{GePPO} requires re-evaluating advantages on past samples via V-Trace~\citep{espeholt2018impala}, incurring additional computation, and employs a modified clipping mechanism together with an adaptive learning rate based on the total variation between the current and behavioral policies.
\citet{zhao2023policy} propose \texttt{SARD-PPO}, a sample reuse \ppo variant that adaptively discards samples and re-evaluates past data under the current critic, employing a modified dual clipping mechanism.
\citet{liang2021ptrppo} incorporate \emph{prioritized experience replay}~\citep{schaul2016prioritized} into \ppo, proposing \texttt{PTR-PPO} and studying different priority assignment strategies. 
\citet{meng2023off} 
introduce \texttt{Off-Policy PPO}, which also requires re-evaluating past trajectories under the current critic to obtain fresh advantage estimates, alongside a heuristic clipping mechanism. 
Recently, \citet{liu2026enhancing} proposed \texttt{HP3O}, a sample reuse \ppo variant maintaining a FIFO trajectory replay buffer, re-evaluates past samples under the current critic, and retains the best-performing trajectory in each optimization batch.
Finally, \citet{lin2025reusing} provide an initial experimental exploration of \texttt{RPPO}, a \ppo variant that reuses samples from recent policies without re-evaluating advantages under the new critic.
All these methods confirm both the ongoing research interest in \ppo and the practical benefits of past-data reuse. However, most approaches require re-evaluation of stored samples under the current critic, rely on heuristics or modified clipping mechanisms, and are presented as standalone methods with tuned configurations rather than as controlled studies of the reuse mechanism itself. To the best of our knowledge, a systematic investigation of when, how, and why data reuse in \ppo is beneficial remains unaddressed.


\textbf{Original Contribution.}~~In this paper, we aim to conduct a controlled study of the effects of sample reuse in \ppo. To preserve the core functioning of \ppo and isolate the effect of data reuse from other factors, we consider the incorporation of past collected data within the standard \emph{multiple importance weighting}~\citep[MIW,][]{owen2013monte,metelli2020importance} framework. This amounts to the simplest possible form of sample reuse for \ppo: augmenting its on-policy dataset by retaining data collected over a window of the $\omega$ most recent iterations, optionally correcting the vanilla \emph{importance weights}~\citep[IWs,][]{owen2013monte}.
Specifically, our main contributions can be summarized as follows:
\begin{itemize}[noitemsep,leftmargin=*,topsep=-6pt]
    \item \textbf{Uniform Sample Reuse.} We study the straightforward MIW instantiation of \ppo, which we term \uppo. Similarly to~\citep{lin2025reusing}, it uniformly weights all samples in the composite dataset using vanilla IWs between the target and behavioral policies, and employs advantage estimates computed under the critic available at collection time. To ground this approach, we derive a policy improvement lower bound that accounts for stale advantages. Beyond validating the \uppo loss, the bound reveals a penalty term independent of the optimization variable, which therefore cannot be controlled by actor updates, suggesting that the policies within the considered window should be close (Section~\ref{sec:sample-reuse-u}).
    \item \textbf{BH-Corrected Sample Reuse.} Motivated by the usual inefficiency of uniform weighting in MIW, we additionally study a \ppo instantiation employing BH-corrected IWs, which we term \bhppo. This consists of weighting each sample via the ratio between the target policy and the uniform mixture of behavioral policies in the window. \bhppo inherits all properties of \uppo and is similarly grounded by a policy improvement lower bound, whose uncontrollable penalty terms are non-decreasing in $\omega$, reinforcing the recommendation for keeping the policies within the window close (Section~\ref{sec:sample-reuse-bh}).
    \item \textbf{Systematic Empirical Study.} We compare \uppo and \bhppo with standard \ppo on MuJoCo~\citep{todorov2012mujoco} tasks under controlled settings that vary how updates and data collection are matched between methods. We study whether reuse helps convergence speed, final performance, or both, whether the benefit 
    grows indefinitely with larger windows, and when excessive reuse becomes harmful. Finally, we study the differences between the two considered weightings (Section~\ref{sec:experiments}).
\end{itemize}
Proofs, experimental details, and additional experiments are in 
Appendices~\ref{apx:proofs}, \ref{apx:exp-details}, 
and~\ref{apx:exp-additional}, respectively.

%% file: content/02_preliminaries.tex
\section{Preliminaries} \label{sec:preliminaries}

\textbf{Notation.}~~For $n,m\in\mathbb{N}$ with $n \ge m$, we denote $\dsb{m,n} \coloneqq \{ m, m+1,  \, \ldots , \, n \}$ and $\dsb{n} \coloneqq \dsb{1,n}$. 
We denote by $\Delta({\Xs})$ the set of probability measures over the measurable set $\Xs$. For $P \in \Delta(\mathcal{X})$, we denote by $p$ its density function \wrt a reference measure that we assume to exist whenever needed. 
The \emph{total variation} (TV) distance is defined as $D_{\text{TV}}(P, Q) \coloneqq \frac{1}{2} \int_{\mathcal{X}} |p(x) - q(x)| \de x$.

\textbf{Markov Decision Processes.}~~A Markov Decision Process~\citep[MDP,][]{puterman1990markov} is represented by $\mathcal{M} \coloneqq \left( \mathcal{S}, \mathcal{A}, p, r, d_{0}, \gamma, T \right)$, where $\mathcal{S} \subseteq \Reals^{\ds}$ and $\mathcal{A} \subseteq \Reals^{\da}$ are the measurable state and action spaces; $p: \mathcal{S} \times \mathcal{A} \xrightarrow[]{} \Delta(\mathcal{S})$ is the transition model, where $p(\vs' | \vs, \va)$ is the probability density of landing in state $\vs'\in \mathcal{S}$ by playing action $\va\in \mathcal{A}$ in state $\vs\in \mathcal{S}$; $r: \mathcal{S} \times \mathcal{A} \xrightarrow[]{} [-R_{\max}, R_{\max}]$ is the reward function, where $r(\vs, \va)$ is the reward the agent gets when playing action $\va$ in state $\vs$; $d_{0} \in \Delta(\mathcal{S})$ is the initial-state distribution; $\gamma \in [0, 1]$ is the discount factor. A trajectory $\tau = \left( \vs_{\tau,1}, \va_{\tau,1},\dots, \vs_{\tau,T}, \va_{\tau,T} \right)$ of length $T \in \Nat \cup \{+\infty\}$ is a sequence of $T$ state-action pairs. We denote by $\cT$ the set of all trajectories. The \emph{discounted return} of a trajectory $\tau \in \cT$ is given by $R(\tau) \coloneqq \sum_{t=1}^{T} \gamma^{t-1} r(\vs_{\tau,t}, \va_{\tau,t})$.\footnote{We allow $\gamma=1$ only when $T<+\infty$.}

\textbf{Policy Optimization.}~~We consider \emph{parametric stochastic policies} $\pi_{\vtheta}: \cS \to \Delta(\cA)$ with parameters $\vtheta \in \Theta \subseteq \Reals^{\dt}$. The policy $\pi_{\vtheta}$ is used to sample actions at each interaction instant $t$ as $\va_{t} \sim \pi_{\vtheta}(\cdot \mid \vs_{t})$, where $\vs_{t}$ is the state observed at time $t$. The performance of a policy $\pi_{\vtheta}$ is assessed via the \emph{expected return} $J: \Theta \to \Reals$, defined as:
\begin{align*}
    J(\vtheta) \coloneqq \E_{\substack{\vs \sim d_{\vtheta} \\ \va \sim 
    \pi_{\vtheta}(\cdot \mid \vs)}} \left[ r(\vs,\va) \right],
\end{align*}
where, with an abuse of notation, the expectation is taken over $d_{\vtheta}(\vs) \coloneqq \sum_{t=1}^{T} \gamma^{t-1} \Pr(\vs_{t} = \vs \mid \vs_{1} \sim d_{0}, \pi_{\vtheta})$, the $\gamma$-discounted state occupancy under $\pi_{\vtheta}$, which is not a proper probability distribution.
Moreover, the $\gamma$-discounted state-action occupancy under $\pi_{\vtheta}$ is defined as $\mu_{\vtheta}(\vs,\va) \coloneqq d_{\vtheta}(\vs) \pi_{\vtheta}(\va \mid \vs)$, for every $\vs \in \cS$ and $\va \in \cA$. Two additional performance indices are the \emph{value function} $V^{\pi_{\vtheta}}: \cS \to \Reals$ and the \emph{action-value function} $Q^{\pi_{\vtheta}}: \cS \times \cA \to \Reals$, defined respectively as: $V^{\pi_{\vtheta}}(\vs) \coloneqq \E [\sum_{t=1}^{T} \gamma^{t-1} r(\vs_{t},\va_{t}) \mid \vs_{1} = \vs, \pi_{\vtheta}]$ and $Q^{\pi_{\vtheta}}(\vs,\va) \coloneqq \E [\sum_{t=1}^{T} \gamma^{t-1} r(\vs_{t},\va_{t}) \mid \vs_{1} = \vs, \va_{1} = \va, \pi_{\vtheta}].$ The goal of PO is to find $\vtheta^{\star} \in \argmax_{\vtheta \in \Theta} J(\vtheta)$. 


\textbf{Trust Region Methods and Proximal Policy Optimization.}~~To learn $\vtheta^{\star}$ in an on-policy fashion, one can either follow the \emph{policy gradient theorem}~\citep{sutton1999policy} 
or one can build surrogate losses allowing for using multiple times the same collected data while constraining the target policy to remain inside a \emph{trust region}, meaning that the target policy is kept near the behavioral one, ensuring the sampled data remain trustworthy. This approach is grounded in the following \emph{policy improvement} lower bound~\citep{achiam2017constrained}:
\begin{align*}
    J(\vtheta') - J(\vtheta) \ge \E_{\substack{\vs \sim d_{\vtheta} \\ \va \sim \pi_{\vtheta}(\cdot \mid \vs)}}\left[ \frac{\pi_{\vtheta'}(\va \mid \vs)}{\pi_{\vtheta}(\va \mid \vs)}A^{\pi_{\vtheta}}(\vs,\va) \right] - \frac{2 \gamma C(\vtheta',\vtheta)}{(1-\gamma)^{2}} \E_{\vs \sim d_{\vtheta}}\left[ D_{\text{TV}}(\pi_{\vtheta'}(\cdot \mid \vs), \pi_{\vtheta}(\cdot \mid \vs)) \right],
\end{align*}
where $A^{\pi_{\vtheta}}(\vs,\va) \coloneqq Q^{\pi_{\vtheta}}(\vs,\va) - V^{\pi_{\vtheta}}(\vs)$ is the \emph{advantage function}, $\frac{\pi_{\vtheta'}(\va \mid \vs)}{\pi_{\vtheta}(\va \mid \vs)}$ is the \emph{importance weight}~\citep[IW,][]{owen2000safe}, and $C(\vtheta',\vtheta) \coloneqq \sup_{\vs \in \cS} \left| \E_{\va \sim \pi_{\vtheta'}(\cdot \mid \vs)}[A^{\pi_{\vtheta}}(\vs,\va)] \right|$. 
The bound reveals that the closer $\pi_{\vtheta'}$ and $\pi_{\vtheta}$ are in TV, the tighter the bound.
\trpo~\citep{schulman2015trust} operationalizes this by maximizing the importance-weighted surrogate subject to a KL constraint, at the cost of requiring second-order information in practice. \ppo~\citep{schulman2017proximal} replaces this constraint with a clipping heuristic, optimizing:
\begin{align*}
    L^{\texttt{PPO}}(\vtheta') \coloneqq \E_{\substack{\vs \sim d_{\vtheta} \\ \va \sim \pi_{\vtheta}(\cdot \mid \vs)}}\left[ \min\left\{ \frac{\pi_{\vtheta'}(\va \mid \vs)}{\pi_{\vtheta}(\va \mid \vs)} A^{\pi_{\vtheta}}(\vs,\va), \clip \left( \frac{\pi_{\vtheta'}(\va \mid \vs)}{\pi_{\vtheta}(\va \mid \vs)}, 1-\epsilon, 1+\epsilon \right) A^{\pi_{\vtheta}}(\vs,\va) \right\} \right],
\end{align*}
where $\clip(a,b,c) = \min \left\{ \max \left\{ a,b \right\}, c \right\}$. In practice, this loss is approximated using a parameterized critic $V_{\vw}(\vs) \approx V^{\pi_{\vtheta}}(\vs)$ with parameters $\vw \in \cW \subseteq \mathbb{R}^{d_{\cW}}$. The policy $\pi_{\vtheta}$ is used to collect a dataset of transitions $\cD_{\vtheta} = \{(\vs_{t}, \va_{t}, \pi_{\vtheta}(\va_{t} \mid \vs_{t}), A_{t}, R_{t})_{t=1}^{H}\}$, where $H$ is the number of samples, $A_{t} \approx A^{\pi_{\vtheta}}(\vs_{t},\va_{t})$ is estimated via \emph{generalized advantage estimation}~\citep[GAE,][]{schulman2015high}, and $R_{t} \approx V^{\pi_{\vtheta}}(\vs_{t})$. The surrogate is then approximated as $\widehat{L}^{\texttt{PPO}}(\vtheta', \cD)$ and the critic is fitted via mean squared error minimization over $R_{t}$. Further details on the practical functioning of \ppo are deferred to Appendix~\ref{apx:ppo-details}.


\textbf{Multiple Importance Weighting.}~~Let $x \in \cX$ be a random variable and let $f: \cX \to \mathbb{R}$ be a function. Consider the goal of estimating the expected value of $f$ under a target distribution $P$, but only having samples coming from $\omega$ behavioral distributions $\{Q_{i}\}_{i=1}^{\omega}$. Specifically, for each $i \in \dsb{\omega}$ consider to have access to a dataset $\cD_{i} = \{ x_{i,j} \}_{j=1}^{N}$ of $N$ independent samples drawn from $Q_{i}$. Assuming for every $i \in \dsb{\omega}$ that $P \ll Q_{i}$ ($P$ is absolutely continuous w.r.t. $Q_{i}$), the \emph{multiple importance weighting} (MIW) estimator of the expected value of $f$ can be written as~\citep{owen2013monte}:
\begin{align*}
    \widehat{\mu}_{P/\{Q_{i}\}_{i=1}^{\omega}} = \sum_{i=1}^{\omega} \frac{1}{N} \sum_{j=1}^{N} \alpha_{i}(x_{i,j})\frac{p(x_{i,j})}{q_{i}(x_{i,j})} f(x_{i,j}),
\end{align*}
where the coefficients $\alpha_{i}$ are a partition of the unit, i.e., for every $i \in \dsb{\omega}$ and $x \in \cX$ we have that $\alpha_{i}(x) \in \mathbb{R}_{\ge 0}$ and $\sum_{i=1}^{\omega} \alpha_{i}(x) = 1$. Notice that $\widehat{\mu}$ is unbiased. Among the choices for the $\alpha_{i}$ coefficients the most simple but inefficient choice is the uniform weighting, meaning for each $i \in \dsb{\omega}$ and $x \in \cX$ we select $\alpha_{i}(x)=1/\omega$, thus giving equal weight to all the samples~\citep{metelli2020importance}. Among the various choices for such coefficients, the most studied one is the \emph{balance heuristic}~\citep[BH,][]{veach1995optimally}, provably nearly optimal in terms of variance within the class of MIW estimators. The BH prescribes to select the weights as: $\alpha^{\mathrm{BH}}_{i}(x) = \frac{q_{i}(x)}{\sum_{l=1}^{\omega}q_{l}(x)}$, leading to the following unbiased estimator:
\begin{align*}
    \widehat{\mu}^{\mathrm{BH}}_{P/\{Q_{i}\}_{i=1}^{\omega}} = \frac{1}{\omega N} \sum_{i=1}^{\omega} \sum_{j=1}^{N}\frac{p(x_{i,j})}{\frac{1}{\omega} \sum_{l=1}^{\omega} q_{l}(x_{i,j})} f(x_{i,j}),
\end{align*}
that, whenever $P \in \{Q_{i}\}_{i=1}^{\omega}$, makes the BH-corrected ratio bounded by $\omega$. 


%% file: content/03_reusing.tex
\section{Sample Reuse in \ppo} \label{sec:sample-reuse}
In this section, we first revisit the functioning of \ppo and the data reuse mechanism already implicit in its design. We then instantiate two \ppo variants within a MIW framework~\citep{owen2013monte,metelli2020importance} that incorporate past collected data from a window of at most $\omega$ recent iterations, without re-evaluating the advantage signals under the current critic.
The first, \uppo, employs uniform weighting across the window; the second, \bhppo, applies a BH~\citep{veach1995optimally} correction to the IWs. For both, we derive a policy improvement lower bound that accounts for stale advantages: beyond providing theoretical grounding for the losses optimized by the reuse variants, the bounds reveal penalty terms independent of the optimization variable, suggesting that the policies in the reuse window should remain close to each other, while the target policy should remain close to the behavioral ones.

\textbf{Native \ppo's Sample Reuse.}~~\ppo already incorporates a form of sample reuse. However, this is limited in two ways: $(i)$ it only considers data collected under the \emph{same} policy, and $(ii)$ it only considers data from the \emph{most recent} behavior policy. Specifically, at iteration $k \in \Nat$, policy $\pi_{\vtheta_{k}}$ and critic $V_{\vw_{k}}$ are used to build dataset $\cD_{k}$ by interacting with the environment for $H$ steps, so that $|\cD_{k}| \eqqcolon D$. For $\widetilde{K}$ epochs, $\cD_{k}$ is shuffled and divided into $N_{\cB}$ mini-batches of size $B$, each used to perform one update of both actor and critic. After performing $\widetilde{K} N_{\cB}$ updates, $\cD_{k}$ is discarded.

\textbf{Cross-Iteration Sample Reuse Framework.}~~As highlighted above, one of the causes of \ppo's sample inefficiency lies in the fact that the dataset $\cD_{k}$ is discarded after a few updates. Instead, past samples could be retained and reused in future iterations, increasing sample efficiency as done in related work~\citep{queeney2021generalized,meng2023off,zhao2023policy,liu2026enhancing}. 
Specifically, suppose we have access to a window of $\omega_{k} \coloneqq \min\{k, \omega\}$ recent iterations, where $\omega$ is the maximum window size. When updating the actor and critic, we consider the composite dataset $\cD(\omega_{k}) \coloneqq \bigcup_{i=0}^{\omega_{k}-1} \cD_{k-i}$, which we shuffle and from which we sample mini-batches, equally representing all the iterations in the window. 
We assume each policy performs $H$ interaction steps with the environment, so that all the datasets will have the same cardinality, i.e., $|\cD_{k-i}| = D$ for every $i \in \dsb{0,\omega_{k}-1}$ and $|\cD(\omega_{k})| = \sum_{i=0}^{\omega_{k}-1} |\cD_{k-i}| = \omega_{k} D$.
Moreover, for each $i \in \dsb{0,\omega_{k}-1}$ let $\cD_{k-i} = \{(\vs_{i,j}, \va_{i,j}, \pi_{\vtheta_{k-i}} (\va_{i,j} \mid \vs_{i,j}), A_{i,j}, R_{i,j})\}_{j=1}^{D}$, where $A_{i,j}$ and $R_{i,j}$ are obtained with the critic available at iteration $k-i$.
The key difference with respect to the native \ppo reuse is that the data buffer now contains samples from \emph{multiple} behavioral policies rather than from a single, most recent one. 
Next, we discuss two MIW instantiations of \ppo for exploiting $\cD(\omega_{k})$.

\subsection{Cross-Iteration Uniform Sample Reuse in \ppo} \label{sec:sample-reuse-u}

Within the described \emph{reuse framework}, the simplest form of reuse consists of \emph{uniformly weighting} all the samples in $\cD(\omega_{k})$. While uniform weighting is known to be inefficient in MIW frameworks~\citep{owen2013monte,metelli2020importance}, it may nevertheless be meaningful in the \ppo setting, where the clipping mechanism already limits the contribution of samples far from the target policy. A first empirical hint of this approach is discussed in~\citep{lin2025reusing}. We refer to the \ppo variant employing this weighting scheme as \uppo, whose empirical loss is:
\begin{equation*}
    \resizebox{\linewidth}{!}{$\displaystyle \widehat{L}^{\texttt{U}}(\vtheta,\cD(\omega_{k})) = \frac{1}{\omega_{k}D} \sum_{i=0}^{\omega_{k}-1} \sum_{j=1}^{D} \min\left\{ \frac{\pi_{\vtheta}(\va_{i,j} \mid \vs_{i,j})}{\pi_{\vtheta_{k-i}}(\va_{i,j} \mid \vs_{i,j})} A_{i,j},\clip \left( \frac{\pi_{\vtheta}(\va_{i,j} \mid \vs_{i,j})}{\pi_{\vtheta_{k-i}}(\va_{i,j} \mid \vs_{i,j})}, 1-\epsilon, 1+\epsilon \right) A_{i,j} \right\},$}
\end{equation*}
where $\vtheta$ is the target actor parameter vector.
Some comments are in order. 
First, in line with our goal of preserving the core \ppo functioning, we deliberately use \emph{stale advantages}, i.e., advantage estimates computed at data collection time under the critic available at that iteration, rather than re-evaluating past samples under the most recent critic $V_{\vw_{k}}$~\citep[e.g.,][]{queeney2021generalized,meng2023off,zhao2023policy}, a practice typically carried out via V-Trace~\citep{espeholt2018impala} and incurring additional computation.
Second, without re-evaluating past data under the current critic, updating the critic on stale returns from distant iterations would mean to update the critic on increasingly stale values, especially for large $\omega$ and $\widetilde{K}$ values. We therefore restrict critic updates to the most recent dataset $\cD_{k}$ only, using past samples exclusively to produce additional actor updates.
Third, this loss can be easily employed in mini-batch stochastic gradient ascent as soon as each mini-batch equally represents the full dataset $\cD(\omega_{k})$, meaning that each window of reuse should have the same amount of data inside each mini-batch. 

Now, we show a policy improvement lower bound providing the theoretical grounding for $\widehat{L}^{\texttt{U}}$.
\begin{restatable}{thr}{uniformBound} \label{thr:uniform-bound}
    Let $\omega_{k} \in \mathbb{N}_{\ge 1}$. Having access to $\{\pi_{\vtheta_{k-i}}\}_{i=0}^{\omega_{k}-1}$, for any target $\pi_{\vtheta}$, the following holds:
    \begin{align*}
        J(\vtheta) - J(\vtheta_{k}) &\ge \frac{1}{\omega_{k}} \sum_{i=0}^{\omega_{k}-1} \E_{\substack{\vs \sim d_{\vtheta_{k-i}} \\ \va \sim \pi_{\vtheta_{k-i}}(\cdot \mid \vs)}}\left[ \frac{\pi_{\vtheta}(\va \mid \vs)}{\pi_{\vtheta_{k-i}}(\va \mid \vs)} A^{\pi_{\vtheta_{k-i}}}(\vs,\va) \right] \\
        &\quad - \frac{2 \gamma C(\vtheta,\vtheta_{k})}{(1-\gamma)^{2} \omega_{k}} \sum_{i=0}^{\omega_{k}-1} \E_{\vs \sim d_{\vtheta_{k-i}}} \left[ D_{\mathrm{TV}}\left( \pi_{\vtheta}(\cdot \mid \vs), \pi_{\vtheta_{k-i}}(\cdot \mid \vs)\right) \right] \\
        &\quad - \frac{4 R_{\max} (1+\gamma)}{(1-\gamma) \omega_{k}} \sum_{i=0}^{\omega_{k}-1} \sup_{\vs \in \cS} D_{\mathrm{TV}}\left( \pi_{\vtheta_{k}}(\cdot \mid \vs), \pi_{\vtheta_{k-i}}(\cdot \mid \vs) \right).
    \end{align*}
\end{restatable}
Several remarks are in order. First, unlike prior work~\citep{queeney2021generalized,meng2023off,zhao2023policy}, this bound accounts for \emph{stale advantages}: it features $A^{\pi_{\vtheta_{k-i}}}$ rather than $A^{\pi_{\vtheta_{k}}}$, and therefore does not require access to advantage values under the current policy $\pi_{\vtheta_{k}}$ for samples collected under past behavioral policies $\pi_{\vtheta_{k-i}}$. 
Second, compared to the standard policy improvement lower bound~\citep{pirotta2013safe,achiam2017constrained}, recovered when $\omega = 1$, the factor $C(\vtheta,\vtheta_{k})$ is unchanged, while two additional averaged TV penalty terms appear. The first, involving $D_{\mathrm{TV}}(\pi_{\vtheta}(\cdot | \vs), \pi_{\vtheta_{k-i}}(\cdot | \vs))$, already appears in~\citep{queeney2021generalized,zhao2023policy,liu2026enhancing} and accounts for sampling $(\vs,\va)$ from $\mu_{\vtheta_{k-i}}$. The second, involving $D_{\mathrm{TV}}(\pi_{\vtheta_{k}}(\cdot | \vs), \pi_{\vtheta_{k-i}}(\cdot | \vs))$, is the price paid for considering stale advantages. Crucially, the latter does not depend on the target policy $\vtheta$, so it vanishes when taking the gradient of the surrogate loss, leaving the optimization unaffected and thus grounding $\widehat{L}^{\texttt{U}}$. However, since it cannot be controlled by the actor update, it suggests that the policies belonging to the reuse window $\omega$ should remain close to each other.
Additionally, the first penalty, which depends on $\vtheta$, requires the target policy to remain close to all the behavioral policies in the window. In \ppo-like methods, this is partially enforced by the clipping mechanism, which limits how far consecutive policies can drift.

\subsection{Cross-Iteration Balance-Heuristic-Corrected Sample Reuse in \ppo} \label{sec:sample-reuse-bh}
As the \emph{balance heuristic}~\citep[BH,][]{veach1995optimally} is the most studied MIW weighting due to its desirable theoretical properties~\citep{metelli2020importance} and effectiveness~\citep{papini2019optimistic,montenegro2025reusing} over uniform weighting, here we study its instantiation in \ppo, which we term \bhppo. Specifically, we replace the vanilla IWs in the actor loss with per-step BH-corrected ones. To do so, for every $\vs \in \cS$ and $\va \in \cA$, let $\pi^{\mathrm{BH}}_{\omega_{k}}(\va \mid \vs) \coloneqq \frac{1}{\omega_{k}} \sum_{i=0}^{\omega_{k}-1} \pi_{\vtheta_{k-i}}(\va \mid \vs)$ be the uniform mixture of policies in the window. 
The empirical loss of \bhppo is:
\begin{align*}
    \resizebox{\linewidth}{!}{$\displaystyle \widehat{L}^{\texttt{BH}}(\vtheta,\cD(\omega_{k})) = \frac{1}{\omega_{k}D} \sum_{i=0}^{\omega_{k}-1} \sum_{j=1}^{D} \min\left\{ \frac{\pi_{\vtheta}(\va_{i,j} \mid \vs_{i,j})}{\pi^{\mathrm{BH}}_{\omega_{k}}(\va_{i,j} \mid \vs_{i,j})} A_{i,j}, \clip \left( \frac{\pi_{\vtheta}(\va_{i,j} \mid \vs_{i,j})}{\pi^{\mathrm{BH}}_{\omega_{k}}(\va_{i,j} \mid \vs_{i,j})}, 1-\epsilon, 1+\epsilon \right) A_{i,j} \right\}.$}
\end{align*}
Some comments are in order. 
First, we highlight that $\widehat{L}^{\texttt{BH}}$ retains most of the features already discussed for $\widehat{L}^{\texttt{U}}$ including exploiting stale advantages, without re-evaluating past samples under the most recent critic. 
Also in this case, we consider fitting the critic over $\cD_{k}$ only, thus employing older data exclusively for additional actor updates.
Second, the BH correction weights each sample with all behavioral policies in the window jointly, rather than only its own. 
As a consequence, unlike \uppo, where the IWs of samples in $\cD_{k}$ equal $1$ at the beginning of each iteration, the BH ratio $\pi_{\vtheta_{k}} / \pi^{\mathrm{BH}}_{\omega_{k}}$ generally differs from $1$ whenever $\omega_{k} > 1$, so that even the most recent samples may be clipped from the first update.
Third, the BH is known to reduce the variance of the IWs compared to uniform weighting~\citep{veach1995optimally,metelli2020importance}. Moreover, since $\pi_{\vtheta_{k}}$ is a component of the mixture at the beginning of each iteration, the BH ratios are bounded by $\omega_{k}$, whereas the vanilla ratios $\pi_{\vtheta_{k}}/\pi_{\vtheta_{k-i}}$ of older samples are unbounded.
Fourth, when updating, the target policy $\pi_{\vtheta}$ is pushed towards the boundary of the trust region of $\pi^{\mathrm{BH}}_{\omega_{k}}$, rather than towards the individual trust regions around each $\pi_{\vtheta_{k-i}}$ as in \uppo. 
This provides a single trust region centered at the mixture of the policies in the window.
Finally, it is worth noting that \bhppo incurs a higher computational cost than \uppo. Specifically, evaluating the BH-corrected loss requires evaluating all samples in $\cD(\omega_{k})$ under all policies $\{\pi_{\vtheta_{k-i}}\}_{i=0}^{\omega_{k}-1}$ in the window. However, this cost of $2 (\omega_{k}-1) D$ evaluations is paid \emph{once per data collection step}.

Now, we show a policy improvement lower bound serving as a theoretical grounding of $\widehat{L}^{\texttt{BH}}$.
\begin{restatable}{thr}{bhBound} \label{thr:bh-bound}
    Let $\omega_{k} \in \mathbb{N}_{\ge 1}$. Having access to $\{\pi_{\vtheta_{k-i}}\}_{i=0}^{\omega_{k}-1}$, let $d_{\omega_{k}}^{\mathrm{BH}}(\vs) \coloneqq \frac{1}{\omega_{k}} \sum_{i=0}^{\omega_{k}-1} d_{\vtheta_{k-i}}(\vs)$ be the uniform mixture of the discounted state occupancies. Then, for any $\pi_{\vtheta}$, the following holds:
    \begin{align*}
        &J(\vtheta) - J(\vtheta_{k}) \ge \\
        &\resizebox{\linewidth}{!}{$\displaystyle \frac{1}{\omega_{k}} \sum_{i=0}^{\omega_{k}-1} \E_{\substack{\vs \sim d_{\vtheta_{k-i}} \\ \va \sim \pi_{\vtheta_{k-i}}(\cdot \mid \vs)}} \left[ \frac{\pi_{\vtheta}(\va \mid \vs)}{\pi^{\mathrm{BH}}_{\omega_{k}}(\va \mid \vs)} A^{\pi_{\vtheta_{k-i}}}(\vs,\va) \right] - \frac{2 \gamma C(\vtheta,\vtheta_{k})}{(1-\gamma)^{2} \omega_{k}} \sum_{i=0}^{\omega_{k}-1} \E_{\vs \sim d_{\vtheta_{k-i}}}\left[ D_{\mathrm{TV}}(\pi_{\vtheta}(\cdot \mid \vs) , \pi^{\mathrm{BH}}_{\omega_{k}}(\cdot \mid \vs)) \right]$}  \\
        &\resizebox{\linewidth}{!}{$\displaystyle- \frac{4 R_{\max} (1+\gamma)}{1-\gamma} \sum_{i=0}^{\omega_{k}-1} \sup_{\vs \in \cS} D_{\mathrm{TV}}\left( \pi_{\vtheta_{k}}(\cdot \mid \vs) , \pi_{\vtheta_{k-i}}(\cdot \mid \vs) \right) - \left(\frac{4 R_{\max}}{(1-\gamma)} + \frac{2 \gamma C(\vtheta,\vtheta_{k})}{(1-\gamma)^{2} \omega_{k}}\right) \sum_{i=0}^{\omega_{k}-1} D_{\mathrm{TV}}\left( d^{\mathrm{BH}}_{\omega_{k}} , d_{\vtheta_{k-i}} \right).$} 
    \end{align*}
\end{restatable}
Several remarks are in order. 
First, as for Theorem~\ref{thr:uniform-bound}, this bound accounts for stale advantages and recovers the standard policy improvement lower bound~\citep{achiam2017constrained,pirotta2013safe} when $\omega = 1$. 
Second, the bound features three TV-based penalty terms. The first involves an average of $D_{\mathrm{TV}}(\pi_{\vtheta}(\cdot | \vs), \pi^{\mathrm{BH}}_{\omega_{k}}(\cdot | \vs))$ over the state occupancies in the window, suggesting that the effective trust region should be defined with respect to $\pi^{\mathrm{BH}}_{\omega_{k}}$ rather than any individual behavioral policy; note that $C(\vtheta,\vtheta_{k})$ remains unchanged relative to Theorem~\ref{thr:uniform-bound} and prior work~\citep{achiam2017constrained,queeney2021generalized,zhao2023policy}. The second penalty, involving $\sup_{\vs} D_{\mathrm{TV}}(\pi_{\vtheta_{k}}(\cdot | \vs), \pi_{\vtheta_{k-i}}(\cdot | \vs))$, arises from window-specific stale advantages and, unlike the analogous term in Theorem~\ref{thr:uniform-bound}, is non-decreasing in $\omega$. The third, involving $D_{\mathrm{TV}}(d^{\mathrm{BH}}_{\omega_{k}}, d_{\vtheta_{k-i}})$, captures the mismatch between the mixture state occupancy and the individual occupancies and similarly is non-decreasing in $\omega$. 
Crucially, both the second and third penalties are independent of $\vtheta$ and vanish when taking the gradient of the surrogate loss, thus $(i)$ grounding $\widehat{L}^{\texttt{BH}}$ and $(ii)$ suggesting that the policies in the window should remain close to each other, while the first penalty, which depends on $\vtheta$, requires the target policy to remain close to the mixture $\pi^{\mathrm{BH}}_{\omega_{k}}$.

%% file: content/04_exp.tex
\section{Experiments} \label{sec:experiments}

We provide an experimental validation of the effects of sample reuse, implemented via \uppo and \bhppo, on \ppo. Here, we report only the most interesting results; an extended experimental campaign across diverse MuJoCo~\citep{todorov2012mujoco} tasks is provided in Appendix~\ref{apx:exp-additional}. 
To ensure a fair comparison, we always start from the \emph{tuned} \ppo configuration of~\citep{rl-zoo3}, and the reuse variants all share common hyperparameters unless otherwise specified. All experimental details are deferred to Appendix~\ref{apx:exp-details}.

\textbf{Comparison Settings.}~~When comparing \uppo and \bhppo against \ppo, several experimental settings are possible. In this section we consider the following two:
\begin{itemize}[leftmargin=*,topsep=-6pt,noitemsep]
    \item \emph{Fixed mini-batch size}: the reuse variants share all hyperparameters with \ppo ($\widetilde{K}$, $B$, and $N_E H$). Since $|\cD(\omega_k)| = \omega_k D$, fixing $B$ means the reuse variants perform $\omega_k$ times as many actor updates per iteration as \ppo, with each datum having a maximum life cycle of $\omega_k \widetilde{K}$ actor updates. Each mini-batch contains $B / \omega_k$ samples per window, so each update sees a smaller fraction of fresh data than \ppo. To isolate the effect of data reuse from the effect of additional updates, we also run \ppo with $\omega_k \widetilde{K}$ epochs per iteration, matching the total update count of the reuse variants.
    \item \emph{Fixed mini-batch count}: \uppo and \bhppo keep the same number of mini-batches $N_{\cB}$ as \ppo, yielding larger mini-batches of size $\omega_k B$. Each update thus sees the same amount of fresh data as \ppo, while the reuse variants additionally incorporate past samples into each mini-batch. All methods therefore perform the same number of updates per iteration, isolating the effect of additional data from that of additional gradient steps.
\end{itemize}
 These settings directly measure sample efficiency: given the same number of environment interactions, they test how well each method exploits the available information. Both settings are designed to maintain fairness with respect to \ppo and to isolate distinct aspects of the data reuse mechanism. 

\begin{figure}[t]
    \centering

    \begin{subfigure}[t]{0.49\linewidth}
        \begin{minipage}[t]{0.49\linewidth}
            \includegraphics[width=\linewidth]{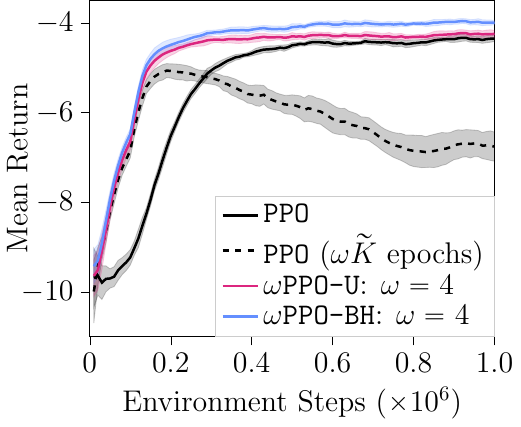}
        \end{minipage}
        \hfill
        \begin{minipage}[t]{0.49\linewidth}
            \includegraphics[width=\linewidth]{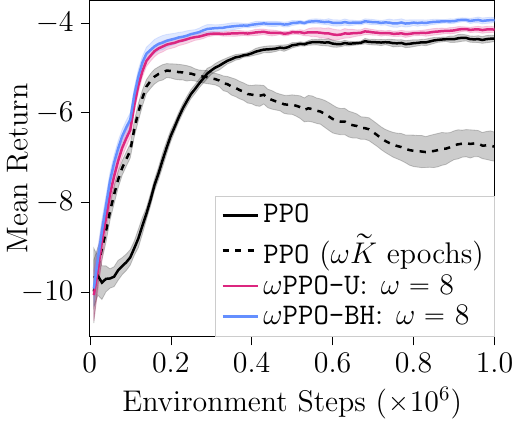}
        \end{minipage}
        \caption{\emph{Reacher-v5}.}
        \label{fig:main-1-reacher}
    \end{subfigure}
    \hfill\begin{subfigure}[t]{0.49\linewidth}
        \begin{minipage}[t]{0.49\linewidth}
            \includegraphics[width=\linewidth]{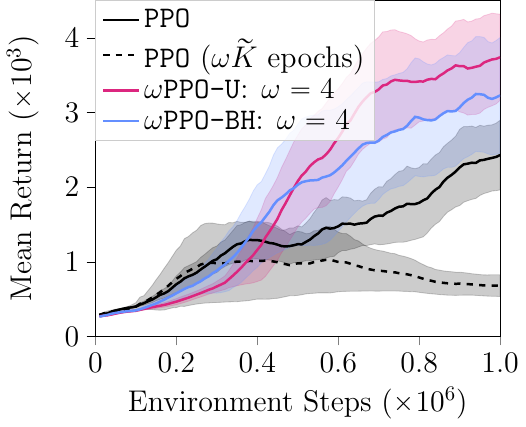}
        \end{minipage}
        \hfill
        \begin{minipage}[t]{0.49\linewidth}
            \includegraphics[width=\linewidth]{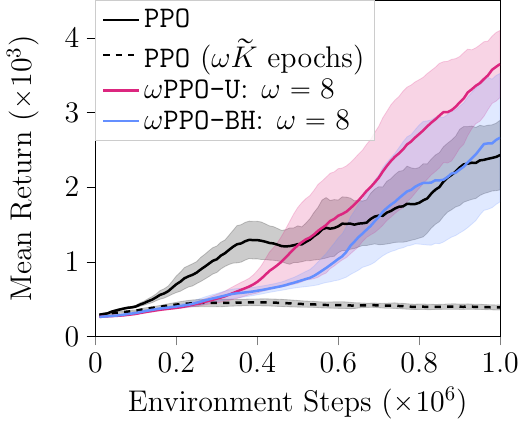}
        \end{minipage}
        \caption{\emph{Walker2d-v5}.}
        \label{fig:main-1-walker}
    \end{subfigure}

    \vspace{0.1cm}
    
    \begin{subfigure}[t]{0.49\linewidth}
        \begin{minipage}[t]{0.49\linewidth}
            \includegraphics[width=\linewidth]{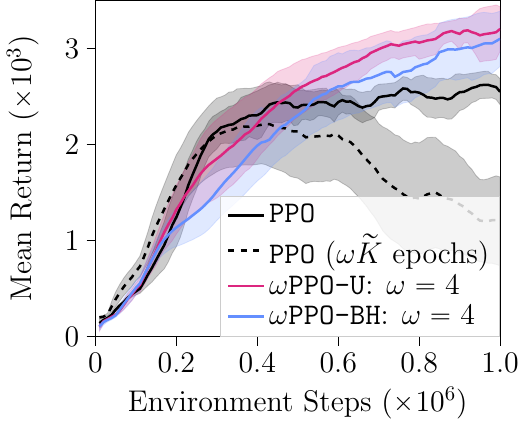}
        \end{minipage}
        \hfill
        \begin{minipage}[t]{0.49\linewidth}
            \includegraphics[width=\linewidth]{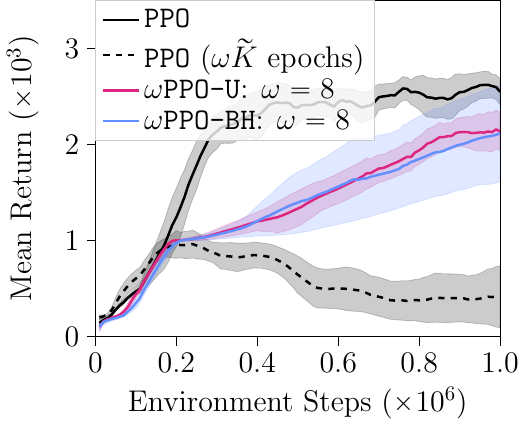}
        \end{minipage}
        \caption{\emph{Hopper-v5}.}
        \label{fig:main-1-hopper}
    \end{subfigure}
    \hfill\begin{subfigure}[t]{0.49\linewidth}
        \begin{minipage}[t]{0.49\linewidth}
            \includegraphics[width=\linewidth]{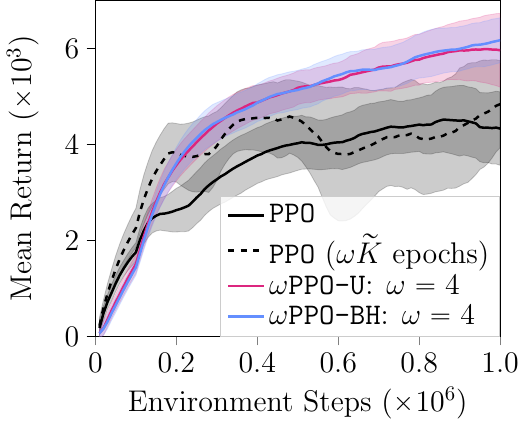}
        \end{minipage}
        \hfill
        \begin{minipage}[t]{0.49\linewidth}
            \includegraphics[width=\linewidth]{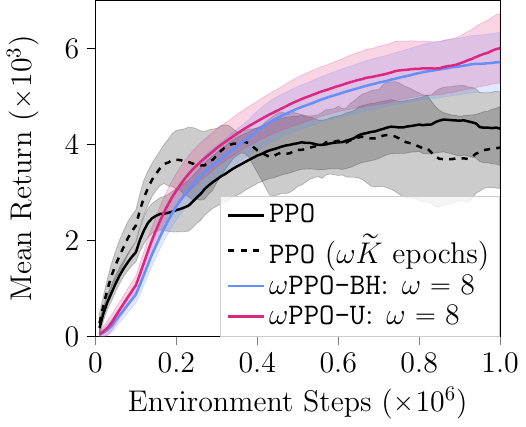}
        \end{minipage}
        \caption{\emph{HalfCheetah-v5}.}
        \label{fig:main-1-half}
    \end{subfigure}

    \caption{\ppo, \uppo, and \bhppo in the \emph{fixed $B$} setting (10 runs, mean $\pm$ 95\% C.I.).}
    \label{fig:main-1}
\end{figure}

\begin{table}[t]
    \centering
    
    \resizebox{\linewidth}{!}{\input{tables/main-perf}}

    \vspace{0.1cm}
    
    \caption{Mean final performance (average of last 10 evaluations, 10 runs $\pm 95\%$ C.I.) for the settings of Figures~\ref{fig:main-1}--\ref{fig:main-2}. Colored cells: best mean performance; bold: methods whose C.I. intersects the best.}
    \label{tab:main}
\end{table}

\textbf{Fixed $B$.}~~Figure~\ref{fig:main-1} shows training evaluation curves across four tasks, comparing \ppo, \uppo, \bhppo, and a \ppo variant running $\omega\widetilde{K}$ epochs per iteration. The latter collapses in almost every setting, while the reuse variants, whose mini-batches come from $\cD(\omega_{k})$, do not exhibit the same instability. 
At $\omega=4$, both \uppo and \bhppo consistently outperform \ppo in terms of final performance of the deployed policy (Table~\ref{tab:main}), suggesting that the benefit of data reuse is not merely due to performing more updates.
However, this trend does not hold uniformly across tasks when enlarging $\omega$. The \emph{Hopper-v5} environment is a pathological case: at $\omega=8$, the reuse variants not only fail to outperform \ppo but exhibit degraded training. 
We conjecture this results from several compounding factors, amplified by this environment's configuration ($H = 2048$, $B = 64$, $\widetilde{K} = 10$). First, at $\omega = 8$ the reuse variants perform $2560$ updates per iteration against the $320$ of \ppo, potentially amplifying the bias introduced by stale advantages across many gradient steps. Second, with $B/\omega = 8$ fresh samples per mini-batch, the fraction of on-policy signal per update is particularly small. Third, since all hyperparameters are inherited from \ppo, the larger number of updates may increase the drift between consecutive policies and hence the diversity of the policies in the window, possibly making past samples falling outside the clipping range more frequently.
All in all, in this setting \ppo consistently benefits from data reuse when the reuse variants are allowed to perform $\omega$ times as many updates as \ppo, up to a moderate window size, 
beyond which the factors above may hinder training.

\begin{figure}[t]
    \centering

    \begin{subfigure}[t]{0.49\linewidth}
        \begin{minipage}[t]{0.49\linewidth}
            \includegraphics[width=\linewidth]{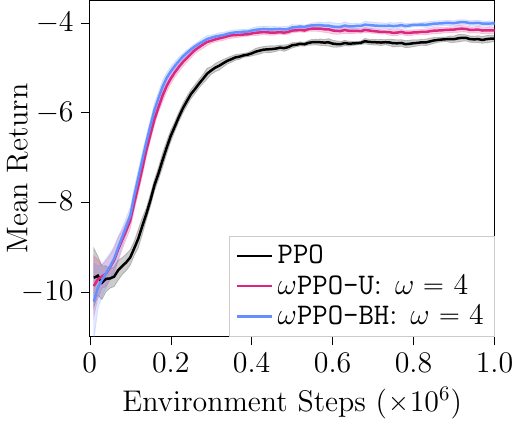}
        \end{minipage}
        \hfill
        \begin{minipage}[t]{0.49\linewidth}
            \includegraphics[width=\linewidth]{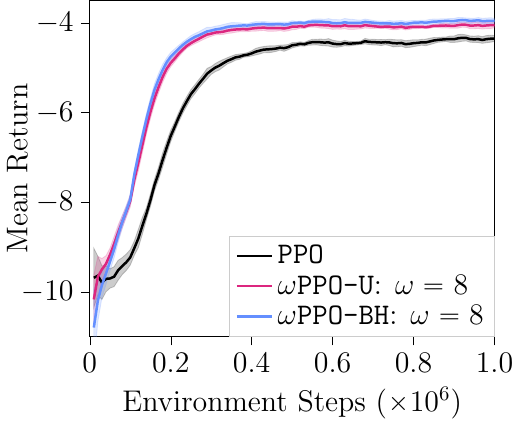}
        \end{minipage}
        \caption{\emph{Reacher-v5}.}
        \label{fig:main-2-reacher}
    \end{subfigure}
    \hfill\begin{subfigure}[t]{0.49\linewidth}
        \begin{minipage}[t]{0.49\linewidth}
            \includegraphics[width=\linewidth]{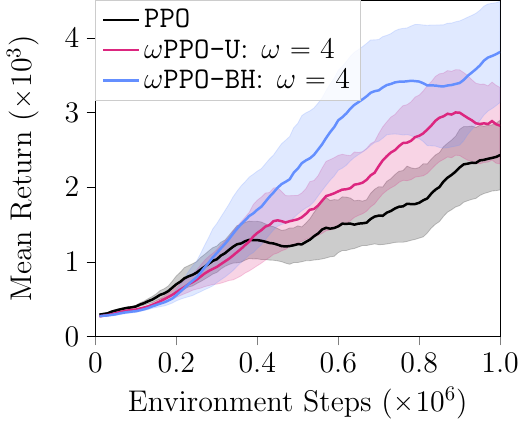}
        \end{minipage}
        \hfill
        \begin{minipage}[t]{0.49\linewidth}
            \includegraphics[width=\linewidth]{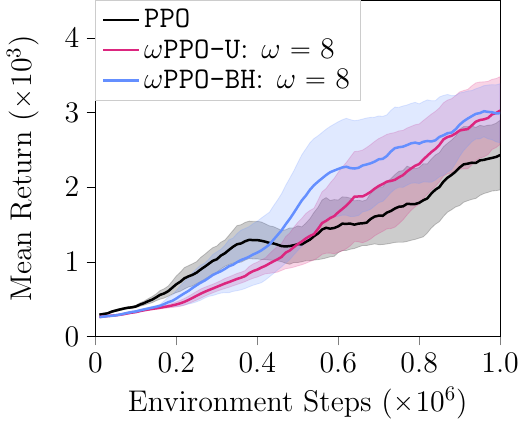}
        \end{minipage}
        \caption{\emph{Walker2d-v5}.}
        \label{fig:main-2-walker}
    \end{subfigure}

    \vspace{0.1cm}
    
    \begin{subfigure}[t]{0.49\linewidth}
        \begin{minipage}[t]{0.49\linewidth}
            \includegraphics[width=\linewidth]{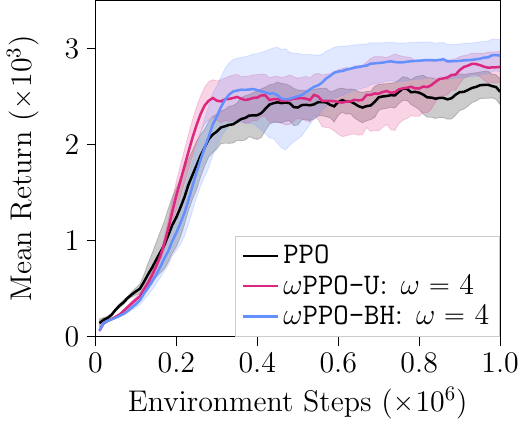}
        \end{minipage}
        \hfill
        \begin{minipage}[t]{0.49\linewidth}
            \includegraphics[width=\linewidth]{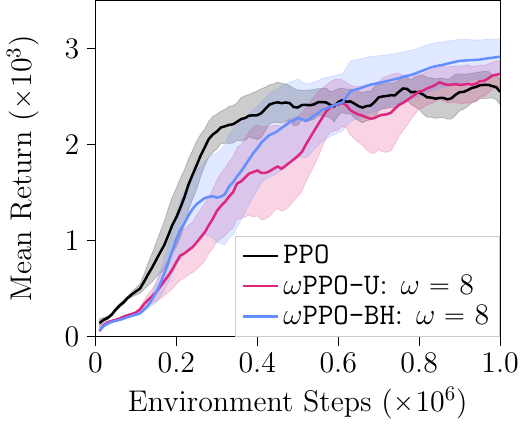}
        \end{minipage}
        \caption{\emph{Hopper-v5}.}
        \label{fig:main-2-hopper}
    \end{subfigure}
    \hfill\begin{subfigure}[t]{0.49\linewidth}
        \begin{minipage}[t]{0.49\linewidth}
            \includegraphics[width=\linewidth]{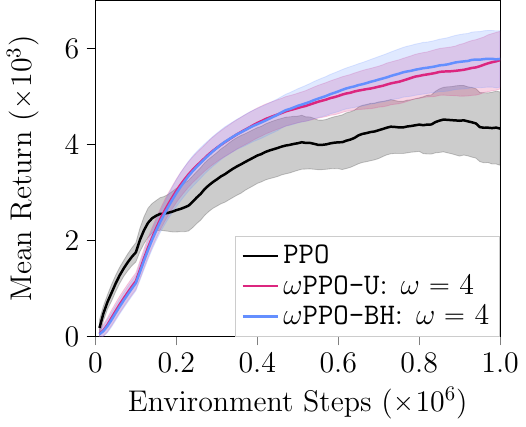}
        \end{minipage}
        \hfill
        \begin{minipage}[t]{0.49\linewidth}
            \includegraphics[width=\linewidth]{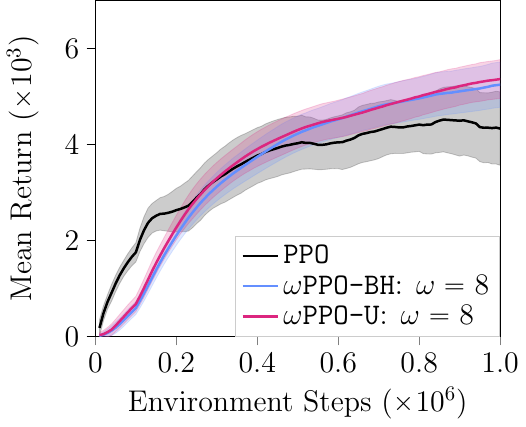}
        \end{minipage}
        \caption{\emph{HalfCheetah-v5}.}
        \label{fig:main-2-half}
    \end{subfigure}

    \caption{\ppo, \uppo, and \bhppo in the \emph{fixed $N_{\cB}$} setting (10 runs, mean $\pm$ 95\% C.I.).}
    \label{fig:main-2}
\end{figure}

\textbf{Fixed $N_{\cB}$.}~~Figure~\ref{fig:main-2} shows training evaluation curves across the same four environments, comparing \ppo, \uppo, and \bhppo. We recall that in this setting all methods perform the same number of updates per iteration and each update sees the same amount of fresh data, while the reuse variants employ mini-batches of size $\omega B$. A consistent trend is easier to identify here: \ppo benefits from sample reuse across all environments, despite the update count being matched. Most of the benefits are observable at $\omega=4$, particularly under \bhppo, as further confirmed by Table~\ref{tab:main}. Increasing to $\omega=8$ yields diminishing returns, with the reuse variants either matching or only marginally improving over \ppo. 
We conjecture this is due to the increasing proportion of stale advantages and of samples collected under increasingly distant policies within each mini-batch, which gradually offsets the benefit of additional data as the window grows.

\begin{figure}[t]
    \centering

    \begin{subfigure}[t]{0.49\linewidth}
        \includegraphics[width=\linewidth]{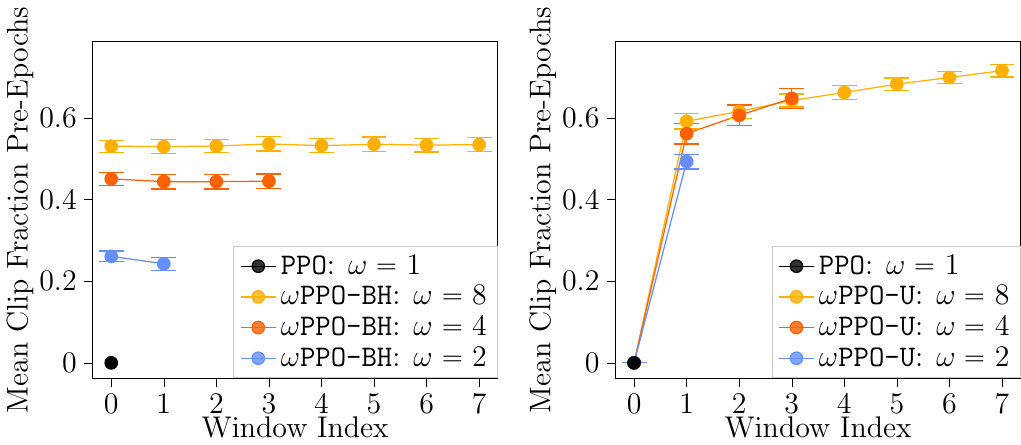}
        \caption{\emph{HalfCheeath-v5}: Fixed $B$.}
        \label{fig:main-clip-half-set1}
    \end{subfigure}
    \hfill\begin{subfigure}[t]{0.49\linewidth}
        \includegraphics[width=\linewidth]{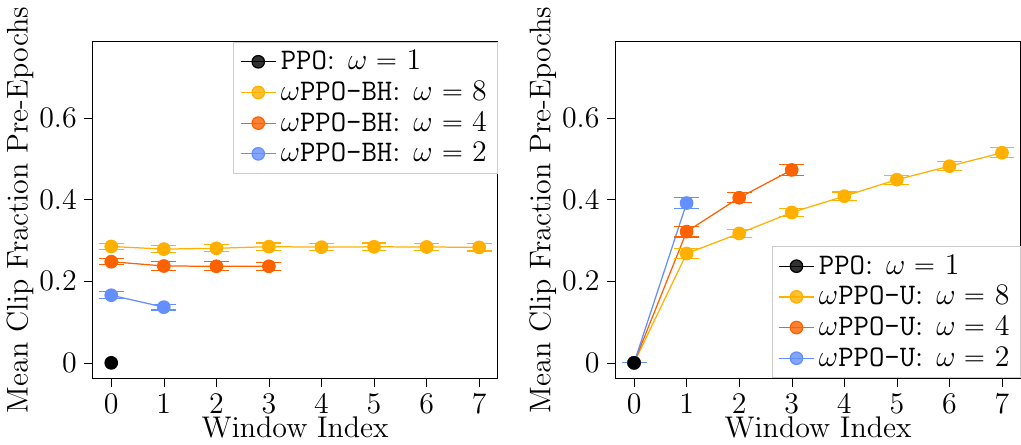}
        \caption{\emph{HalfCheetah-v5}: Fixed $N_{\cB}$.}
        \label{fig:main-clip-half-set3}
    \end{subfigure}

    \vspace{-0.1cm}

    \begin{subfigure}[t]{0.49\linewidth}
        \includegraphics[width=\linewidth]{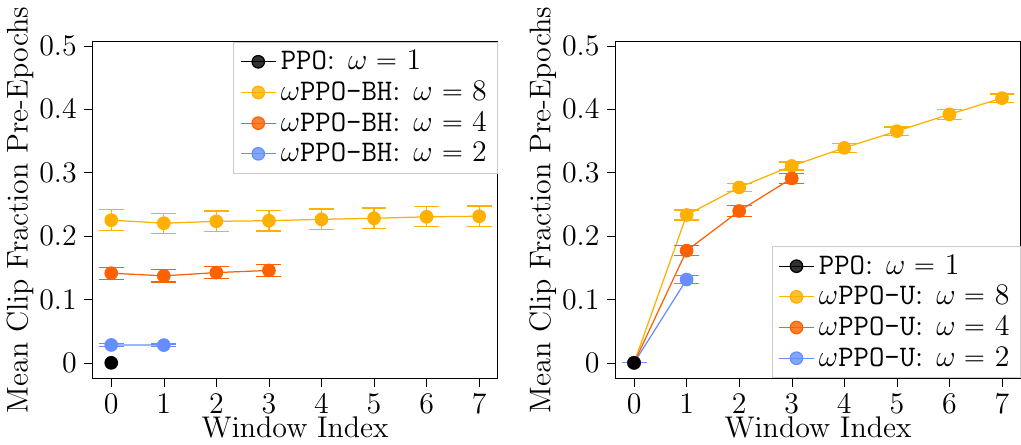}
        \caption{\emph{Hopper-v5}: Fixed $B$.}
        \label{fig:main-clip-hop-set1}
    \end{subfigure}
    \hfill
    \begin{subfigure}[t]{0.49\linewidth}
        \includegraphics[width=\linewidth]{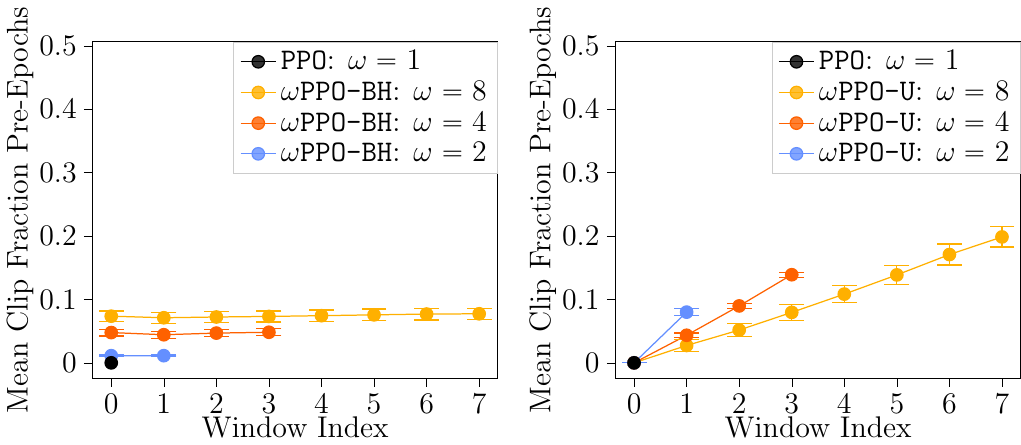}
        \caption{\emph{Hopper-v5}: Fixed $N_{\cB}$.}
        \label{fig:main-clip-hop-set3}
    \end{subfigure}

    \vspace{-0.1cm}

    \begin{subfigure}[t]{0.49\linewidth}
        \includegraphics[width=\linewidth]{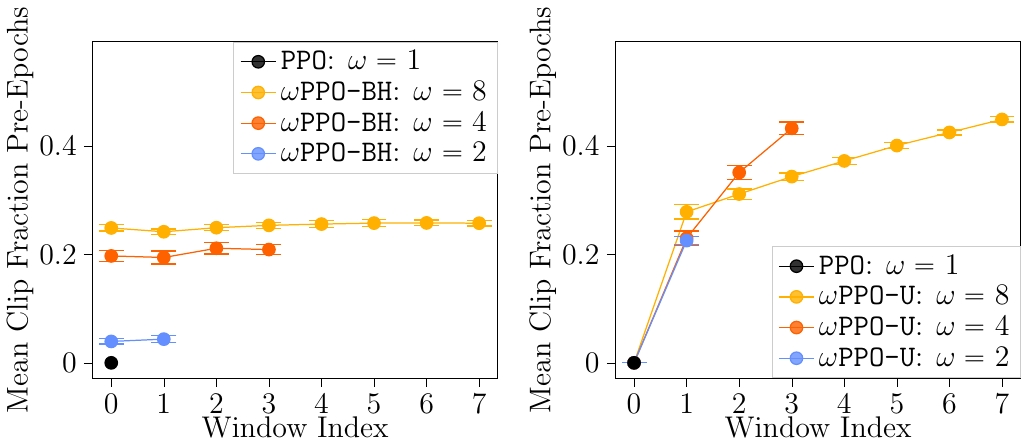}
        \caption{\emph{Walker2d-v5}: Fixed $B$.}
        \label{fig:main-clip-wal-set1}
    \end{subfigure}
    \hfill
    \begin{subfigure}[t]{0.49\linewidth}
        \includegraphics[width=\linewidth]{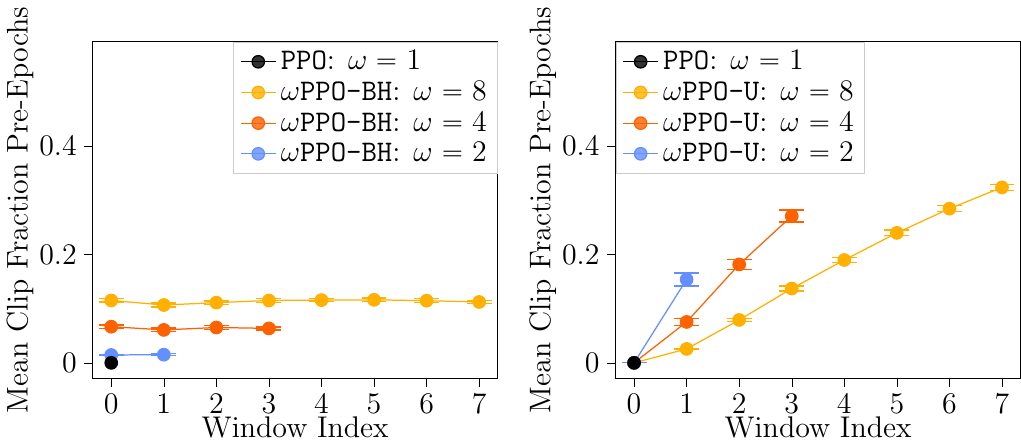}
        \caption{\emph{Walker2d-v5}: Fixed $N_{\cB}$.}
        \label{fig:main-clip-wal-set3}
    \end{subfigure}

    \vspace{-0.1cm}

    \begin{subfigure}[t]{0.49\linewidth}
        \includegraphics[width=\linewidth]{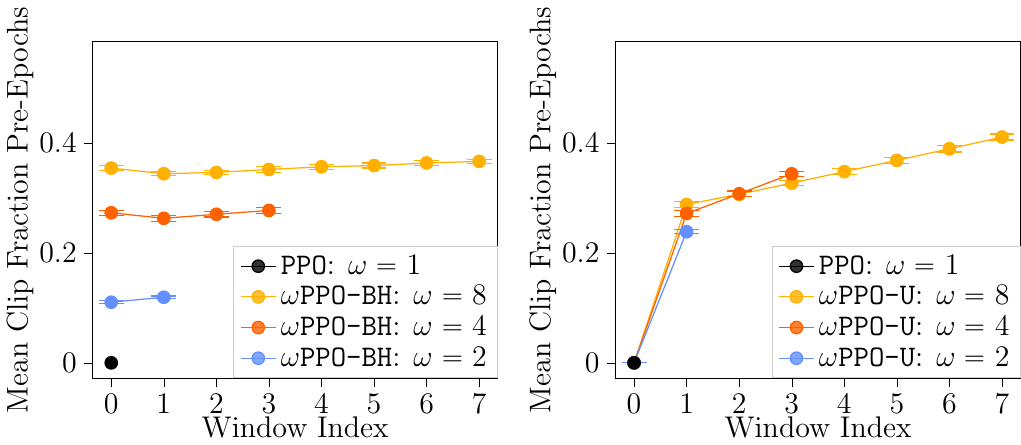}
        \caption{\emph{Reacher-v5}: Fixed $B$.}
        \label{fig:main-clip-rea-set1}
    \end{subfigure}
    \hfill
    \begin{subfigure}[t]{0.49\linewidth}
        \includegraphics[width=\linewidth]{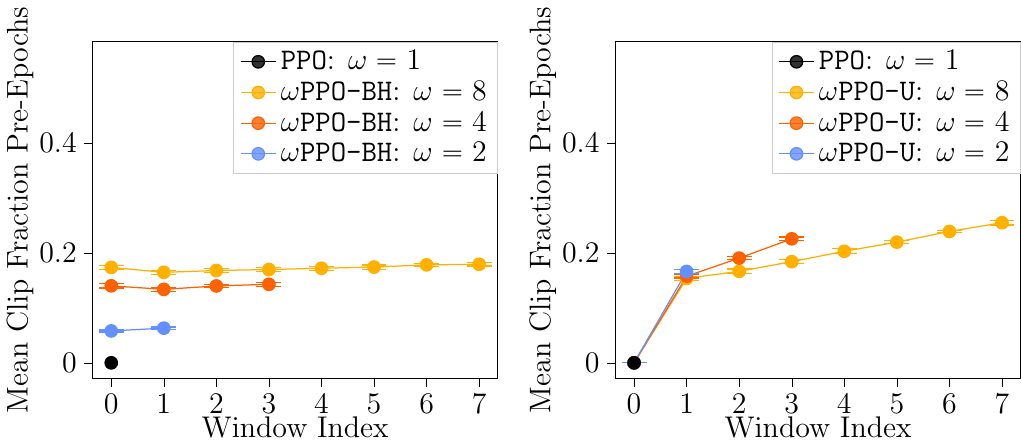}
        \caption{\emph{Reacher-v5}: Fixed $N_{\cB}$.}
        \label{fig:main-clip-rea-set3}
    \end{subfigure}

    \caption{Clip ratio diagnostics (10 runs, mean $\pm$ 95\% C.I.).}
    \label{fig:main-3}
\end{figure}

\textbf{\uppo vs.\ \bhppo.}~~In general, neither weighting scheme consistently dominates the other: the best-performing policy is achievable with both, given an appropriate $\omega$ and environment (Table~\ref{tab:main}). However, the two schemes exploit past data differently. 
In \uppo, each sample is weighted against its own behavioral policy $\pi_{\vtheta_{k-i}}$, so the ratio deviates further from $1$ for older window indices, causing \uppo to clip an increasing fraction of past data as $\omega$ grows. In \bhppo, the mixture denominator $\pi^{\mathrm{BH}}_{\omega_k}$ has two consequences: $(i)$ at the beginning of each iteration, the IWs are bounded by $\omega_k$, which limits how far above the clipping range older samples can start; and $(ii)$ the IWs of the most recent samples generally differ from $1$, so that, at small $\epsilon$ and large $\omega$, part of the fresh data may be clipped from the first update. Figure~\ref{fig:main-3} reports the average pre-optimization clip fractions, grouped by window index: \bhppo maintains a lower pre-optimization clip fraction for older windows, consistently with $(i)$, while \uppo exhibits pre-optimization clip fractions that grow with the window index, confirming that older samples contribute increasingly less.

\textbf{Main Takeaways.}~~We now summarize when and how \ppo may benefit from data reuse, based on the results of this section and of Appendix~\ref{apx:exp-additional}. We recall that this paper studies data reuse per se: feeding \ppo with a composite buffer $\cD(\omega_{k})$ from the most recent $\omega$ iterations, either with vanilla importance weights (\uppo) or with a BH correction (\bhppo). 
First, \ppo benefits from sample reuse in both considered settings (fixed $B$ and fixed $N_{\cB}$) and the benefit \emph{stems from from the additional data} rather than from the increased number of updates alone.
Second, the impact of reuse manifests \emph{primarily in the final performance} of the deployed policy, while \emph{improvements in convergence speed} also occur but are closely tied to the environment and hyperparameter configuration, appearing consistently only in \emph{Reacher-v5}, the simplest here, and \emph{Ant-v5} (see Appendix~\ref{apx:exp-additional}). 
Third, in most environments, reuse is most effective at moderate window sizes (e.g., $\omega = 4$). With larger windows, particularly in the fixed $B$ setting, \emph{stale advantages}, the \emph{increased diversity of the policies} in the window, and the \emph{dilution of fresh data} may offset the benefits of the additional data.
Finally, while both \uppo and \bhppo are beneficial depending on the task and hyperparameters, they exploit past data differently: 
\bhppo clips a smaller fraction of older samples thanks to its mixture denominator but may clip part of the most recent ones from the start, while \uppo fully exploits fresh data
at the cost of clipping a larger fraction of older samples.

%% file: tables/main-perf.tex
{\renewcommand{\arraystretch}{1.2}
\begin{tabular}{ll cc cc cc cc}
    \toprule
    & & \multicolumn{2}{c}{\emph{Reacher-v5}} & \multicolumn{2}{c}{\emph{Walker2d-v5}} & \multicolumn{2}{c}{\emph{Hopper-v5}} & \multicolumn{2}{c}{\emph{HalfCheetah-v5}} \\
    \cmidrule(lr){3-4} \cmidrule(lr){5-6} \cmidrule(lr){7-8} \cmidrule(lr){9-10}
    \textbf{Algorithm} & & \textbf{Fixed $B$} & \textbf{Fixed $N_{\cB}$} & \textbf{Fixed $B$} & \textbf{Fixed $N_{\cB}$} & \textbf{Fixed $B$} & \textbf{Fixed $N_{\cB}$} & \textbf{Fixed $B$} & \textbf{Fixed $N_{\cB}$} \\
    \midrule
    \ppo & & \multicolumn{2}{c}{$-4.35 \pm 0.063$} & \multicolumn{2}{c}{$2430 \pm 443$} & \multicolumn{2}{c}{$2530 \pm 137$} & \multicolumn{2}{c}{$4330 \pm 725$} \\
    \midrule
    \multirow{3}{*}{\ppo ($\omega \widetilde{K}$ epochs)}
        & $2\times$ & $-4.70 \pm 0.069$ & --- & $2510 \pm 556$ & --- & $2790 \pm 305$ & --- & $4500 \pm 1120$ & --- \\
        & $4\times$ & $-6.76 \pm 0.308$ & --- & $682 \pm 138$  & --- & $1190 \pm 442$ & --- & $4840 \pm 861$  & --- \\
        & $8\times$ & $-7.97 \pm 0.531$ & --- & $392 \pm 32.4$ & --- & $415 \pm 317$  & --- & $3940 \pm 810$  & --- \\
    \midrule
    \multirow{3}{*}{\uppo}
        & $\omega=2$ & $-4.29 \pm 0.065$ & $-4.22 \pm 0.060$ & $\mathbf{3350 \pm 556}$ & $\mathbf{3340 \pm 535}$ & \cellcolor{ibmMagenta!20}$\mathbf{3240 \pm 94.3}$ & $\mathbf{2680 \pm 163}$ & $\mathbf{5560 \pm 335}$ & $\mathbf{5160 \pm 515}$ \\
        & $\omega=4$ & $-4.25 \pm 0.075$ & $-4.17 \pm 0.063$ & \cellcolor{ibmMagenta!20}$\mathbf{3740 \pm 546}$ & $\mathbf{2820 \pm 485}$ & $\mathbf{3220 \pm 192}$ & $\mathbf{2810 \pm 157}$ & $\mathbf{5960 \pm 729}$ & $\mathbf{5760 \pm 569}$ \\
        & $\omega=8$ & $-4.15 \pm 0.057$ & $\mathbf{-4.05 \pm 0.057}$ & $\mathbf{3650 \pm 426}$ & $\mathbf{3030 \pm 430}$ & $2130 \pm 195$ & $\mathbf{2750 \pm 131}$ & $\mathbf{5870 \pm 497}$ & $\mathbf{5360 \pm 384}$ \\
    \midrule
    \multirow{3}{*}{\bhppo}
        & $\omega=2$ & $-4.14 \pm 0.047$ & $-4.12 \pm 0.052$ & $\mathbf{3360 \pm 491}$ & $\mathbf{3250 \pm 530}$ & $\mathbf{3220 \pm 94.6}$ & $\mathbf{2740 \pm 128}$ & $\mathbf{5530 \pm 550}$ & $\mathbf{4580 \pm 949}$ \\
        & $\omega=4$ & $\mathbf{-3.99 \pm 0.045}$ & $\mathbf{-4.01 \pm 0.045}$ & $\mathbf{3230 \pm 735}$ & \cellcolor{ibmUltramarine!20}$\mathbf{3810 \pm 642}$ & $\mathbf{3110 \pm 274}$ & \cellcolor{ibmUltramarine!20}$\mathbf{2920 \pm 159}$ & \cellcolor{ibmUltramarine!20}$\mathbf{6170 \pm 444}$ & \cellcolor{ibmUltramarine!20}$\mathbf{5780 \pm 558}$ \\
        & $\omega=8$ & \cellcolor{ibmUltramarine!20}$\mathbf{-3.94 \pm 0.052}$ & \cellcolor{ibmUltramarine!20}$\mathbf{-3.96 \pm 0.054}$ & $\mathbf{2660 \pm 815}$ & $\mathbf{2990 \pm 375}$ & $2120 \pm 476$ & $\mathbf{2910 \pm 177}$ & $\mathbf{5720 \pm 586}$ & $\mathbf{5240 \pm 447}$ \\
    \bottomrule
\end{tabular}}

%% file: content/05_conclusion.tex
\section{Conclusion} \label{sec:conclusion}
We studied sample reuse in \ppo within a MIW framework, deliberately preserving its core mechanics to isolate the effect of data reuse from other factors.
We instantiated \uppo and \bhppo, employing uniform and BH-corrected IWs, respectively, and derived policy improvement lower bounds for both. Beyond grounding their losses, the bounds reveal penalties uncontrollable by the actor updates, suggesting that the policies in the reuse window should remain close to each other.
Empirically, we found that moderate data reuse consistently improves the final performances across diverse control tasks, while excessive reuse can be harmful. Neither \uppo nor \bhppo consistently dominates the other, though \bhppo may better exploit older samples at the cost of potentially clipping part of the fresh data.
Future work should extend this analysis to other deep on-policy methods~\citep[e.g.,][]{sutton1999policy,shao2024deepseekmath}, and systematically isolate the individual contributions of advantage re-evaluation, modified clipping, and other heuristics introduced by existing \ppo reuse variants~\citep[e.g.,][]{queeney2021generalized,meng2023off,zhao2023policy}.

%% file: appendix/apx_ppo_details.tex
\section{Algorithmic Details: \ppo and Data Collection} \label{apx:ppo-details}
In this section, we provide details on \ppo, remaining tied to its common practical implementation~\citep[e.g.,][]{stable-baselines3}.

\input{algo/algo_data}

\subsection{On-Policy Data Collection}
We begin by detailing the data collection process. Unlike the theoretical formulation based on complete episodes, practical implementations~\citep{stable-baselines3} interact with the environment for a \emph{fixed} number of timesteps $H \in \Nat$, regardless of the task's actual horizon. This fixed-length rollout is adopted primarily for computational efficiency and memory management. 
To handle episode termination, we employ a boolean flag $\xi_{h}$, which is true if the environment enters an absorbing state at timestep $h$. Crucially, if the rollout ends at timestep $H$ without reaching an absorbing state (truncation), the value of the last state is used for bootstrapping.

For the sake of generality, we consider a set of parallel environments $\{E_{e}\}_{e=1}^{N_{E}}$ collecting independent trajectory segments of length $H$. During interaction, the agent stores the essential data required for the optimization phase: the state-action pairs $(\vs_{h}, \va_{h})$, the rewards $r_{h}$, the value estimates $V_{\vw}(\vs_{h})$, and the probabilities $ \pi_{\vtheta}(\va_{h} \mid \vs_{h})$.
After collection, the algorithm computes the advantage estimates $A_{h}$ and the return targets $R_{h}$. The standard approach is Generalized Advantage Estimation~\citep[GAE,][]{schulman2015high}, which balances the bias-variance trade-off akin to $\text{TD}(\lambda)$. 
While a Monte Carlo estimate is unbiased but suffers from high variance, a $\text{TD}(0)$ approach has low variance but introduces significant bias due to bootstrapping. GAE interpolates between these regimes. 
Formally, let $\delta_{h} = r_{h} + \gamma (1-\xi_{h}) V_{\vw}(\vs_{h+1}) - V_{\vw}(\vs_{h})$ be the TD-error at step $h$. The GAE estimate is defined as the exponentially weighted sum of these errors:
\begin{align*}
    A_{h} = \sum_{k=0}^{H-h-1} (\gamma \lambda)^{k} \delta_{h+k},
\end{align*}
where $\lambda \in [0,1]$ controls the trade-off, and the sum is truncated at the rollout boundary $H$ (handling the last state via bootstrapping as previously described). Algorithm~\ref{alg:data-collection} details the data collection. There, for $n,m\in\mathbb{N}$ with $n \ge m$, we denote $\rsb{n,m} \coloneqq \{n, n-1, \, \dots , \, m\}$ and $\rsb{n} \coloneqq \rsb{n,1}$.

\input{algo/algo_ppo}

\subsection{\ppo Details}
As highlighted in Section~\ref{sec:introduction}, \ppo~\citep{schulman2017proximal} is one of the most used and successful RL methods. 
In order to extend it in a meaningful way, we detail here an (almost) practical version of it.
Before proceeding, we need to introduce additional notation.
Let $L^{\texttt{PPO}}: \Theta \times \cW \to \Reals$ be a joint loss for both the actor and the critic, defined as:
\begin{align*}
    L^{\texttt{PPO}}(\vtheta,\vw) &= \E_{\substack{\vs \sim d_{\vtheta_{b}} \\ \va \sim \pi_{\vtheta_{b}}(\cdot \mid \vs)}} \left[ \min\left\{ \frac{\pi_{\vtheta}(\va \mid \vs)}{\pi_{\vtheta_{b}}(\va \mid \vs)} A^{\pi_{\vtheta_{b}}}(\vs,\va), \; \mathrm{clip}\left(\frac{\pi_{\vtheta}(\va \mid \vs)}{\pi_{\vtheta_{b}}(\va \mid \vs)}, 1-\epsilon, 1+\epsilon \right) A^{\pi_{\vtheta_{b}}}(\vs,\va) \right\} \right]\\
    &\quad - \frac{c_{1}}{2} \E_{\vs \sim d_{\vtheta_{b}}} \left[ \left( V_{\vw}(\vs) - V^{\pi_{\vtheta_{b}}}(\vs) \right)^{2}\right] \\
    &\quad + c_{2} \E_{\vs \sim d_{\vtheta_{b}}}\left[ \cH(\pi_{\vtheta}(\cdot \mid \vs)) \right],
\end{align*}
where $\pi_{\vtheta_{b}}$ is a behavioral policy used to collect data, where $\cH(\pi_{\vtheta}(\cdot \mid \vs)) \coloneqq \E_{\va \sim \pi_{\vtheta}}\left[ -\log \pi_{\vtheta}(\va \mid \vs) \right]$ is the entropy of the policy $\pi_{\vtheta}$ in state $\vs$, and where $c_{1}, c_{2} \in \Reals[0]$ are weighting constants.

Now, let $\cB = \{ (\vs_{i}, \va_{i}, \rho_{i}, A_{i}, R_{i}) \}_{i=1}^{M}$ be a batch of $M$ interaction tuples, where $A_{i}$ and $R_{i}$ are estimates of the advantage function and the return-to-go, both \wrt the policy used to collect data, respectively. 
We can define the sample-based version of $L^{\texttt{PPO}}$ as follows:
\begin{align*}
    \widehat{L}^{\texttt{PPO}}(\vtheta, \vw, \cB) &= \frac{1}{|\cB|} \sum_{i \in \cB} \min\left\{ \frac{\pi_{\vtheta}(\va_{i} \mid \vs_{i})}{\pi_{\vtheta_{i}}(\va_{i} \mid \vs_{i})} A_{i}, \mathrm{clip}\left( \frac{\pi_{\vtheta}(\va_{i} \mid \vs_{i})}{\pi_{\vtheta_{i}}(\va_{i} \mid \vs_{i})}, 1-\epsilon, 1+\epsilon \right)A_{i} \right\} \\
    &\qquad - \frac{c_{1}}{2} (V_{\vw}(\vs_{i}) - R_{i})^{2} + c_{2} \cH(\pi_{\vtheta}(\cdot\mid \vs_{i})).
\end{align*}
That being said, the parameters can then be updated as:
\begin{align*}
    \vtheta \leftarrow \vtheta + \alpha \nabla_{\vtheta} \widehat{L}^{\texttt{PPO}}(\vtheta, \vw, \cB) \quad \text{and} \quad \vw \leftarrow \vw + \alpha \nabla_{\vw} \widehat{L}^{\texttt{PPO}}(\vtheta, \vw, \cB),
\end{align*}
where $\alpha$ is the step size, usually managed by Adam~\citep{kingma2014adam} or by other schedulers (\eg by linearly decreasing the step size or adapting the step size in order to maintain a certain KL divergence among policies across iterations~\citep{stable-baselines3,alhafez2023b,queeney2021generalized}).
Notice that, in many implementations, the \emph{critic} loss is clipped too.
The detailed pseudo-code is reported in Algorithm~\ref{alg:ppo}.

%% file: algo/algo_data.tex
\RestyleAlgo{ruled}
\LinesNumbered
\begin{algorithm}[h!]
    \caption{On-Policy Data Collection.}\label{alg:data-collection}
    \SetKwInOut{Input}{Input}
    \SetKwInOut{Output}{Output}
    \small
    \Input{Actor $\pi_{\vtheta}$, critic $V_{\vw}$, interaction length $H$, environments $\{E_{e}\}_{e=1}^{N_{E}}$, GAE parameter $\lambda$, discount factor $\gamma$.}
    
    Initialize global buffer $\cD \leftarrow \emptyset$;

    \For{$e \in \dsb{N_E}$}{
        Initialize local buffer $\cD_{\text{local}} \leftarrow \emptyset$;
        
        Receive initial state $\vs \leftarrow \vs_1$ from $E_e$;

        \tcp{Interaction}
        \For{$h \in \dsb{H}$}{
            Set $\vs_{h} \leftarrow \vs$;
            
            Sample $\va_{h} \sim \pi_{\vtheta}(\cdot \mid \vs_{h})$ and execute it in $E_{e}$; 
            
            Observe $r_h$, $\vs_{h+1}$, and done flag $\xi_h$;
            
            Compute value estimate $v_{h} = V_{\vw}(\vs_{h})$ and $\rho_{h} = \pi_{\vtheta}(\va_{h} \mid \vs_{h})$;
            
            Store tuple $(\vs_{h}, \va_{h}, r_{h}, v_{h}, \rho_{h}, \xi_{h})$ in $\cD_{\text{local}}$;
            
            \If{$\xi_{h}$ is True}{
                Reset $E_e$ and receive new $\vs_{h+1}$;
            }

            Set $\vs \leftarrow \vs_{h+1}$;
        }

        \tcp{GAE Calculation taking samples from $\cD_{\text{local}}$}
        Compute the last value function $v_{H+1} = V_{\vw}(\vs_{H+1})$;
        
        Initialize next advantage $A_{H+1} = 0$;

        \For{$h \in \rsb{H}$}{
            $\delta_{h} = r_{h} + \gamma (1-\xi_{h}) v_{h+1} - v_{h}$;
            
            $A_{h} = \delta_{h} + \gamma \lambda (1-\xi_{h}) A_{h+1}$;
            
            $R_{h} = A_{h} + v_{h}$;
            
            Add $(\vs_{h}, \va_{h}, \rho_{h}, A_{h}, R_{h})$ to $\cD$;
        }
    }
    \Output{Buffer $\cD$.}
\end{algorithm}

%% file: algo/algo_ppo.tex
\RestyleAlgo{ruled}
\LinesNumbered
\begin{algorithm}[h!]
    \caption{\ppo~\citep{schulman2017proximal}.}\label{alg:ppo}
    \SetKwInOut{Input}{Input}
    \SetKwInOut{Output}{Output}
    \small
    \Input{Initial parameters $\vtheta_{0}$ and $\vw_{0}$, iterations $K$, epochs per iteration $\widetilde{K}$, batch size $B$, environments $\{E_e\}_{e=1}^{N_E}$, GAE parameter $\lambda$, interaction length $H$, discount factor $\gamma$, loss coefficients $c_{1}$ and $c_{2}$, learning rate $\alpha$, clipping parameter $\epsilon$.}

    Initialize optimization variables: $\vtheta \leftarrow \vtheta_{0}$, $\vw \leftarrow \vw_{0}$.
    
    \For{$k \in \dsb{1,K}$}{
        \tcp{Data Collection}
        $\cD_{k} \leftarrow$ Run Algorithm~\ref{alg:data-collection};

        \tcp{Multi-Epoch Optimization}
        \For{$j \in \dsb{\widetilde{K}}$}{
            Shuffle $\cD_{k}$; 
            
            Create $\widetilde{\cD}_{k} = \{\cB_{z}\}_{z=1}^{N_{\cB}}$, where $N_{\cB} = \left\lceil \frac{|\cD_{k}|}{B} \right\rceil$, such that $\forall z \in \dsb{N_{\cB}}: |\cB_{z}| = B$;

            \For{$\cB \in \widetilde{\cD}_{k}$}{
                \tcp{Compute Loss}
                $$\widehat{L}^{\texttt{PPO}}(\vtheta, \vw, \cB) = \frac{1}{|\cB|} \sum_{i \in \cB} L^{\text{clip}}_i(\vtheta) - c_{1}L^{\text{vf}}_i(\vw) + c_{2} \cH(\pi_{\vtheta}(\cdot\mid \vs_{i})),$$
                where
                $$L^{\text{clip}}_i(\vtheta) = \min\left\{ \frac{\pi_{\vtheta}(\va_{i} \mid \vs_{i})}{\pi_{\vtheta_{i}}(\va_{i} \mid \vs_{i})} A_{i}, \mathrm{clip}\left( \frac{\pi_{\vtheta}(\va_{i} \mid \vs_{i})}{\pi_{\vtheta_{i}}(\va_{i} \mid \vs_{i})}, 1-\epsilon, 1+\epsilon \right)A_{i} \right\};$$ 
                $$L^{\text{vf}}_i(\vw) = \frac{1}{2} (V_{\vw}(\vs_{i}) - R_{i})^{2}.$$ 

                \tcp{Update}
                $$\vtheta \leftarrow \vtheta + \alpha \nabla_{\vtheta}\widehat{L}(\vtheta, \vw) \quad \text{and} \quad \vw \leftarrow \vw + \alpha \nabla_{\vw}\widehat{L}(\vtheta, \vw).$$
            }
        }
    }
    \Output{$\vtheta_{K}$.}
\end{algorithm}

%% file: appendix/apx_proofs.tex
\section{Proofs} \label{apx:proofs}

\begin{lemma}[Performance Difference Lemma~{\citep[Lemma 6.1,][]{kakade2002approximately}}] \label{lemma:pdl}
    For every $\vtheta_{1}, \vtheta_{2} \in \Theta$, the following holds:
    \begin{align*}
        J(\vtheta_{1}) - J(\vtheta_{2}) = \E_{\substack{\vs \sim d_{\vtheta_{1}} \\ \va \sim \pi_{\vtheta_{1}}(\cdot \mid \vs)}} \left[ A^{\pi_{\vtheta_{2}}}(\vs,\va) \right].
    \end{align*}
\end{lemma}

\begin{lemma}[Policy Improvement Lower Bound~{\citep[Corollary 1,][]{achiam2017constrained}}] \label{lemma:pi-lb}
    For every $\vtheta_{1}, \vtheta_{2} \in \Theta$, the following holds:
    \begin{align*}
        J(\vtheta_{1}) - J(\vtheta_{2}) \ge \E_{\substack{\vs \sim d_{\vtheta_{2}} \\ \va \sim \pi_{\vtheta_{2}}(\cdot \mid \vs)}}\left[ \frac{\pi_{\vtheta_{1}}(\va \mid \vs)}{\pi_{\vtheta_{2}}(\va \mid \vs)}A^{\pi_{\vtheta_{2}}}(\vs,\va) \right] - \frac{2 \gamma C(\vtheta_{1},\vtheta_{2})}{(1-\gamma)^{2}} \E_{\vs \sim d_{\vtheta_{2}}}\left[ D_{\mathrm{TV}}(\pi_{\vtheta_{1}}(\cdot \mid \vs), \pi_{\vtheta_{2}}(\cdot \mid \vs)) \right],
    \end{align*}
    where $C(\vtheta_{1},\vtheta_{2}) \coloneqq \sup_{\vs \in \cS} \left| \E_{\va \sim \pi_{\vtheta_{1}}(\cdot \mid \vs)}[A^{\pi_{\vtheta_{2}}}(\vs,\va)] \right|$. 
\end{lemma}

\begin{lemma} \label{lemma:stale-adv}
    For every $\vtheta,\vtheta_{1},\vtheta_{2} \in \Theta$, the following holds:
    \begin{align*}
        \left| \E_{\substack{\vs \sim d_{\vtheta_{2}} \\ \va \sim \pi_{\vtheta}(\cdot | \vs)}} \left[ A^{\pi_{\vtheta_{1}}}(\vs,\va) - A^{\pi_{\vtheta_{2}}}(\vs,\va) \right] \right| \le \frac{4 R_{\max} (1+\gamma)}{1-\gamma} \sup_{\vs \in \cS} D_{\mathrm{TV}}\left(\pi_{\vtheta_{1}}(\cdot \mid \vs) , \pi_{\vtheta_{2}}(\cdot \mid \vs) \right).
    \end{align*}
\end{lemma}
\begin{proof}
    To start the proof, we can simply note that:
    \begin{align*}
        &\left| \E_{\substack{\vs \sim d_{\vtheta_{2}} \\ \va \sim \pi_{\vtheta}(\cdot | \vs)}} \left[ A^{\pi_{\vtheta_{1}}}(\vs,\va) - A^{\pi_{\vtheta_{2}}}(\vs,\va) \right] \right| \\
        &\le \E_{\substack{\vs \sim d_{\vtheta_{2}} \\ \va \sim \pi_{\vtheta}(\cdot | \vs)}} \left[ \left| A^{\pi_{\vtheta_{1}}}(\vs,\va) - A^{\pi_{\vtheta_{2}}}(\vs,\va) \right| \right] \\
        &\le \underbrace{\E_{\substack{\vs \sim d_{\vtheta_{2}} \\ \va \sim \pi_{\vtheta}(\cdot | \vs)}} \left[ \left| Q^{\pi_{\vtheta_{1}}}(\vs,\va) - Q^{\pi_{\vtheta_{2}}}(\vs,\va) \right| \right]}_{\eqqcolon \textsf{A}} + \underbrace{\E_{\substack{\vs \sim d_{\vtheta_{2}} \\ \phantom{\va \sim \pi_{\vtheta}(\cdot | \vs)}}} \left[ \left| V^{\pi_{\vtheta_{1}}}(\vs) - V^{\pi_{\vtheta_{2}}}(\vs) \right| \right]}_{\eqqcolon \textsf{B}},
    \end{align*}
    where we employed the Jensen's inequality, then the advantage function definition, and finally we used the triangular inequality and the linearity of the expectation. Notice that in the \textsf{B} term we lose the action sampling from $\pi_{\vtheta}$.

    \emph{Bounding the Term \textsf{B}.}~~We now focus on term \textsf{B}. 
    Before starting with this part, let us define for every $\vs,\vs_{0} \in \cS$ and $\vtheta \in \Theta$ the following:
    \begin{align*}
        d_{\vtheta}^{\vs_{0}}(\vs) \coloneqq \sum_{t=1}^{T} \gamma^{t-1} \Pr (\vs_{t} = \vs \mid \vs_{1} = \vs_{0}, \pi_{\vtheta}),
    \end{align*}
    which is the $\gamma$-discounted state $\vs$ visitation under policy $\pi_{\vtheta}$ and starting from state $\vs_{0}$. Notice that it holds $d_{\vtheta}(\vs) = \E_{\vs_{0} \sim d_{0}}\left[ d_{\vtheta}^{\vs_{0}}(\vs) \right]$.
    Specifically, considering the argument of the expectation, we have the following:
    \begin{align*}
        V^{\pi_{\vtheta_{1}}}(\vs) - V^{\pi_{\vtheta_{2}}}(\vs) &= \E_{\substack{\vs' \sim d_{\vtheta_{1}}^{\vs}\\ \va \sim \pi_{\vtheta_{1}}(\cdot \mid \vs')}} \left[ A^{\pi_{\vtheta_{2}}}(\vs',\va) \right] \\
        &= \E_{\vs' \sim d_{\vtheta_{1}}^{\vs}} \left[ \E_{\va \sim \pi_{\vtheta_{1}}(\cdot \mid \vs')} \left[ A^{\pi_{\vtheta_{2}}}(\vs',\va) \right] - \E_{\va \sim \pi_{\vtheta_{2}}(\cdot \mid \vs')} \left[ A^{\pi_{\vtheta_{2}}}(\vs',\va) \right] \right] \\
        &= \E_{\vs' \sim d_{\vtheta_{1}}^{\vs}} \left[ \int (\pi_{\vtheta_{1}}(\va \mid \vs') - \pi_{\vtheta_{2}}(\va \mid \vs')) A^{\pi_{\vtheta_{2}}}(\vs',\va) \de \va \right],
    \end{align*}
    where we applied Lemma~\ref{lemma:pdl} (see also~\citep[Lemma 1.16,][]{agarwal2019reinforcement}) and then we exploited the fact that, for every $\vtheta \in \Theta$ and $\vs \in \cS$, it holds that $\E_{\va \sim \pi_{\vtheta}(\cdot \mid \vs)}\left[ A^{\pi_{\vtheta}}(\vs,\va) \right] = 0$ (as similarly done in~\citep[Lemma 2,][]{schulman2015trust}). 
    Thus, by considering the absolute value of the interested quantity, we have the following:
    \begin{align*}
        \left| V^{\pi_{\vtheta_{1}}}(\vs) - V^{\pi_{\vtheta_{2}}}(\vs) \right| 
        &\le \E_{\vs' \sim d_{\vtheta_{1}}^{\vs}} \left[ \int \left| \pi_{\vtheta_{1}}(\va \mid \vs') - \pi_{\vtheta_{2}}(\va \mid \vs') \right| \left| A^{\pi_{\vtheta_{2}}}(\vs',\va) \right| \de \va \right] \\
        &\le \frac{2 R_{\max}}{1-\gamma} \E_{\vs' \sim d_{\vtheta_{1}}^{\vs}} \left[ \int \left| \pi_{\vtheta_{1}}(\va \mid \vs') - \pi_{\vtheta_{2}}(\va \mid \vs') \right| \de \va \right] \\
        &= \frac{4 R_{\max}}{1-\gamma} \E_{\vs' \sim d_{\vtheta_{1}}^{\vs}} \left[ D_{\mathrm{TV}}(\pi_{\vtheta_{1}}(\cdot \mid \vs') , \pi_{\vtheta_{2}}(\cdot \mid \vs')) \right] \\
        &\le \frac{4 R_{\max}}{1-\gamma} \sup_{\vs \in \cS} D_{\mathrm{TV}}(\pi_{\vtheta_{1}}(\cdot \mid \vs) , \pi_{\vtheta_{2}}(\cdot \mid \vs)),
    \end{align*}
    for which we just employed the Jensen's inequality, the fact that, for every $\vtheta \in \Theta$, $\vs \in \cS$, and $\va \in \cA$, it holds that $|A^{\pi_{\vtheta}}(\vs,\va)| \le \frac{2 R_{\max}}{1-\gamma}$, and finally the TV definition.
    We can finally bound the \textsf{B} quantity, by exploiting this last obtained result. Specifically, we have:
    \begin{align*}
        \textsf{B} &= \E_{\vs \sim d_{\vtheta_{2}}} \left[ \left| V^{\pi_{\vtheta_{1}}}(\vs) - V^{\pi_{\vtheta_{2}}}(\vs) \right| \right] \\
        &\le \frac{4 R_{\max}}{1-\gamma} \sup_{\vs \in \cS} D_{\mathrm{TV}}(\pi_{\vtheta_{1}}(\cdot \mid \vs) , \pi_{\vtheta_{2}}(\cdot \mid \vs)).
    \end{align*}

    \emph{Bounding the Term \textsf{A}.}~~To bound the term \textsf{A}, we will exploit what done for the term \textsf{B}, by exploiting the following:
    \begin{align*}
        Q^{\pi_{\vtheta_{1}}}(\vs,\va) - Q^{\pi_{\vtheta_{2}}}(\vs,\va) = \gamma \int p(\vx \mid \vs,\va) \left( V^{\pi_{\vtheta_{1}}}(\vx) - V^{\pi_{\vtheta_{2}}}(\vx) \right) \de \vx,
    \end{align*}
    just by exploiting the action-value function definition.
    Thus, we have the following:
    \begin{align*}
        \textsf{A} &= \E_{\substack{\vs \sim d_{\vtheta_{2}} \\ \va \sim \pi_{\vtheta}(\cdot | \vs)}} \left[ \left| Q^{\pi_{\vtheta_{1}}}(\vs,\va) - Q^{\pi_{\vtheta_{2}}}(\vs,\va) \right| \right] \\
        &= \gamma \E_{\substack{\vs \sim d_{\vtheta_{2}} \\ \va \sim \pi_{\vtheta}(\cdot | \vs)}} \left[ \E_{\vx \sim p(\cdot \mid \vs,\va)} \left[ \left| V^{\pi_{\vtheta_{1}}}(\vx) - V^{\pi_{\vtheta_{2}}}(\vx) \right| \right] \right] \\
        &\le \frac{4 \gamma R_{\max}}{1-\gamma} \sup_{\vs \in \cS} D_{\mathrm{TV}}(\pi_{\vtheta_{1}}(\cdot \mid \vs) , \pi_{\vtheta_{2}}(\cdot \mid \vs)).
    \end{align*}

    \emph{Final Result.}~~Putting all together, we can conclude that:
    \begin{align*}
        \left| \E_{\substack{\vs \sim d_{\vtheta_{2}} \\ \va \sim \pi_{\vtheta}(\cdot | \vs)}} \left[ A^{\pi_{\vtheta_{1}}}(\vs,\va) - A^{\pi_{\vtheta_{2}}}(\vs,\va) \right] \right| \le \frac{4 R_{\max} (1+\gamma)}{1-\gamma} \sup_{\vs \in \cS} D_{\mathrm{TV}}(\pi_{\vtheta_{1}}(\cdot \mid \vs) , \pi_{\vtheta_{2}}(\cdot \mid \vs)) .
    \end{align*}
\end{proof}

\begin{thr} \label{thr:pi-lb-stale}
    Consider to have access to $\omega \in \mathbb{N}_{\ge 1}$ policies $\Pi_{\omega} = \{\pi_{\vtheta_{i}}\}_{i=1}^{\omega}$ and to a distribution $\nu$ over the policies into $\Pi_{\omega}$. Then, for every $\vtheta,\tar \in \Theta$, the following holds:
    \begin{align*}
        J(\vtheta) - J(\tar) &\ge \E_{i \sim \nu}\left[ \E_{\substack{\vs \sim d_{\vtheta_{i}} \\ \va \sim \pi_{\vtheta_{i}}(\cdot \mid \vs)}} \left[ \frac{\pi_{\vtheta}(\va \mid \vs)}{\pi_{\vtheta_{i}}(\va \mid \vs)} A^{\pi_{\vtheta_{i}}}(\vs,\va) \right] \right] \\
        &\quad - \frac{2 \gamma C(\vtheta, \tar)}{(1-\gamma)^{2}} \E_{i \sim \nu}\left[ \E_{\vs \sim d_{\vtheta_{i}} } \left[ D_{\mathrm{TV}}\left( \pi_{\vtheta}(\cdot \mid \vs), \pi_{\vtheta_{i}}(\cdot \mid \vs) \right) \right] \right] \\
        &\quad - \frac{4 R_{\max} (1+\gamma)}{1-\gamma} \E_{i \sim \nu} \left[ \sup_{\vs \in \cS} D_{\mathrm{TV}}\left( \pi_{\tar}(\cdot \mid \vs), \pi_{\vtheta_{i}}(\cdot \mid \vs) \right) \right].
    \end{align*}
\end{thr}
\begin{proof}
    We start the proof from Lemma~\ref{lemma:pdl}~\citep{kakade2002approximately}. For ever $\vtheta,\tar \in \Theta$ the following holds:
    \begin{align*}
        J(\vtheta) - J(\tar) &= \E_{\substack{\vs \sim d_{\vtheta} \\ \va \sim \pi_{\vtheta}(\cdot \mid \vs)}} \left[ A^{\pi_{\vtheta}}(\vs,\va) \right] \\
        &= \E_{\substack{\vs \sim d_{\tar} \\ \va \sim \pi_{\tar}(\cdot \mid \vs)}} \left[ \frac{d_{\vtheta}(\vs) \pi_{\vtheta}(\va \mid \vs)}{d_{\tar}(\vs) \pi_{\tar}(\va \mid \vs)} A^{\pi_{\vtheta}}(\vs,\va) \right].
    \end{align*}
    Now, by following the same procedure as Lemma~\ref{lemma:pi-lb}~\citep{pirotta2013safe,achiam2017constrained} and~\citep[Lemma 7,][]{queeney2021generalized}, for every $\pi_{\vtheta_{i}} \in \Pi_{\omega}$ we have the following:
    \begin{align*}
        J(\vtheta) - J(\tar) \ge \underbrace{\E_{\substack{\vs \sim d_{\vtheta_{i}} \\ \va \sim \pi_{\vtheta_{i}}(\cdot \mid \vs)}}\left[ \frac{\pi_{\vtheta}(\va \mid \vs)}{\pi_{\vtheta_{i}}(\va \mid \vs)} A^{\pi_{\tar}}(\vs,\va) \right]}_{\eqqcolon \textsf{A}} - \frac{2 \gamma C(\vtheta,\tar)}{(1-\gamma)^{2}} \E_{\vs \sim d_{\vtheta_{i}}} \left[ D_{\mathrm{TV}}\left( \pi_{\vtheta}(\cdot \mid \vs), \pi_{\vtheta_{i}}(\cdot \mid \vs)\right) \right],
    \end{align*}
    where $C(\vtheta,\tar) \coloneqq \sup_{\vs \in \cS} \left| \E_{\va \sim \pi_{\vtheta}(\cdot \mid \vs)}[A^{\pi_{\tar}}(\vs,\va)] \right|$ is the same as in Lemma~\ref{lemma:pi-lb}, depending just on the policies appearing on the left-hand side of the policy improvement lower bound. 
    Stopping here, we would have the policy improvement bound of~\citep{queeney2021generalized,meng2023off,zhao2023policy,liu2026enhancing} which would require to have access to advantages relative to policy $\pi_{\tar}$ in every sampled state-action pairs from policies in $\Pi_{\omega}$. To fix this, we would leverage Lemma~\ref{lemma:stale-adv}.\footnote{A similar need is faced in~\citep[Lemma 3,][]{liu2026enhancing} to change the appearing advantage, however, it is not the one related to the data-sampling policy $\pi_{i}$.}

    To continue the proof, let us sum and subtract $A^{\pi_{\vtheta_{i}}}(\vs,\va)$ into the argument of the expectation of term \textsf{A}:
    \begin{align*}
        \textsf{A} &= \E_{\substack{\vs \sim d_{\vtheta_{i}} \\ \va \sim \pi_{\vtheta_{i}}(\cdot \mid \vs)}}\left[ \frac{\pi_{\vtheta}(\va \mid \vs)}{\pi_{\vtheta_{i}}(\va \mid \vs)} A^{\pi_{\tar}}(\vs,\va) \right] \\
        &= \E_{\substack{\vs \sim d_{\vtheta_{i}} \\ \va \sim \pi_{\vtheta_{i}}(\cdot \mid \vs)}}\left[ \frac{\pi_{\vtheta}(\va \mid \vs)}{\pi_{\vtheta_{i}}(\va \mid \vs)} A^{\pi_{\vtheta_{i}}}(\vs,\va) \right] + \underbrace{\E_{\substack{\vs \sim d_{\vtheta_{i}} \\ \va \sim \pi_{\vtheta_{i}}(\cdot \mid \vs)}}\left[ \frac{\pi_{\vtheta}(\va \mid \vs)}{\pi_{\vtheta_{i}}(\va \mid \vs)} \left( A^{\pi_{\tar}}(\vs,\va) - A^{\pi_{\vtheta_{i}}}(\vs,\va) \right) \right]}_{\eqqcolon \textsf{B}}.
    \end{align*}
    Now, we just need to bound term \textsf{B}, that can be easily done thanks to Lemma~\ref{lemma:stale-adv}. Specifically, we have the following:
    \begin{align*}
        \textsf{B} &= \E_{\substack{\vs \sim d_{\vtheta_{i}} \\ \va \sim \pi_{\vtheta_{i}}(\cdot \mid \vs)}}\left[ \frac{\pi_{\vtheta}(\va \mid \vs)}{\pi_{\vtheta_{i}}(\va \mid \vs)} \left( A^{\pi_{\tar}}(\vs,\va) - A^{\pi_{\vtheta_{i}}}(\vs,\va) \right) \right] \\
        &= \E_{\substack{\vs \sim d_{\vtheta_{i}} \\ \va \sim \pi_{\vtheta}(\cdot \mid \vs)}}\left[ A^{\pi_{\tar}}(\vs,\va) - A^{\pi_{\vtheta_{i}}}(\vs,\va) \right].
    \end{align*}
    By noting that $\textsf{B} \ge -|\textsf{B}|$ and then applying Lemma~\ref{lemma:stale-adv}, we can conclude that:
    \begin{align*}
        J(\vtheta) - J(\tar) &\ge \E_{\substack{\vs \sim d_{\vtheta_{i}} \\ \va \sim \pi_{\vtheta_{i}}(\cdot \mid \vs)}}\left[ \frac{\pi_{\vtheta}(\va \mid \vs)}{\pi_{\vtheta_{i}}(\va \mid \vs)} A^{\pi_{\vtheta_{i}}}(\vs,\va) \right] - \frac{2 \gamma C(\vtheta,\tar)}{(1-\gamma)^{2}} \E_{\vs \sim d_{\vtheta_{i}}} \left[ D_{\mathrm{TV}}\left( \pi_{\vtheta}(\cdot \mid \vs), \pi_{\vtheta_{i}}(\cdot \mid \vs)\right) \right] \\
        &\quad - \frac{4 R_{\max} (1+\gamma)}{1-\gamma} \sup_{\vs \in \cS} D_{\mathrm{TV}}\left( \pi_{\tar}(\cdot \mid \vs), \pi_{\vtheta_{i}}(\cdot \mid \vs) \right).
    \end{align*}

    Finally, as done in~\citep[Theorem 1,][]{queeney2021generalized}, since the previous bound is holding for every $\pi_{\vtheta_{i}} \in \Pi_{\omega}$, then for any choice of distribution $\nu$ over the policies in $\Pi_{\omega}$ we have the following:
    \begin{align*}
        J(\vtheta) - J(\tar) &\ge \E_{i \sim \nu} \left[ \E_{\substack{\vs \sim d_{\vtheta_{i}} \\ \va \sim \pi_{\vtheta_{i}}(\cdot \mid \vs)}}\left[ \frac{\pi_{\vtheta}(\va \mid \vs)}{\pi_{\vtheta_{i}}(\va \mid \vs)} A^{\pi_{\vtheta_{i}}}(\vs,\va) \right] \right] \\
        &\quad - \frac{2 \gamma C(\vtheta,\tar)}{(1-\gamma)^{2}} \E_{i \sim \nu} \left[ \E_{\vs \sim d_{\vtheta_{i}}} \left[ D_{\mathrm{TV}}\left( \pi_{\vtheta}(\cdot \mid \vs), \pi_{\vtheta_{i}}(\cdot \mid \vs)\right) \right] \right] \\
        &\quad -  \frac{4 R_{\max} (1+\gamma)}{1-\gamma} \E_{i \sim \nu} \left[ \sup_{\vs \in \cS} D_{\mathrm{TV}}\left( \pi_{\tar}(\cdot \mid \vs), \pi_{\vtheta_{i}}(\cdot \mid \vs) \right) \right].
    \end{align*}
\end{proof}

\uniformBound*
\begin{proof}
    The proof straightforwardly follows from Theorem~\ref{thr:pi-lb-stale}. In particular, setting $\Pi_{\omega} =\{\pi_{\omega_{k-i}}\}_{i=0}^{\omega_{k}-1}$, $\tar = \vtheta_{k}$, and by considering a uniform distribution $\nu$ over the policies in the window. Thus, we have the following:
    \begin{align*}
        J(\vtheta) - J(\vtheta_{k}) &\ge \frac{1}{\omega_{k}} \sum_{i=0}^{\omega_{k}-1} \E_{\substack{\vs \sim d_{\vtheta_{k-i}} \\ \va \sim \pi_{\vtheta_{k-i}}(\cdot \mid \vs)}}\left[ \frac{\pi_{\vtheta}(\va \mid \vs)}{\pi_{\vtheta_{k-i}}(\va \mid \vs)} A^{\pi_{\vtheta_{k-i}}}(\vs,\va) \right] \\
        &\quad - \frac{2 \gamma C(\vtheta,\vtheta_{k})}{(1-\gamma)^{2} \omega_{k}} \sum_{i=0}^{\omega_{k}-1} \E_{\vs \sim d_{\vtheta_{k-i}}} \left[ D_{\mathrm{TV}}\left( \pi_{\vtheta}(\cdot \mid \vs), \pi_{\vtheta_{k-i}}(\cdot \mid \vs)\right) \right] \\
        &\quad -  \frac{4 R_{\max} (1+\gamma)}{(1-\gamma) \omega_{k}} \sum_{i=0}^{\omega_{k}-1} \sup_{\vs \in \cS} D_{\mathrm{TV}}\left( \pi_{\vtheta_{k}}(\cdot \mid \vs), \pi_{\vtheta_{k-i}}(\cdot \mid \vs) \right) .
    \end{align*}
\end{proof}

\bhBound*
\begin{proof}
    To start the proof, let us start from Lemma~\ref{lemma:pdl}:
    \begin{align*}
        &J(\vtheta) - J(\vtheta_{k}) \\ 
        &=  \E_{\substack{\vs \sim d_{\vtheta} \\ \va \sim \pi_{\vtheta}(\cdot \mid \vs)}} \left[ A^{\pi_{\vtheta_{k}}}(\vs,\va) \right] \pm \E_{\substack{\vs \sim  d_{\omega_{k}}^{\mathrm{BH}} \\ \va \sim \pi_{\vtheta}(\cdot \mid \vs)}} \left[ A^{\pi_{\vtheta_{k}}}(\vs,\va) \right]\\
        &= \E_{\substack{\vs \sim d_{\omega_{k}}^{\mathrm{BH}} \\ \va \sim \pi_{\omega_{k}}^{\mathrm{BH}} (\cdot \mid \vs)}} \left[ \frac{\pi_{\vtheta}(\va \mid \vs)}{\pi_{\omega_{k}}^{\mathrm{BH}}(\va \mid \vs)} A^{\pi_{\vtheta_{k}}}(\vs,\va) \right] - \int \left( d_{\omega_{k}}^{\mathrm{BH}}(\vs) - d_{\vtheta}(\vs) \right) \E_{\va \sim \pi_{\vtheta}(\cdot \mid \vs)} \left[A^{\pi_{\vtheta_{k}}}(\vs,\va)\right] \de \vs,
    \end{align*}
    where we just changed, applying the needed importance weight correction, the state and action sampling distributions.

    Now, let us introduce the BH state-action occupancy. Specifically, for every $\vs \in \cS$ and $\va \in \cA$, let $\mu_{\omega_{k}}^{\mathrm{BH}}(\vs,\va) \coloneqq \frac{1}{\omega_{k}} \sum_{i=0}^{\omega_{k}-1} d_{\vtheta_{k-i}}(\vs) \pi_{\vtheta_{k-i}}(\va \mid \vs)$. Notice that $d^{\mathrm{BH}}_{\omega_k}$ is the discounted state occupancy induced by $\mu^{\mathrm{BH}}_{\omega_k}$ and that:
    \begin{align*}
        \int \mu_{\omega_{k}}^{\mathrm{BH}}(\vs,\va) \de \va = \frac{1}{\omega_{k}} \sum_{i=0}^{\omega_{k}-1} \int d_{\vtheta_{k-i}}(\vs) \pi_{\vtheta_{k-i}}(\va \mid \vs) \de \va = d_{\omega_{k}}^{\mathrm{BH}}(\vs).
    \end{align*}
    Letting $\mu_{\omega_{k}}^{\mathrm{BH}}(\va \mid \vs) \coloneqq \mu_{\omega_{k}}^{\mathrm{BH}}(\vs,\va) / d_{\omega_{k}}^{\mathrm{BH}}(\vs)$ 
    we note that, in general, we have the following:
    \begin{align*}
        \mu_{\omega_{k}}^{\mathrm{BH}}(\va \mid \vs) = \frac{\sum_{i=0}^{\omega_k-1} d_{\vtheta_{k-i}}(\vs) \pi_{\vtheta_{k-i}}(\va \mid \vs)}{\sum_{i=0}^{\omega_k-1} d_{\vtheta_{k-i}}(\vs)} \neq \frac{1}{\omega_k} \sum_{i=0}^{\omega_k-1} \pi_{\vtheta_{k-i}}(\va \mid \vs) = \pi_{\omega_{k}}^{\mathrm{BH}} (\va \mid \vs).
    \end{align*}

    That being said, by applying the argument of~ \citep{achiam2017constrained}, the following holds:
    \begin{align*}
        &J(\vtheta) - J(\vtheta_{k}) \\ 
        &= \E_{\substack{\vs \sim d_{\omega_{k}}^{\mathrm{BH}} \\ \va \sim \pi_{\omega_{k}}^{\mathrm{BH}} (\cdot \mid \vs)}} \left[ \frac{\pi_{\vtheta}(\va \mid \vs)}{\pi_{\omega_{k}}^{\mathrm{BH}}(\va \mid \vs)} A^{\pi_{\vtheta_{k}}}(\vs,\va) \right] - \int \left( d_{\omega_{k}}^{\mathrm{BH}}(\vs) - d_{\vtheta}(\vs) \right) \E_{\va \sim \pi_{\vtheta}(\cdot \mid \vs)} \left[A^{\pi_{\vtheta_{k}}}(\vs,\va)\right] \de \vs \\
        &\ge \E_{\substack{\vs \sim d_{\omega_{k}}^{\mathrm{BH}} \\ \va \sim \pi_{\omega_{k}}^{\mathrm{BH}} (\cdot \mid \vs)}} \left[ \frac{\pi_{\vtheta}(\va \mid \vs)}{\pi_{\omega_{k}}^{\mathrm{BH}}(\va \mid \vs)} A^{\pi_{\vtheta_{k}}}(\vs,\va) \right] - \frac{2 \gamma C(\vtheta,\vtheta_{k})}{(1-\gamma)^{2}} \E_{\vs \sim d_{\omega_{k}}^{\mathrm{BH}}} \left[ D_{\mathrm{TV}}\left( \pi_{\vtheta}(\cdot \mid \vs), \mu_{\omega_{k}}^{\mathrm{BH}} (\cdot \mid \vs) \right) \right] \\
        &= \E_{\substack{\vs \sim d_{\omega_{k}}^{\mathrm{BH}} \\ \va \sim \pi_{\omega_{k}}^{\mathrm{BH}} (\cdot \mid \vs)}} \left[ \frac{\pi_{\vtheta}(\va \mid \vs)}{\pi_{\omega_{k}}^{\mathrm{BH}}(\va \mid \vs)} A^{\pi_{\vtheta_{k}}}(\vs,\va) \right] - \frac{2 \gamma C(\vtheta,\vtheta_{k})}{(1-\gamma)^{2} \omega_{k}} \sum_{i=0}^{\omega_{k}-1} \E_{\vs \sim d_{\vtheta_{k-i}}} \left[ D_{\mathrm{TV}}\left( \pi_{\vtheta}(\cdot \mid \vs), \mu_{\omega_{k}}^{\mathrm{BH}} (\cdot \mid \vs) \right) \right].
    \end{align*}

    Now, by triangular inequality we have that: 
    \begin{align*}
    D_{\mathrm{TV}}(\pi_{\vtheta}(\cdot \mid \vs) , \mu^{\mathrm{BH}}_{\omega_k}(\cdot \mid \vs) ) \le D_{\mathrm{TV}}(\pi_{\vtheta}(\cdot \mid \vs) , \pi^{\mathrm{BH}}_{\omega_k}(\cdot \mid \vs) ) + D_{\mathrm{TV}}(\pi^{\mathrm{BH}}_{\omega_k}(\cdot \mid \vs) , \mu^{\mathrm{BH}}_{\omega_k}(\cdot \mid \vs) ).
    \end{align*}
    Furthermore, note that:
    \begin{align}
        \left| \pi^{\mathrm{BH}}_{\omega_{k}}(\va \mid \vs) - \mu^{\mathrm{BH}}_{\omega_{k}}(\va \mid \vs) \right| &= \frac{1}{\omega_{k}} \left| \sum_{i=0}^{\omega_{k}-1} \pi_{\vtheta_{k-i}}(\va \mid \vs) - \frac{d_{\vtheta_{k-i}}(\vs) \pi_{\vtheta_{k-i}}(\va \mid \vs)}{d^{\mathrm{BH}}_{\omega_{k}}(\vs)} \right| \nonumber \\
        &\le \frac{1}{\omega_{k}} \sum_{i=0}^{\omega_{k}-1} \left| 1 - \frac{d_{\vtheta_{k-i}}(\vs)}{d^{\mathrm{BH}}_{\omega_{k}}(\vs)} \right| \pi_{\vtheta_{k-i}}(\va \mid \vs) \nonumber \\
        &\le \frac{1}{\omega_{k}} \sum_{i=0}^{\omega_{k}-1} \frac{\left| d^{\mathrm{BH}}_{\omega_{k}}(\vs) - d_{\vtheta_{k-i}}(\vs) \right|}{d^{\mathrm{BH}}_{\omega_{k}}(\vs)} \pi_{\vtheta_{k-i}}(\va \mid \vs), \label{eq:use-me}
    \end{align}
    which simply follows from the triangular inequality and from algebraic manipulations taking into account the $d_{\omega_{k}}^{\mathrm{BH}}$ definition. Equipped with these two arguments, we have that:
    \begin{align*}
        &\E_{\vs \sim d^{\mathrm{BH}}_{\omega_k}}\left[ D_{\mathrm{TV}}(\pi^{\mathrm{BH}}_{\omega_k}(\cdot \mid \vs), \mu^{\mathrm{BH}}_{\omega_k}(\cdot \mid \vs)) \right] \\
        &\le \frac{1}{2\omega_k} \sum_{i=0}^{\omega_k-1} \int_{\cS} d^{\mathrm{BH}}_{\omega_k}(\vs) \frac{|d^{\mathrm{BH}}_{\omega_k}(\vs) - d_{\vtheta_{k-i}}(\vs)|}{d^{\mathrm{BH}}_{\omega_k}(\vs)} \int_{\cA} \pi_{\vtheta_{k-i}}(\va \mid \vs)\, \mathrm{d}\va\, \mathrm{d}\vs \\
        &= \frac{1}{\omega_k} \sum_{i=0}^{\omega_k-1} D_{\mathrm{TV}}(d^{\mathrm{BH}}_{\omega_k}, d_{\vtheta_{k-i}}),
    \end{align*}
    thus leading to the following intermediate form of the performance improvement lower bound we aim to reach:
    \begin{align}
        J(\vtheta) - J(\vtheta_{k}) &\ge \E_{\substack{\vs \sim d_{\omega_{k}}^{\mathrm{BH}} \\ \va \sim \pi_{\omega_{k}}^{\mathrm{BH}} (\cdot \mid \vs)}} \left[ \frac{\pi_{\vtheta}(\va \mid \vs)}{\pi_{\omega_{k}}^{\mathrm{BH}}(\va \mid \vs)} A^{\pi_{\vtheta_{k}}}(\vs,\va) \right] \nonumber \\
        &\quad - \frac{2 \gamma C(\vtheta,\vtheta_{k})}{(1-\gamma)^{2} \omega_{k}} \sum_{i=0}^{\omega_{k}-1} \E_{\vs \sim d_{\vtheta_{k-i}}} \left[ D_{\mathrm{TV}}\left( \pi_{\vtheta}(\cdot \mid \vs), \pi_{\omega_{k}}^{\mathrm{BH}} (\cdot \mid \vs) \right) \right] \label{eq:bh-pi-lb-stale-1} \\
        &\quad - \frac{2 \gamma C(\vtheta,\vtheta_{k})}{(1-\gamma)^{2} \omega_{k}} \sum_{i=0}^{\omega_{k}-1} D_{\mathrm{TV}}\left( d^{\mathrm{BH}}_{\omega_k}, d_{\vtheta_{k-i}} \right). \nonumber
    \end{align}
    
    In order to carry on with the derivation, we introduce the following quantity:
    \begin{align*}
        \textsf{A} &\coloneqq \E_{(\vs,\va) \sim \mu_{\omega_{k}}^{\mathrm{BH}}} \left[ \frac{\pi_{\vtheta}(\va \mid \vs)}{ \pi_{\omega_{k}}^{\mathrm{BH}}(\va \mid \vs)} A^{\pi_{\vtheta_{k}}}(\vs,\va) \right] \\
        &= \frac{1}{\omega_{k}} \sum_{i=0}^{\omega_{k}-1} \iint d_{\vtheta_{k-i}}(\vs) \pi_{\vtheta_{k-i}}(\va \mid \vs) \frac{\pi_{\vtheta}(\va \mid \vs)}{ \pi_{\omega_{k}}^{\mathrm{BH}}(\va \mid \vs)} A^{\pi_{\vtheta_{k}}}(\vs,\va) \de \vs \de \va \\
        &= \frac{1}{\omega_{k}} \sum_{i=0}^{\omega_{k}-1} \E_{\substack{\vs \sim d_{\vtheta_{k-i}} \\ \va \sim \pi_{\vtheta_{k-i}}(\cdot \mid \vs)}} \left[ \frac{\pi_{\vtheta}(\va \mid \vs)}{ \pi_{\omega_{k}}^{\mathrm{BH}}(\va \mid \vs)} A^{\pi_{\vtheta_{k}}}(\vs,\va) \right],
    \end{align*}
    where we just exploited the $\mu_{\omega_{k}}^{\mathrm{BH}}$ definition. 
    
    Notice that the next steps are to make \textsf{A} appear into Equation~\eqref{eq:bh-pi-lb-stale-1}, then employing Lemma~\ref{lemma:stale-adv} to finally have the stale advantages. That being said, we proceed by adding and subtracting \textsf{A} into Equation~\eqref{eq:bh-pi-lb-stale-1}, obtaining the following:
    \begin{align}
        J(\vtheta) - J(\vtheta_{k}) &\ge \frac{1}{\omega_{k}} \sum_{i=0}^{\omega_{k}-1} \E_{\substack{\vs \sim d_{\vtheta_{k-i}} \\ \va \sim \pi_{\vtheta_{k-i}}(\cdot \mid \vs)}} \left[ \frac{\pi_{\vtheta}(\va \mid \vs)}{ \pi_{\omega_{k}}^{\mathrm{BH}}(\va \mid \vs)} A^{\pi_{\vtheta_{k}}}(\vs,\va) \right] \nonumber \\
        &\quad - \frac{2 \gamma C(\vtheta,\vtheta_{k})}{(1-\gamma)^{2} \omega_{k}} \sum_{i=0}^{\omega_{k}-1} \E_{\vs \sim d_{\vtheta_{k-i}}} \left[ D_{\mathrm{TV}}\left( \pi_{\vtheta}(\cdot \mid \vs), \pi_{\omega_{k}}^{\mathrm{BH}} (\cdot \mid \vs) \right) \right] \label{eq:bh-pi-lb-stale-2} \\
        &\quad - \frac{2 \gamma C(\vtheta,\vtheta_{k})}{(1-\gamma)^{2} \omega_{k}} \sum_{i=0}^{\omega_{k}-1} D_{\mathrm{TV}}\left( d^{\mathrm{BH}}_{\omega_k}, d_{\vtheta_{k-i}} \right) \nonumber \\
        &\quad + \underbrace{\E_{\substack{\vs \sim d_{\omega_{k}}^{\mathrm{BH}} \\ \va \sim \pi_{\omega_{k}}^{\mathrm{BH}} (\cdot \mid \vs)}} \left[ \frac{\pi_{\vtheta}(\va \mid \vs)}{ \pi_{\omega_{k}}^{\mathrm{BH}}(\va \mid \vs)} A^{\pi_{\vtheta_{k}}}(\vs,\va) \right] - \E_{(\vs,\va) \sim \mu_{\omega_{k}}^{\mathrm{BH}}} \left[ \frac{\pi_{\vtheta}(\va \mid \vs)}{ \pi_{\omega_{k}}^{\mathrm{BH}}(\va \mid \vs)} A^{\pi_{\vtheta_{k}}}(\vs,\va) \right]}_{\eqqcolon \textsf{B}}. \nonumber
    \end{align}

    Next, we have to treat \textsf{B}. To do so, we can exploit the fact that $\int \mu_{\omega_{k}}^{\mathrm{BH}}(\vs,\va) \de \va = d_{\omega_{k}}^{\mathrm{BH}}(\vs)$ for every $\vs \in \cS$. Then, the following holds:
    \begin{align*}
        \textsf{B} &= \E_{\substack{\vs \sim d^{\mathrm{BH}}_{\omega_{k}} \\ \va \sim \pi^{\mathrm{BH}}_{\omega_{k}}(\cdot \mid \vs)}} \left[ \frac{\pi_{\vtheta}(\va \mid \vs)}{\pi^{\mathrm{BH}}_{\omega_{k}}(\va \mid \vs)} A^{\pi_{\vtheta_{k}}}(\vs,\va) \right] - \E_{(\vs,\va) \sim \mu^{\mathrm{BH}}_{\omega_{k}}} \left[ \frac{\pi_{\vtheta}(\va \mid \vs)}{\pi^{\mathrm{BH}}_{\omega_{k}}(\va \mid \vs)} A^{\pi_{\vtheta_{k}}}(\vs,\va) \right] \\
        &= \E_{\vs \sim d^{\mathrm{BH}}_{\omega_{k}}} \left[ \E_{\va \sim \pi^{\mathrm{BH}}_{\omega_{k}}(\cdot \mid \vs)} \left[ \frac{\pi_{\vtheta}(\va \mid \vs)}{\pi^{\mathrm{BH}}_{\omega_{k}}(\va \mid \vs)} A^{\pi_{\vtheta_{k}}}(\vs,\va) \right] - \E_{\va \sim \mu^{\mathrm{BH}}_{\omega_{k}}(\cdot \mid \vs)} \left[ \frac{\pi_{\vtheta}(\va \mid \vs)}{\pi^{\mathrm{BH}}_{\omega_{k}}(\va \mid \vs)} A^{\pi_{\vtheta_{k}}}(\vs,\va) \right] \right] \\
        &= \E_{\vs \sim d^{\mathrm{BH}}_{\omega_{k}}} \left[ \int (\pi^{\mathrm{BH}}_{\omega_{k}}(\va \mid \vs) - \mu^{\mathrm{BH}}_{\omega_{k}}(\va \mid \vs)) \frac{\pi_{\vtheta}(\va \mid \vs)}{\pi^{\mathrm{BH}}_{\omega_{k}}(\va \mid \vs)} A^{\pi_{\vtheta_{k}}}(\vs,\va) \de \va \right].
    \end{align*}
    Because of the presence of the BH-corrected ratio $\frac{\pi_{\vtheta}(\va \mid \vs)}{\pi^{\text{BH}}_{\omega_{k}}(\va \mid \vs)}$, we cannot directly bound $|\textsf{B}|$ with a standard TV argument. Thus, we first note that:
    \begin{align*}
        \textsf{B} &\le \E_{\vs \sim d^{\mathrm{BH}}_{\omega_{k}} } \left[ \int \left| \pi^{\mathrm{BH}}_{\omega_{k}}(\va \mid \vs) - \mu^{\mathrm{BH}}_{\omega_{k}}(\va \mid \vs) \right| \frac{\pi_{\vtheta}(\va \mid \vs)}{\pi^{\mathrm{BH}}_{\omega_{k}}(\va \mid \vs)} \left| A^{\pi_{\vtheta_{k}}}(\vs,\va) \right| \de \va \right] \\
        &\le \frac{2 R_{\max}}{1-\gamma} \E_{\substack{\vs \sim d^{\mathrm{BH}}_{\omega_{k}} \\ \va \sim \pi_{\vtheta}(\cdot \mid \vs)}} \left[ \frac{\left| \pi^{\mathrm{BH}}_{\omega_{k}}(\va \mid \vs) - \mu^{\mathrm{BH}}_{\omega_{k}}(\va \mid \vs) \right|}{\pi^{\mathrm{BH}}_{\omega_{k}}(\va \mid \vs)} \right],
    \end{align*}
    where we exploited the fact that for every $\vs \in \cS$, $\va \in \cA$, and $\vtheta \in \Theta$ it holds that $|A^{\pi_{\vtheta}}(\vs,\va)| \le \frac{2 R_{\max}}{1-\gamma}$ and we used $\pi_{\vtheta}$ as an action-sampling distribution.

    Exploiting Equation~\eqref{eq:use-me}, we can finally conclude the $|\textsf{B}|$ bound as:
    \begin{align*}
        \textsf{B} &\le \frac{2 R_{\max}}{1-\gamma} \E_{\substack{\vs \sim d^{\mathrm{BH}}_{\omega_{k}} \\ \va \sim \pi_{\vtheta}(\cdot \mid \vs)}} \left[ \frac{\left| \pi^{\mathrm{BH}}_{\omega_{k}}(\va \mid \vs) - \mu^{\mathrm{BH}}_{\omega_{k}}(\va \mid \vs) \right|}{\pi^{\mathrm{BH}}_{\omega_{k}}(\va \mid \vs)} \right] \\
        &\le \frac{2 R_{\max}}{(1-\gamma) \omega_{k}} \sum_{i=0}^{\omega_{k}-1} \E_{\substack{\vs \sim d^{\mathrm{BH}}_{\omega_{k}} \\ \va \sim \pi_{\vtheta}(\cdot \mid \vs)}} \left[ \frac{\left| d^{\mathrm{BH}}_{\omega_{k}}(\vs) - d_{\vtheta_{k-i}}(\vs) \right|}{d^{\mathrm{BH}}_{\omega_{k}}(\vs)} \frac{\pi_{\vtheta_{k-i}}(\va \mid \vs)}{\pi^{\mathrm{BH}}_{\omega_{k}}(\va \mid \vs)} \right] \\
        &\le \frac{2 R_{\max}}{(1-\gamma)} \sum_{i=0}^{\omega_{k}-1} \E_{\vs \sim d^{\mathrm{BH}}_{\omega_{k}}} \left[ \frac{\left| d^{\mathrm{BH}}_{\omega_{k}}(\vs) - d_{\vtheta_{k-i}}(\vs) \right|}{d^{\mathrm{BH}}_{\omega_{k}}(\vs)} \right] \\
        &= \frac{4 R_{\max}}{(1-\gamma)} \sum_{i=0}^{\omega_{k}-1} D_{\mathrm{TV}}\left( d^{\mathrm{BH}}_{\omega_{k}} , d_{\vtheta_{k-i}} \right),
    \end{align*}
    where: $(i)$ we applied the previous bound on $\left| \pi^{\mathrm{BH}}_{\omega_{k}}(\va \mid \vs) - \mu^{\mathrm{BH}}_{\omega_{k}}(\va \mid \vs) \right|$ from Equation~\eqref{eq:use-me}; $(ii)$ we exploited the fact that $\frac{\pi_{\vtheta_{k-i}}(\va \mid \vs)}{\pi^{\mathrm{BH}}_{\omega_{k}}(\va \mid \vs)} \le \omega_{k}$ for every $\vs \in \cS$, $\va \in \cA$, and $i \in \dsb{0,\omega_{k}-1}$; and $(iii)$ we exploited the TV definition. Notice that, when exploiting the specific structure of the BH importance weight with $\pi_{\vtheta_{k-i}}$ at the numerator, we lose the dependence on $\pi_{\vtheta}$ in the inner expectation term.
    
    We can now plug this result into Equation~\eqref{eq:bh-pi-lb-stale-2}, obtaining:
    \begin{align}
        J(\vtheta) - J(\vtheta_{k}) &\ge \frac{1}{\omega_{k}} \sum_{i=0}^{\omega_{k}-1} \overbrace{\E_{\substack{\vs \sim d_{\vtheta_{k-i}} \\ \va \sim \pi_{\vtheta_{k-i}}(\cdot \mid \vs)}} \left[ \frac{\pi_{\vtheta}(\va \mid \vs)}{\pi^{\mathrm{BH}}_{\omega_{k}}(\va \mid \vs)} A^{\pi_{\vtheta_{k}}}(\vs,\va) \right]}^{\eqqcolon \textsf{C}}  \nonumber \\
        &\quad - \frac{2 \gamma C(\vtheta,\vtheta_{k})}{(1-\gamma)^{2} \omega_{k}} \sum_{i=0}^{\omega_{k}-1} \E_{\vs \sim d_{\vtheta_{k-i}}}\left[ D_{\mathrm{TV}}(\pi_{\vtheta}(\cdot \mid \vs) , \pi^{\mathrm{BH}}_{\omega_{k}}(\cdot \mid \vs)) \right] \label{eq:bh-pi-lb-stale-3} \\
        &\quad - \left(\frac{4 R_{\max}}{(1-\gamma)} + \frac{2 \gamma C(\vtheta,\vtheta_{k})}{(1-\gamma)^{2} \omega_{k}}\right)\sum_{i=0}^{\omega_{k}-1} D_{\mathrm{TV}}\left( d^{\mathrm{BH}}_{\omega_{k}} , d_{\vtheta_{k-i}} \right). \nonumber 
    \end{align}
    
    The last thing to do to conclude the proof is to substitute in \textsf{C} the fresh advantage $A^{\pi_{\vtheta_{k}}}(\vs,\va)$ with the stale one $A^{\pi_{\vtheta_{k-i}}}(\vs,\va)$, for every $i \in \dsb{0,\omega_{k}-1}$. To do so, we will proceed as done in Theorem~\ref{thr:uniform-bound}, adding and subtracting the expectation term with the stale advantage, then resorting to Lemma~\ref{lemma:stale-adv}. Specifically, the following holds
    \begin{align*}
        \textsf{C}  &= \E_{\substack{\vs \sim d_{\vtheta_{k-i}} \\ \va \sim \pi_{\vtheta_{k-i}}(\cdot \mid \vs)}} \left[ \frac{\pi_{\vtheta}(\va \mid \vs)}{\pi^{\mathrm{BH}}_{\omega_{k}}(\va \mid \vs)} A^{\pi_{\vtheta_{k}}}(\vs,\va) \right] \\
        &= \E_{\substack{\vs \sim d_{\vtheta_{k-i}} \\ \va \sim \pi_{\vtheta_{k-i}}(\cdot \mid \vs)}} \left[ \frac{\pi_{\vtheta}(\va \mid \vs)}{\pi^{\mathrm{BH}}_{\omega_{k}}(\va \mid \vs)} A^{\pi_{\vtheta_{k-i}}}(\vs,\va) \right] + \underbrace{\E_{\substack{\vs \sim d_{\vtheta_{k-i}} \\ \va \sim \pi_{\vtheta_{k-i}}(\cdot \mid \vs)}} \left[ \frac{\pi_{\vtheta}(\va \mid \vs)}{\pi^{\mathrm{BH}}_{\omega_{k}}(\va \mid \vs)} \left( A^{\pi_{\vtheta_{k}}}(\vs,\va) - A^{\pi_{\vtheta_{k-i}}}(\vs,\va) \right) \right]}_{\eqqcolon \textsf{D}}.
    \end{align*}
    
    So now, it suffices to apply Lemma~\ref{lemma:stale-adv} to \textsf{D}, thus having:
    \begin{align*}
         \textsf{D} &= \E_{\substack{\vs \sim d_{\vtheta_{k-i}} \\ \va \sim \pi_{\vtheta_{k-i}}(\cdot \mid \vs)}} \left[ \frac{\pi_{\vtheta}(\va \mid \vs)}{\pi^{\mathrm{BH}}_{\omega_{k}}(\va \mid \vs)} \left( A^{\pi_{\vtheta_{k}}}(\vs,\va) - A^{\pi_{\vtheta_{k-i}}}(\vs,\va) \right) \right] \\
         &= \E_{\substack{\vs \sim d_{\vtheta_{k-i}} \\ \va \sim \pi_{\vtheta}(\cdot \mid \vs)}} \left[ \frac{\pi_{\vtheta_{k-i}}(\va \mid \vs)}{\pi^{\mathrm{BH}}_{\omega_{k}}(\va \mid \vs)} \left( A^{\pi_{\vtheta_{k}}}(\vs,\va) - A^{\pi_{\vtheta_{k-i}}}(\vs,\va) \right) \right] \\
         &\le \omega_{k} \E_{\substack{\vs \sim d_{\vtheta_{k-i}} \\ \va \sim \pi_{\vtheta}(\cdot \mid \vs)}} \left[ \left| A^{\pi_{\vtheta_{k}}}(\vs,\va) - A^{\pi_{\vtheta_{k-i}}}(\vs,\va) \right| \right] \\
         &\le \omega_{k} \frac{4 R_{\max} (1+\gamma)}{1-\gamma} \sup_{\vs \in \cS} D_{\mathrm{TV}}\left( \pi_{\vtheta_{k}}(\cdot \mid \vs), \pi_{\vtheta_{k-i}}(\cdot \mid \vs) \right),
    \end{align*}
    where $(i)$ we changed the action sampling distribution from $\pi_{\vtheta_{k-i}}$ to $\pi_{\vtheta}$; $(ii)$ we exploited $\frac{\pi_{\vtheta_{k-i}}(\va \mid \vs)}{\pi^{\mathrm{BH}}_{\omega_{k}}(\va \mid \vs)} \le \omega_{k}$ for every $\vs \in \cS$, $\va \in \cA$, and $i \in \dsb{0,\omega_{k}-1}$; $(iii)$ we used the fact that $\textsf{D} \ge -|\textsf{D}|$, then applying the Jensen's inequality; and $(iv)$ we employed Lemma~\ref{lemma:stale-adv}.
    
    We can finally put all together, thus having the following final policy improvement bound:
    \begin{align*}
        J(\vtheta) - J(\vtheta_{k}) &\ge \frac{1}{\omega_{k}} \sum_{i=0}^{\omega_{k}-1} \E_{\substack{\vs \sim d_{\vtheta_{k-i}} \\ \va \sim \pi_{\vtheta_{k-i}}(\cdot \mid \vs)}} \left[ \frac{\pi_{\vtheta}(\va \mid \vs)}{\pi^{\mathrm{BH}}_{\omega_{k}}(\va \mid \vs)} A^{\pi_{\vtheta_{k-i}}}(\vs,\va) \right]  \\
        &\quad - \frac{2 \gamma C(\vtheta,\vtheta_{k})}{(1-\gamma)^{2} \omega_{k}} \sum_{i=0}^{\omega_{k}-1} \E_{\vs \sim d_{\vtheta_{k-i}}}\left[ D_{\mathrm{TV}}(\pi_{\vtheta}(\cdot \mid \vs) , \pi^{\mathrm{BH}}_{\omega_{k}}(\cdot \mid \vs)) \right]  \\
        &\quad - \frac{4 R_{\max} (1+\gamma)}{1-\gamma} \sum_{i=0}^{\omega_{k}-1} \sup_{\vs \in \cS} D_{\mathrm{TV}}\left( \pi_{\vtheta_{k}}(\cdot \mid \vs) , \pi_{\vtheta_{k-i}}(\cdot \mid \vs) \right) \\
        &\quad - \left(\frac{4 R_{\max}}{(1-\gamma)} + \frac{2 \gamma C(\vtheta,\vtheta_{k})}{(1-\gamma)^{2} \omega_{k}}\right) \sum_{i=0}^{\omega_{k}-1} D_{\mathrm{TV}}\left( d^{\mathrm{BH}}_{\omega_{k}} , d_{\vtheta_{k-i}} \right). 
    \end{align*}
\end{proof}

%% file: appendix/apx_exp_details.tex
\section{Experimental Details} \label{apx:exp-details}
To help with reproducibility, we provide all the necessary implementation and experimental details we used for conducting the experiments presented in Section~\ref{sec:experiments} and Appendix~\ref{apx:exp-additional}. 

\textbf{Environments.}~~We conducted our numerical evaluation on six continuous control environments leveraging the MuJoCo~\cite{todorov2012mujoco} control suite. Specifically, we considered: \emph{Reacher-v5}, \emph{Swimmer-v5}, \emph{Hopper-v5}, \emph{Walker2d-v5}, \emph{HalfCheetah-v5}, and \emph{Ant-v5}.
Observations are normalized online using a running mean and variance estimate.

\input{algo/uppo_bhppo}

\textbf{Implementation.}~~For implementing \uppo and \bhppo we extended the \ppo implementation of~\citep{stable-baselines3} with the following modifications. First, the rollout buffer is extended to hold the most recent $\omega_k$ rollouts (see Algorithm~\ref{alg:data-collection} describing the data collection phase): at each iteration $k$, the oldest rollout $\cD_{k-\omega_{k}}$ is discarded and the new one $\cD_{k}$ is appended. 
Second, mini-batches $\{\cB_{j}\}_{j=1}^{N_{\cB}}$ are sampled uniformly at random from the full buffer $\cD(\omega_{k})$, 
with the constraint that each $\cB_j$ contains the same number of transitions for every window index ($B/\omega_k$ in the fixed $B$ setting, $B$ in the fixed $N_{\cB}$ setting).
Third, \bhppo additionally stores the parameters of the $\omega_{k} - 1$ most recent past policies, which are used to evaluate the BH mixture denominator on every stored transition in $\cD(\omega_{k})$. Specifically, after every data collection phase, \bhppo requires to evaluate all the samples in $\cD(\omega_{k})$ coming from past policies $\{\pi_{\vtheta_{k-i}}\}_{i=1}^{\omega_{k} - 1}$ under $\pi_{\vtheta_{k}}$, as well as evaluating newer samples in $\cD_{k}$ under the past policies $\{\pi_{\vtheta_{k-i}}\}_{i=1}^{\omega_{k} - 1}$. This requires to perform additional $2D(\omega_{k}-1)$ evaluations, once per data collection.
Finally, the ratio in the actor loss is replaced by the per-policy ratio for \uppo or the BH-corrected ratio \bhppo, as defined in Section~\ref{sec:sample-reuse}. All remaining implementation details are shared with standard \ppo. The pseudo-code of the considered \ppo sample reuse variants is presented in Algorithm~\ref{alg:pseudo}.

\begin{table}[t]
    \centering

    \resizebox{\linewidth}{!}{\input{tables/hyperparameters}}

    \vspace{0.1cm}
    
    \caption{Employed hyperparameters values for conducting the experimental validation shown in Section~\ref{sec:experiments} and~\ref{apx:exp-additional}.}
    \label{tab:hyperparams}
\end{table}

\textbf{Hyperparameters.}~~For each environment, in order to conduct a fair comparison, we adopt the tuned \ppo hyperparameters from~\citep{rl-zoo3} as our baseline configuration, with the exception of \emph{Hopper-v5}, for which the provided configuration did not yield the expected performance and hyperparameters were selected manually.
The full set of hyperparameters for each environment is reported in
Table~\ref{tab:hyperparams}. For the networks' activation functions we used \texttt{tanh} ones.
\uppo and \bhppo share all common hyperparameters with \ppo, with only the parameters specific to each comparison setting varied as described in Section~\ref{sec:experiments} and~\ref{apx:exp-additional}.
This choice reflects our goal of providing a controlled comparison: by evaluating the reuse variants against a strongly tuned \ppo baseline, any observed improvement can be attributed to the data reuse mechanism itself rather than to a suboptimal baseline.

\textbf{Evaluation Protocol.}~~Each method is evaluated over 10 independent runs using the same set of fixed random seeds across all compared methods. Specifically, we employed seeds in $\dsb{0,9}$. 
Policy performance is assessed every $10000$ environment steps by running 50 complete episodes under the current deterministic policy, and the evaluation metric is the mean episodic return, computed as the undiscounted sum of rewards over a complete episode. For plotting purposes, the resulting mean return is smoothed via a moving average with a window of 10 evaluation points, and all plots report the mean and 95\% confidence interval across runs. 
Final performance values reported in tables are computed as the mean over the last 10 evaluation points of each run, prior to smoothing.

\textbf{Diagnostics.}~~To monitor the distributional shift induced by data reuse, we track two complementary diagnostics at every iteration. Each diagnostic is computed independently for every window index $i \in \dsb{0,\omega_{k}-1}$ over the corresponding rollout set $\mathcal{D}_{k-i}$, and is evaluated before the optimization phase to distinguish the shift present at the start of each iteration.
Specifically, the two diagnostics are:
\begin{itemize}[noitemsep,leftmargin=*,topsep=-3pt]
    \item \emph{Clip Fraction}:
    \begin{equation*}
        \mathrm{CF}^{\Box}_i = \frac{1}{D} \sum_{j=1}^{D}
        \mathds{1}\!\left[w^{\Box}_{i,j} \notin [1-\epsilon,\, 1+\epsilon]\right],
    \end{equation*}
    where for \uppo we have 
    $w^{\texttt{U}}_{i,j} = \frac{\pi_{\vtheta}(\va_{i,j} \mid \vs_{i,j})}{\pi_{\vtheta_{k-i}}(\va_{i,j} \mid \vs_{i,j})}$
    while for \bhppo we have $w^{\texttt{BH}}_{i,j} = \frac{\pi_{\vtheta}(\va_{i,j} \mid \vs_{i,j})}{\pi^{\mathrm{BH}}_{\omega_{k}}(\va_{i,j} \mid \vs_{i,j})}$.
    $\mathrm{CF}^{\Box}_i$ is the fraction of transitions in $\mathcal{D}_{k-i}$ whose weight $w^{\Box}_{i,j}$ falls outside $[1-\epsilon, 1+\epsilon]$. It measures how frequently the clipping mechanism is active for data from window index $i$, and therefore how often those samples are effectively excluded from the gradient signal. However, it is insensitive to the magnitude of deviations among clipped transitions.
    \item \emph{(Normalized) Effective Sample Size (ESS)}:
    \begin{equation*}
        \mathrm{ESS}^{\Box}_i = \frac{\left(\sum_j w^{\Box}_{i,j}\right)^2}
        {D \cdot \sum_j (w^{\Box}_{i,j})^2} \;\in [0, 1].
    \end{equation*}
    In a standard importance weighting setting, the the \emph{effective sample size}~\citep[ESS,][]{kong1992note} is the number of samples drawn directly from $P$ for which the Monte Carlo estimator of the mean has approximately the same variance as the self-normalized importance weighting estimator employing $D$ samples from $Q$~\citep{metelli2020importance}. So, the normalized ESS is $1$ whenever $P=Q$, while it decreases as the dissimilarity of the two distributions enlarges.
\end{itemize}

Each metric is averaged across all iterations of a single training run, 
reported for the pre-optimization phase. Plots show 
the mean and 95\% confidence interval over 10 independent runs.

\textbf{Computational Resources}~~All experiments were conducted on a machine equipped with an AMD Ryzen Threadripper PRO 9965WX CPU (24 cores, 48 threads) and an NVIDIA RTX A1000 GPU (8\,GB VRAM), running Ubuntu 24.04 LTS. A single run of \ppo in \emph{HalfCheetah-v5} considering the hyperparameters of Table~\ref{tab:hyperparams} takes $\approx 25$ minutes. Additionally, in the same setting, a run of \bhppo and \uppo with $\omega=4$ in the \emph{fixed $B$} scenario takes $\approx 50$ minutes, while for \emph{fixed $N_{\cB}$} scenario it takes $\approx 35$ minutes. 

\textbf{LLMs Usage.}~~We employed LLMs mainly as aid for editing purposes, for checking proof correctness, and for assistance with code. 

%% file: algo/uppo_bhppo.tex
\RestyleAlgo{ruled}
\LinesNumbered
\begin{algorithm}[h!]
    \caption{\uppo and \bhppo.}\label{alg:pseudo}
    \SetKwInOut{Input}{Input}
    \SetKwInOut{Output}{Output}
    \small
    \Input{Initial parameters $\vtheta_{0}$ and $\vw_{0}$, iterations $K$, epochs per iteration $\widetilde{K}$, batch size $B$, environments $\{E_e\}_{e=1}^{N_E}$, GAE parameter $\lambda$, interaction length $H$, discount factor $\gamma$, loss coefficients $c_{1}$ and $c_{2}$, learning rate $\alpha$, clipping parameter $\epsilon$, window size $\omega$.}
    Initialize optimization variables: $\vtheta \leftarrow \vtheta_{0}$, $\vw \leftarrow \vw_{0}$; initialize buffer $\cD(\omega_{0}) \leftarrow \emptyset$.
    
    \For{$k \in \dsb{1,K}$}{
        \tcp{Data Collection}
        $\cD_{k} \leftarrow$ Run Algorithm~\ref{alg:data-collection}\;
        \tcp{Buffer Update}
        \eIf{$\omega_{k} < \omega$}{
            $\cD(\omega_{k}) \leftarrow \cD(\omega_{k-1}) \cup \cD_{k}$\;
        }{
            $\cD(\omega_{k}) \leftarrow \cD(\omega_{k-1}) \setminus \cD_{k-\omega} \cup \cD_{k}$\;
        }
        \tcp{[\bhppo only] Precompute mixture denominator}
        $\pi^{\mathrm{BH}}_{\omega_{k}}(\va \mid \vs) \leftarrow \frac{1}{\omega_{k}} \sum_{i=0}^{\omega_{k}-1} \pi_{\vtheta_{k-i}}(\va \mid \vs)$ \quad for all $(\vs,\va) \in \cD(\omega_{k})$\;
        \tcp{Multi-Epoch Optimization}
        \For{$l \in \dsb{\widetilde{K}}$}{
            Shuffle $\cD(\omega_{k})$\;
            Create $\widetilde{\cD}_{k} = \{\cB_{z}\}_{z=1}^{N_{\cB}}$, where $N_{\cB} = \left\lceil \frac{|\cD(\omega_{k})|}{B} \right\rceil$, such that $\forall z \in \dsb{N_{\cB}}: |\cB_{z}| = B$ equally represents all $\omega_{k}$ iterations\;
            \For{$\cB \in \widetilde{\cD}_{k}$}{
                \tcp{Actor Loss}
                \uppo: $\quad\rho_{i,j} \leftarrow \dfrac{\pi_{\vtheta}(\va_{i,j} \mid \vs_{i,j})}{\pi_{\vtheta_{k-i}}(\va_{i,j} \mid \vs_{i,j})}$\;
                \bhppo: $\quad\rho_{i,j} \leftarrow \dfrac{\pi_{\vtheta}(\va_{i,j} \mid \vs_{i,j})}{\pi^{\mathrm{BH}}_{\omega_{k}}(\va_{i,j} \mid \vs_{i,j})}$\;
                $$\widehat{L}^{\mathrm{actor}}(\vtheta, \cB) = \frac{1}{|\cB|} \sum_{(i,j) \in \cB} \min\left\{ \rho_{i,j}\, A_{i,j},\ \mathrm{clip}\left( \rho_{i,j}, 1-\epsilon, 1+\epsilon \right) A_{i,j} \right\}$$
                
                \tcp{Critic Loss (on $\cD_{k}$ samples only)}
                $$\widehat{L}^{\mathrm{vf}}(\vw, \cB) = \frac{1}{|\cB \cap \cD_{k}|} \sum_{(0,j) \in \cB} \frac{1}{2}(V_{\vw}(\vs_{0,j}) - R_{0,j})^{2}$$
                
                \tcp{Update}
                $$\vtheta \leftarrow \vtheta + \alpha \nabla_{\vtheta}\left[\widehat{L}^{\mathrm{actor}}(\vtheta,\cB) + c_{2}\, \frac{1}{|\cB|} \sum_{(i,j) \in \cB} H(\pi_{\vtheta}(\cdot \mid \vs_{i,j}))\right]$$
                $$\vw \leftarrow \vw - \alpha\, c_{1} \nabla_{\vw}\widehat{L}^{\mathrm{vf}}(\vw, \cB)$$
            }
        }
    }
    \Output{$\vtheta_{K}$.}
\end{algorithm}

%% file: tables/hyperparameters.tex
\begin{tabular}{lcccccc}
    \toprule
    \textbf{Hyperparameter} & \emph{Ant-v5} & \emph{HalfCheetah-v5} & \emph{Hopper-v5} & \emph{Reacher-v5} & \emph{Swimmer-v5} & \emph{Walker2d-v5} \\
    \midrule
    \multicolumn{7}{l}{\textit{Environment}} \\
    Total timesteps         & $3 \times 10^6$ & $1 \times 10^6$ & $1 \times 10^6$ & $1 \times 10^6$ & $1 \times 10^6$ & $1 \times 10^6$ \\
    Number of environments ($N_E$)  & 1 & 1 & 1 & 1 & 4 & 1 \\
    \midrule
    \multicolumn{7}{l}{\textit{PPO}} \\
    Learning rate           & $1.91 \times 10^{-5}$ & $2.06 \times 10^{-5}$ & $3 \times 10^{-4}$ & $1.04 \times 10^{-4}$ & $6 \times 10^{-4}$ & $5.05 \times 10^{-5}$ \\
    Steps per rollout ($H$) & 512  & 512  & 2048 & 512  & 1024 & 512  \\
    Batch size              & 32   & 64   & 64   & 32   & 256  & 32   \\
    Epochs per update       & 10   & 20   & 10   & 5    & 10   & 20   \\
    Discount factor $\gamma$        & 0.98   & 0.98   & 0.999  & 0.9    & 0.9999 & 0.99   \\
    GAE $\lambda$           & 0.8  & 0.92 & 0.95 & 1.0  & 0.98 & 0.95 \\
    Clip range $\epsilon$   & 0.1  & 0.1  & 0.2  & 0.3  & 0.2  & 0.1  \\
    Value function coefficient      & 0.677 & 0.581 & 0.5   & 0.950 & 0.5   & 0.872 \\
    Entropy coefficient             & $4.96 \times 10^{-7}$ & $4.02 \times 10^{-4}$ & $0.0$ & $7.53 \times 10^{-8}$ & $0.0$ & $5.85 \times 10^{-4}$ \\
    Max gradient norm               & 0.6  & 0.8  & 0.5  & 0.9  & 0.5  & 1.0  \\
    Normalize advantage             & \checkmark & \checkmark & \checkmark & \checkmark & \checkmark & \checkmark \\
    \midrule
    \multicolumn{7}{l}{\textit{Policy Network}} \\
    Architecture            & $64 \times 64$  & $256 \times 256$ & $64 \times 64$  & $64 \times 64$  & $64 \times 64$  & $64 \times 64$  \\
    Orthogonal init         & \checkmark & $\times$   & \checkmark & \checkmark & \checkmark & \checkmark \\
    Log std init            & $0.0$ & $-2.0$ & $0.0$ & $0.0$ & $0.0$ & $0.0$ \\
    \bottomrule
\end{tabular}%

%% file: appendix/apx_exp_additional.tex
\section{Additional Experiments}
\label{apx:exp-additional}

In this section, we provide the complete experimental campaign conducted for this work, of which Section~\ref{sec:experiments} presents a subset. Our experiments aim to empirically understand when and how \ppo can benefit from sample reuse. All experiments are designed to maintain fairness with respect to \ppo, as detailed in Appendix~\ref{apx:exp-details}.

\textbf{\ppo Mechanics and Update Accounting.}~~Before starting, we recall that at each iteration $k \in \dsb{K}$, \ppo collects a dataset $\cD_{k}$ of $|\cD_{k}| = N_{E} H \eqqcolon D$\footnote{Where $N_E$ represents the number of parallel environments and $H$ the number of steps performed in each environment.} samples under $\pi_{\vtheta_{k}}$. For $\widetilde{K}$ epochs, $\cD_{k}$ is shuffled and divided into $N_{\cB}$ mini-batches of size $B$, each used to perform one parameter update, so \ppo performs $\widetilde{K} N_{\cB}$ actor and critic updates per iteration. Each datum therefore has an update life cycle of $\widetilde{K}$ updates before being discarded.

\textbf{\uppo and \bhppo Mechanics and Update Accounting.}~~As detailed in Section~\ref{sec:sample-reuse}, rather than discarding $\cD_k$ after each iteration, \uppo and \bhppo maintain a sliding window of the $\omega_k \coloneqq \min\{k,\omega\}$ most recent datasets, forming the composite buffer $\cD(\omega_k) = \bigcup_{i=0}^{\omega_k-1} \cD_{k-i}$ of total size $\omega_k D$,\footnote{The buffers are considered to have the same size since at each iteration $k$ policy $\pi_{k}$ is used to collect $N_E H$ samples, as previously detailed.} from which balanced mini-batches are drawn. 
As in \ppo, $\widetilde{K}$ epochs are performed per iteration, while the number and size of mini-batches depend on the comparison setting (see below and Table~\ref{tab:settings}).
Because each datum remains in the buffer for $\omega_k$ iterations and is replayed for $\widetilde{K}$ epochs per iteration, its update life cycle extends to $\omega_k \widetilde{K}$ epochs. 

\textbf{Comparison Settings.}~~The mini-batch sampling strategy from $\cD(\omega_k)$ governs how the reuse variants differ from \ppo in terms of update count, mini-batch composition, and data collection cost. We consider three settings, each designed to isolate a distinct aspect of the reuse mechanism.
\begin{itemize}[leftmargin=*, topsep=-3pt, noitemsep]
    \item \emph{Fixed mini-batch size (Fixed $B$).}~~The reuse variants share the same $B$ and $\widetilde{K}$ as \ppo. Since $|\cD(\omega_k)| = \omega_k D$, this yields $\omega_k N_{\cB}$ mini-batches per epoch, so the reuse variants perform $\omega_k \widetilde{K} N_{\cB}$ actor updates per iteration, exactly $\omega_k$ times as many as \ppo, with each mini-batch containing $B/\omega_k$ samples per window index. This setting investigates whether a larger buffer, made of fresh and past data, can support more gradient steps without additional environment interaction, and whether keeping $B$ fixed is sufficient to avoid update instability. To isolate the effect of reuse from the increased update count, we also run a \ppo baseline with $\omega_k \widetilde{K}$ epochs per iteration, matching the total actor updates of the reuse variants.
    \item \emph{Fixed mini-batch count (Fixed $N_{\cB}$).}~~The reuse variants keep the same $N_{\cB}$ as \ppo, yielding larger mini-batches of size $\omega_k B$. All methods therefore perform $\widetilde{K} N_{\cB}$ actor updates per iteration, with each update seeing the same amount of fresh data as in \ppo while additionally drawing on past samples. This setting investigates whether keeping the update count fixed while enriching each mini-batch with historical data is sufficient to improve learning, and whether mixing current and past samples within a larger mini-batch provides a benefit over using fresh data alone.
    \item \emph{Reduced data collection ($H/\omega$ Steps).}~~The reuse variants collect only $H/\omega_k$ new environment steps per iteration, keeping $B$, $N_{\cB}$, and $\widetilde{K}$ fixed at their \ppo values. The composite buffer therefore has the same total size $D$ as \ppo's dataset, but only a fraction $1/\omega_k$ of it consists of fresh transitions, and all methods perform $\widetilde{K} N_{\cB}$ actor updates per iteration. This setting investigates whether stored historical data can compensate for a reduced fresh signal, allowing the policy to advance more rapidly relative to environment interaction. Among the three settings, it poses the most stringent test of sample efficiency.
\end{itemize}

Table~\ref{tab:settings} summarizes the main characteristics of the discussed comparison settings.
In the following, we analyze all three settings and use diagnostic tools to investigate when and how data reuse helps in \ppo.

\begin{table}[t]
\centering
\resizebox{\linewidth}{!}{\renewcommand{\arraystretch}{1.2}
\setlength{\tabcolsep}{10pt}
\begin{tabular}{lccccc}
\toprule
\textbf{Setting}
    & \makecell{\textbf{Environment}\\\textbf{Steps}}
    & \makecell{\textbf{Update}\\\textbf{Epochs}}
    & \makecell{\textbf{Mini-Batch}\\\textbf{Size}}
    & \makecell{\textbf{Mini-Batch}\\\textbf{Count}}
    & \makecell{\textbf{Updates}\\\textbf{per Iteration}} \\
\midrule
\ppo
    & $H$
    & $\widetilde{K}$
    & $B$
    & $N_{\cB}$
    & $\widetilde{K} N_{\cB}$ \\
\ppo~($\omega\widetilde{K}$ epochs)
    & $H$
    & \cellcolor{ibmUltramarine!20}$\omega\widetilde{K}$
    & $B$
    & $N_{\cB}$
    & \cellcolor{ibmUltramarine!20}$\omega\widetilde{K} N_{\cB}$ \\
\midrule
\wppo: Fixed $B$
    & $H$
    & $\widetilde{K}$
    & $B$
    & \cellcolor{ibmUltramarine!20}$\omega N_{\cB}$
    & \cellcolor{ibmUltramarine!20}$\omega\widetilde{K} N_{\cB}$ \\
\wppo: Fixed $N_{\cB}$
    & $H$
    & $\widetilde{K}$
    & \cellcolor{ibmUltramarine!20}$\omega B$
    & $N_{\cB}$
    & $\widetilde{K} N_{\cB}$ \\
\wppo: $H/\omega$ Steps
    & \cellcolor{ibmUltramarine!20}$H/\omega$
    & $\widetilde{K}$
    & $B$
    & $N_{\cB}$
    & $\widetilde{K} N_{\cB}$ \\
\bottomrule
\end{tabular}}

\vspace{0.2cm}

\caption{Summary of the main features of the considered comparison setting.
Highlighted cells mark values that differ from the \ppo baseline.}
\label{tab:settings}
\end{table}

\subsection{Fixed Mini-Batch Size and Increased \ppo Epochs Settings} \label{apx:exp-set2}
Here, we analyze the behavior of \uppo and \bhppo when sharing all hyperparameters with the tuned \ppo baseline (Appendix~\ref{apx:exp-details}). Specifically, the reuse variants retain the same mini-batch size $B$, so each mini-batch contains $B/\omega_{k}$ samples per reuse window and the reuse variants perform $\omega_k$ times as many actor updates per iteration as \ppo. To isolate the effect of sample reuse from the increased number of updates alone, we also include a \ppo baseline running $\omega \widetilde{K}$ epochs per iteration, matching the total update count of the reuse variants.

Figures~\ref{fig:set-2-train} and~\ref{fig:set-2-train-2} show training evaluation curves across the considered MuJoCo~\citep{todorov2012mujoco} environments and Table~\ref{tab:set2-perf} reports final performance achieved by all methods, clustered by environment. All experiments consider three window sizes, $\omega \in \{2, 4, 8\}$.
Overall, this set of experiments supports five conclusions: $(i)$ \ppo may truly benefit from sample reuse; $(ii)$ this benefit is not merely attributable to performing more gradient updates; $(iii)$ in this setting, data reuse does not always accelerate convergence but leads to higher final performance in most cases; $(iv)$ neither \uppo nor \bhppo consistently dominates the other, with both methods behaving similarly across environments; and $(v)$ it appears that in this setting there exists a window value for which the reuse variants exhibit most of the benefit, although this window value is environment-dependent.
We now detail the results for each environment in turn.

Figure~\ref{fig:set-2-reacher} reports the evaluation curves on the \emph{Reacher-v5} environment. Even a reuse window of $\omega=2$ yields a clear improvement in sample efficiency, with most of the gain exhibited at $\omega=4$. Further increasing to $\omega=8$ continues to improve sample efficiency, though with diminishing returns relative to the $\omega=2$ to $\omega=4$ step. Notably, data reuse not only accelerates learning but also leads to higher final performance (Table~\ref{tab:set2-perf}), with \bhppo consistently achieving the best results across window sizes. As expected, more updates also increase wall-clock training time.
The comparison against the $\omega \widetilde{K}$-epoch \ppo baseline is particularly informative. While this baseline initially matches the sample efficiency of the reuse variants, it becomes unstable at larger values of $\omega$, a phenomenon that does not affect \uppo or \bhppo. 
This suggests that the benefits of data reuse do not stem solely from performing more updates, but from the additional data itself.

Similar observations hold for the \emph{Swimmer-v5} environment, shown in Figure~\ref{fig:set-2-swimmer}. Data reuse consistently benefits \ppo, though the improvement at $\omega=2$ is mild and the advantage over the $\omega\widetilde{K}$-epoch \ppo baseline is not clearly separable at this window size. The benefit becomes more pronounced at $\omega=4$ and $\omega=8$, where the reuse variants reach higher final performances while the $\omega\widetilde{K}$-epoch \ppo baseline saturates early. Notably, unlike in \emph{Reacher-v5}, increasing the window size does not accelerate convergence. Instead, data reuse allows the agent to extract more information from the available samples, translating into higher final performance, as also confirmed by Table~\ref{tab:set2-perf}, rather than faster initial learning.

One of the most instructive scenarios is the \emph{Hopper-v5} environment, shown in Figure~\ref{fig:set-2-hopper}. Data reuse is beneficial for $\omega \in \{2, 4\}$: in both cases, \uppo and \bhppo behave similarly and outperform both \ppo and its $\omega\widetilde{K}$-epoch variant, with the latter improving over \ppo at $\omega=2$ but collapsing at $\omega=4$ due to the increased number of updates. For $\omega=8$, however, both \uppo and \bhppo degrade significantly, with only \bhppo eventually recovering to match \ppo performance given its wide confidence interval. The $\omega\widetilde{K}$-epoch \ppo baseline collapses early in this regime. This setting clearly demonstrates that an optimal window size exists and that excessive reuse can actively harm performance. 
We conjecture this results from several compounding factors, amplified by this environment's configuration ($H = 2048$, $B = 64$, $\widetilde{K} = 10$). First, at $\omega = 8$ the reuse variants perform $2560$ updates per iteration against the $320$ of \ppo, amplifying the bias introduced by stale advantages across many gradient steps. Second, each window contributes only $B/\omega = 8$ samples per mini-batch, meaning that just $12.5\%$ of each actor update is driven by fresh data and that each critic update employs only $8$ samples. Third, since all hyperparameters are inherited from \ppo, the larger number of updates may increase the drift between consecutive policies and hence the diversity of the policies in the window. Past samples falling outside the clipping range contribute no gradient.
In Appendix~\ref{apx:exp-set3}, we investigate what happens when \ppo and its reuse variants are instead constrained to perform the same number of updates per iteration while matching the number of mini-batches $N_{\cB}$.

Figure~\ref{fig:set-2-walker} shows the same evaluation conducted in the \emph{Walker2d-v5} environment. The results are broadly consistent with the other environments: data reuse is beneficial, especially for $\omega \in \{2, 4\}$. The $\omega\widetilde{K}$-epoch \ppo variant slightly improves over \ppo at $\omega=2$ but fails to match the reuse variants, while for $\omega \in \{4, 8\}$ it collapses entirely. Notably, the BH correction does not appear to offer a clear advantage over uniform weighting in this environment, particularly at $\omega=8$. Finally, while larger windows do not consistently accelerate convergence, they allow the reuse variants to reach higher final performances, as confirmed by Table~\ref{tab:set2-perf}.

Figure~\ref{fig:set-2-half} reports the evaluation curves on the \emph{HalfCheetah-v5} environment.
Data reuse consistently leads to higher final performance across all window sizes $\omega \in \{2, 4, 8\}$, as confirmed also by Table~\ref{tab:set2-perf}.
There is a visible improvement from $\omega=2$ to $\omega=4$, while the gain from $\omega=4$ to $\omega=8$ is negligible, suggesting that most of the benefit is already captured at moderate window sizes.
\uppo and \bhppo perform comparably throughout, with neither variant offering a clear advantage over the other.
The $\omega\widetilde{K}$-epoch \ppo baseline exhibits faster initial learning but fails to reach the final performance of the reuse variants, 
and shows increased instability, possibly due to overexploiting a fixed data batch.
In contrast, \uppo and \bhppo perform the same number of updates while leveraging a larger and more diverse composite buffer, which appears to stabilize training and support more effective policy improvement.

Figure~\ref{fig:set-2-ant} reports the evaluation curves on the \emph{Ant-v5} environment.
Unlike the other environments, data reuse does not yield a clear improvement in final performance, as confirmed by Table~\ref{tab:set2-perf}; instead, the primary benefit lies in the speed at which the agent reaches that performance.
At $\omega=2$, the $\omega\widetilde{K}$-epoch \ppo baseline achieves a convergence speed comparable to the reuse variants without exhibiting instability.
However, at $\omega \in \{4, 8\}$, this baseline learns faster initially but then diverges, whereas \uppo and \bhppo remain stable and they reach the final performance level sooner than the baseline \ppo.

\begin{figure}[t]
    \centering
    \begin{subfigure}[t]{0.32\linewidth}
        \centering
        \includegraphics[width=\linewidth]{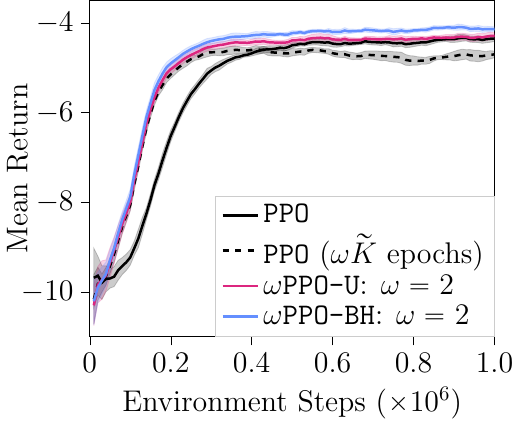}
    \end{subfigure}
    \hfill
    \begin{subfigure}[t]{0.32\linewidth}
        \centering
        \includegraphics[width=\linewidth]{figs/reacher/p-rea-set2-w4.pdf}
        \caption{\emph{Reacher-v5}.}
        \label{fig:set-2-reacher}
    \end{subfigure}
    \hfill
    \begin{subfigure}[t]{0.32\linewidth}
        \centering
        \includegraphics[width=\linewidth]{figs/reacher/p-rea-set2-w8.pdf}
    \end{subfigure}

    \vspace{0.3cm}

    \begin{subfigure}[t]{0.32\linewidth}
        \centering
        \includegraphics[width=\linewidth]{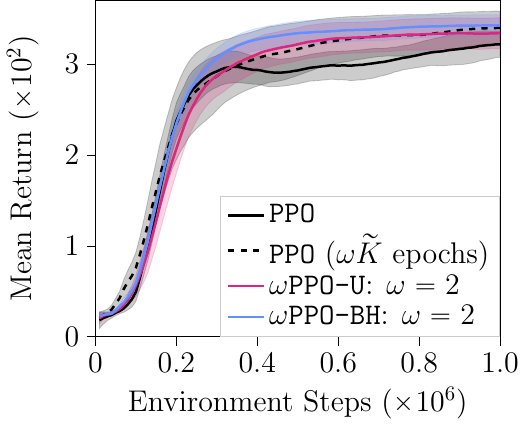}
    \end{subfigure}
    \hfill
    \begin{subfigure}[t]{0.32\linewidth}
        \centering
        \includegraphics[width=\linewidth]{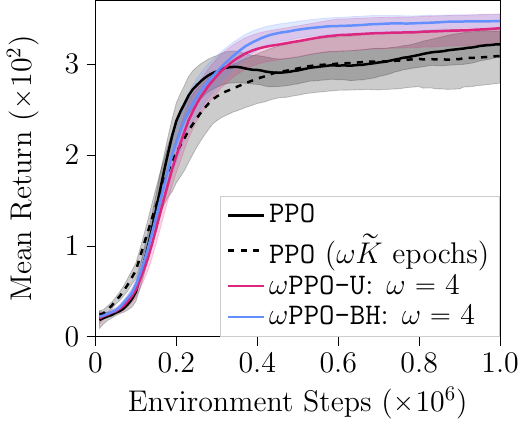}
        \caption{\emph{Swimmer-v5}.}
        \label{fig:set-2-swimmer}
    \end{subfigure}
    \hfill
    \begin{subfigure}[t]{0.32\linewidth}
        \centering
        \includegraphics[width=\linewidth]{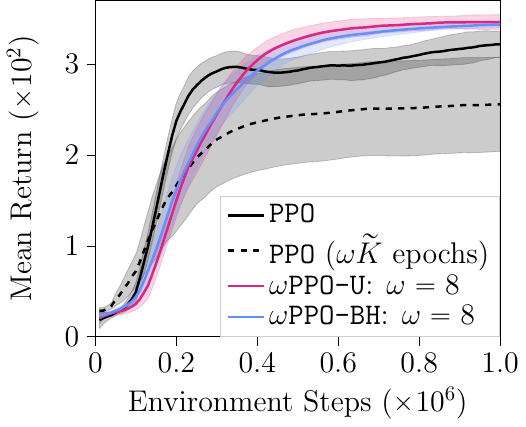}
    \end{subfigure}

    \vspace{0.3cm}

    \begin{subfigure}[t]{0.32\linewidth}
        \centering
        \includegraphics[width=\linewidth]{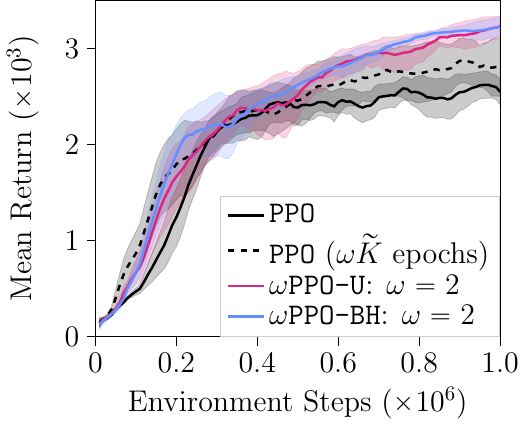}
    \end{subfigure}
    \hfill
    \begin{subfigure}[t]{0.32\linewidth}
        \centering
        \includegraphics[width=\linewidth]{figs/hopper/p-hop-set2-w4.pdf}
        \caption{\emph{Hopper-v5}.}
        \label{fig:set-2-hopper}
    \end{subfigure}
    \hfill
    \begin{subfigure}[t]{0.32\linewidth}
        \centering
        \includegraphics[width=\linewidth]{figs/hopper/p-hop-set2-w8.pdf}
    \end{subfigure}

    \vspace{0.3cm}

    \begin{subfigure}[t]{0.32\linewidth}
        \centering
        \includegraphics[width=\linewidth]{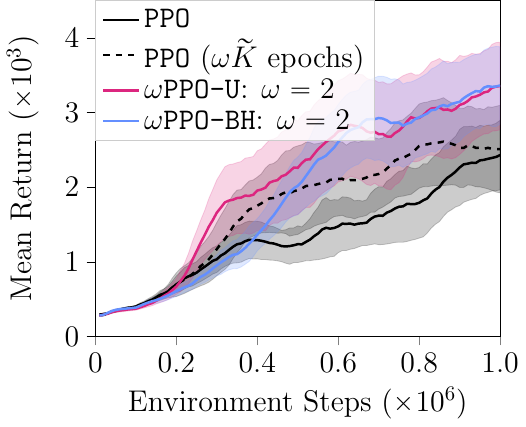}
    \end{subfigure}
    \hfill
    \begin{subfigure}[t]{0.32\linewidth}
        \centering
        \includegraphics[width=\linewidth]{figs/walker/p-wal-set2-w4.pdf}
        \caption{\emph{Walker2d-v5}.}
        \label{fig:set-2-walker}
    \end{subfigure}
    \hfill
    \begin{subfigure}[t]{0.32\linewidth}
        \centering
        \includegraphics[width=\linewidth]{figs/walker/p-wal-set2-w8.pdf}
    \end{subfigure}
    
    \caption{Comparison among \ppo, \bhppo, and \uppo in various MuJoCo environments under different window values in the \emph{fixed $B$} setting, showing the \ppo version with increased epochs to match the number of updates performed by the reuse variants (10 runs, mean $\pm$ 95\% C.I.).}
    \label{fig:set-2-train}
\end{figure}

\begin{figure}[t]
    \begin{subfigure}[t]{0.32\linewidth}
        \centering
        \includegraphics[width=\linewidth]{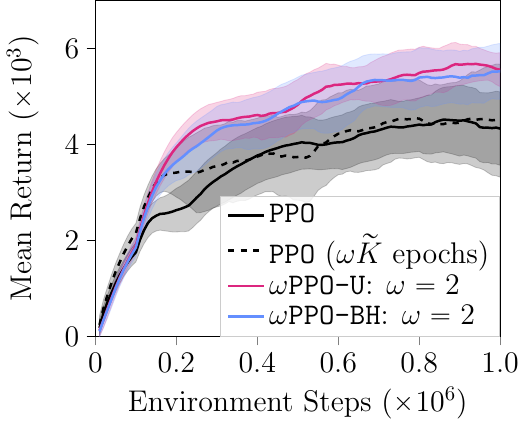}
    \end{subfigure}
    \hfill
    \begin{subfigure}[t]{0.32\linewidth}
        \centering
        \includegraphics[width=\linewidth]{figs/half/p-half-set2-w4.pdf}
        \caption{\emph{HalfCheetah-v5}.}
        \label{fig:set-2-half}
    \end{subfigure}
    \hfill
    \begin{subfigure}[t]{0.32\linewidth}
        \centering
        \includegraphics[width=\linewidth]{figs/half/p-half-set2-w8.pdf}
    \end{subfigure}

    \vspace{0.3cm}

    \begin{subfigure}[t]{0.32\linewidth}
        \centering
        \includegraphics[width=\linewidth]{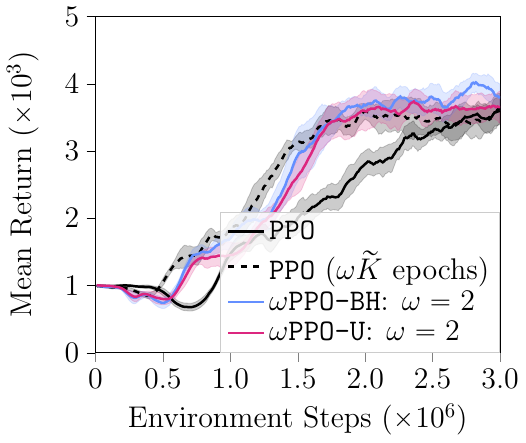}
    \end{subfigure}
    \hfill
    \begin{subfigure}[t]{0.32\linewidth}
        \centering
        \includegraphics[width=\linewidth]{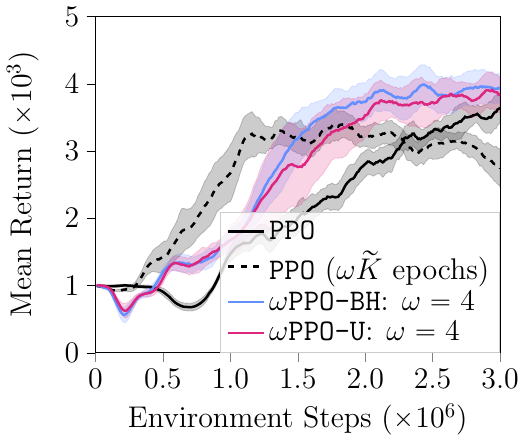}
        \caption{\emph{Ant-v5}.}
        \label{fig:set-2-ant}
    \end{subfigure}
    \hfill
    \begin{subfigure}[t]{0.32\linewidth}
        \centering
        \includegraphics[width=\linewidth]{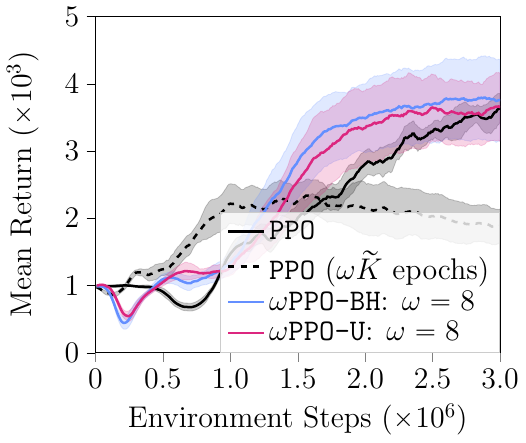}
    \end{subfigure}
    
    \caption{Comparison among \ppo, \bhppo, and \uppo in various MuJoCo environments under different window values in the \emph{fixed $B$} setting, showing the \ppo version with increased epochs to match the number of updates performed by the reuse variants (10 runs, mean $\pm$ 95\% C.I.).}
    \label{fig:set-2-train-2}
\end{figure}

\begin{table}[t]
    \centering
    \resizebox{\linewidth}{!}{\input{tables/set1-perf}}

    \vspace{0.1cm}
    
    \caption{Average of the last 10 final performance values of the curves shown in the \emph{fixed $B$} setting (Figures~\ref{fig:set-2-train} and~\ref{fig:set-2-train-2}). Each cell shows the mean $\pm 95\%$ C.I. edges (10 runs).}
    \label{tab:set2-perf}
\end{table}

\subsection{Fixed Mini-Batch Count Setting} \label{apx:exp-set3}

\begin{table}[t]
    \centering
    
    \resizebox{\linewidth}{!}{\input{tables/set3-perf}}

    \vspace{0.1cm}
    
    \caption{Average of the last 10 final performance values of the curves shown in the \emph{fixed $N_{\cB}$} setting (Figures~\ref{fig:set-3-train} and~\ref{fig:set-3-train-2}). Each cell shows the mean $\pm 95\%$ C.I. edges (10 runs).}
    \label{tab:set3-perf}
\end{table}

Here, we analyze the behavior of \uppo and \bhppo when sharing all hyperparameters with the tuned \ppo baseline (Appendix~\ref{apx:exp-details}), but, in contrast to Appendix~\ref{apx:exp-set2}, we keep fixed the number of mini-batches per epoch $N_{\cB}$ rather than the mini-batch size $B$. By design, the reuse variants perform the \emph{same} number of updates per epoch as \ppo, but leverage mini-batches of size $\omega B$, thereby preserving the same amount of fresh information per update, i.e., the same number of samples from the most recent policy $\pi_{\vtheta_{k}}$.

Figures~\ref{fig:set-3-train} and~\ref{fig:set-3-train-2} show training evaluation curves across the considered MuJoCo~\citep{todorov2012mujoco} environments and Table~\ref{tab:set3-perf} reports the final performance reached by the considered methods. All experiments consider three window sizes, $\omega \in \{2, 4, 8\}$.
The results of this setting can be summarized as follows: $(i)$ \ppo consistently benefits from sample reuse; $(ii)$ in this setting the benefit manifests in both faster convergence and higher final performance, as further confirmed by Table~\ref{tab:set3-perf}; $(iii)$ \bhppo tends to reach higher final performances with tighter confidence intervals, particularly for $\omega \in \{4, 8\}$; and $(iv)$ across environments, $\omega = 4$ appears to offer the best trade-off, with larger windows providing diminishing returns. 

In Figure~\ref{fig:set-3-reacher}, we show results for the \emph{Reacher-v5} environment, the simplest in our experimental campaign. The results are consistent with those of Appendix~\ref{apx:exp-set2}, where the \emph{baseline} setting is considered. The benefit of data reuse increases with $\omega$: larger windows yield both faster convergence and higher final performance, with \bhppo performing slightly better throughout. Importantly, all methods here perform the same number of updates, confirming that the gains are not an artifact of additional gradient steps.

For the other environments, similar conclusions can be drawn. At $\omega=2$, data reuse provides a marginal improvement, while most of the benefit is observed at $\omega=4$, particularly for \bhppo. At $\omega=8$, the pathological collapse seen in the baseline setting for \emph{Hopper-v5} (Figure~\ref{fig:set-2-hopper}) disappears, suggesting that larger mini-batches and reduced update count stabilize training. However, a hindered early convergence speed remains visible. 
We conjecture this is due to the larger policy changes occurring during early training: consecutive policies differ the most in this phase, so that, with a large reuse window, older samples carry the most stale advantages and come from the most distant behavioral policies. 

The exception in this setting is the \emph{Ant-v5} environment, the most complex among the ones considered. Unlike what is observed with fixed $B$ (Appendix~\ref{apx:exp-set2}), data reuse does not appear to help with convergence speed here and provides only a marginal improvement in final performance, with \bhppo at $\omega=4$ being the best-performing variant (Table~\ref{tab:set3-perf}), though the \ppo baseline's confidence interval slightly overlaps with it.

\begin{figure}[t]
    \centering
    \begin{subfigure}[t]{0.32\linewidth}
        \centering
        \includegraphics[width=\linewidth]{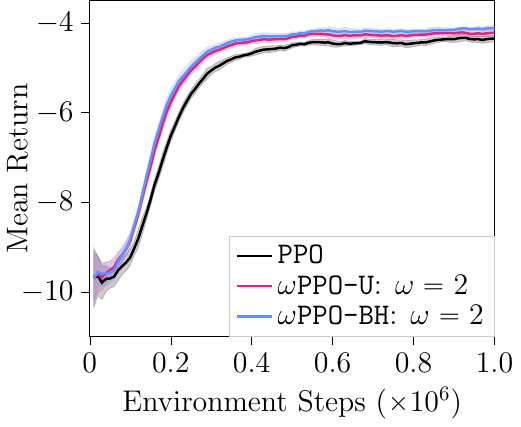}
    \end{subfigure}
    \hfill
    \begin{subfigure}[t]{0.32\linewidth}
        \centering
        \includegraphics[width=\linewidth]{figs/reacher/p-rea-set3-w4.pdf}
        \caption{\emph{Reacher-v5}.}
        \label{fig:set-3-reacher}
    \end{subfigure}
    \hfill
    \begin{subfigure}[t]{0.32\linewidth}
        \centering
        \includegraphics[width=\linewidth]{figs/reacher/p-rea-set3-w8.pdf}
    \end{subfigure}

    \vspace{0.3cm}

    \begin{subfigure}[t]{0.32\linewidth}
        \centering
        \includegraphics[width=\linewidth]{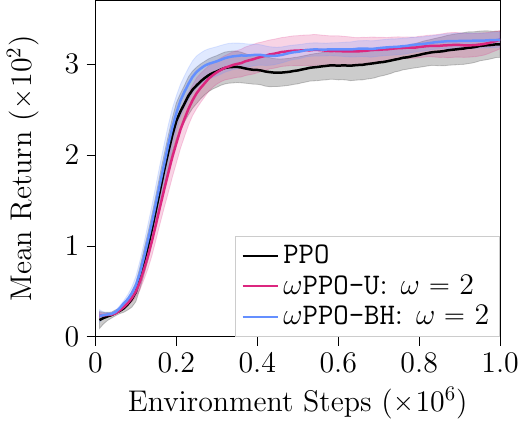}
    \end{subfigure}
    \hfill
    \begin{subfigure}[t]{0.32\linewidth}
        \centering
        \includegraphics[width=\linewidth]{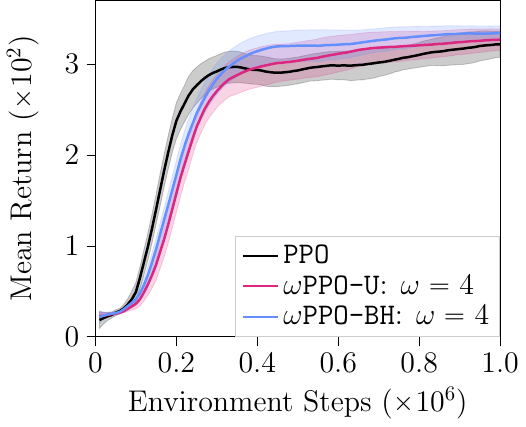}
        \caption{\emph{Swimmer-v5}.}
        \label{fig:set-3-swimmer}
    \end{subfigure}
    \hfill
    \begin{subfigure}[t]{0.32\linewidth}
        \centering
        \includegraphics[width=\linewidth]{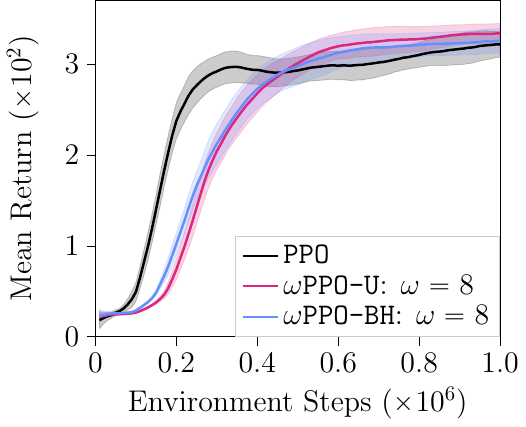}
    \end{subfigure}

    \vspace{0.3cm}

    \begin{subfigure}[t]{0.32\linewidth}
        \centering
        \includegraphics[width=\linewidth]{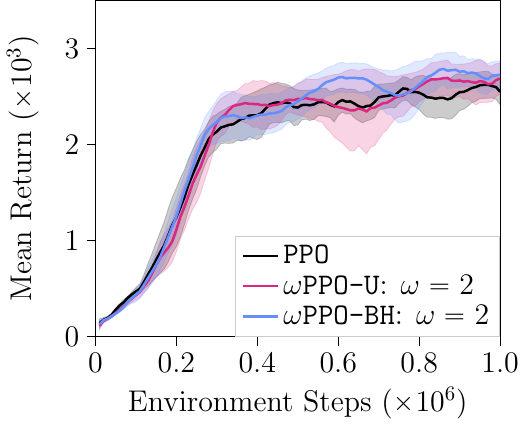}
    \end{subfigure}
    \hfill
    \begin{subfigure}[t]{0.32\linewidth}
        \centering
        \includegraphics[width=\linewidth]{figs/hopper/p-hop-set3-w4.pdf}
        \caption{\emph{Hopper-v5}.}
        \label{fig:set-3-hopper}
    \end{subfigure}
    \hfill
    \begin{subfigure}[t]{0.32\linewidth}
        \centering
        \includegraphics[width=\linewidth]{figs/hopper/p-hop-set3-w8.pdf}
    \end{subfigure}

    \vspace{0.3cm}

    \begin{subfigure}[t]{0.32\linewidth}
        \centering
        \includegraphics[width=\linewidth]{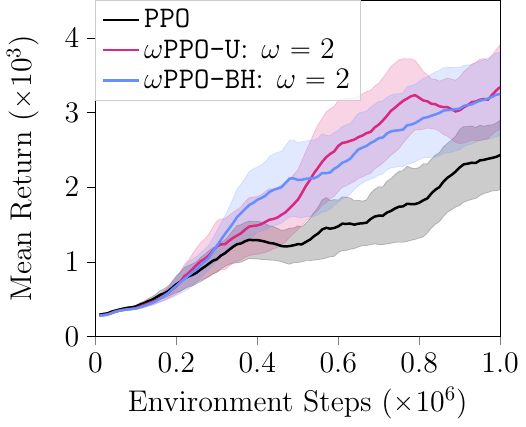}
    \end{subfigure}
    \hfill
    \begin{subfigure}[t]{0.32\linewidth}
        \centering
        \includegraphics[width=\linewidth]{figs/walker/p-wal-set3-w4.pdf}
        \caption{\emph{Walker2d-v5}.}
        \label{fig:set-3-walker}
    \end{subfigure}
    \hfill
    \begin{subfigure}[t]{0.32\linewidth}
        \centering
        \includegraphics[width=\linewidth]{figs/walker/p-wal-set3-w8.pdf}
    \end{subfigure}
    
    \caption{Comparison among \ppo, \bhppo, and \uppo in various MuJoCo environments under different window values in the \emph{fixed $N_{\cB}$} setting (10 runs, mean $\pm$ 95\%).}
    \label{fig:set-3-train}
\end{figure}

\begin{figure}[t]
    \centering
    
    \begin{subfigure}[t]{0.32\linewidth}
        \centering
        \includegraphics[width=\linewidth]{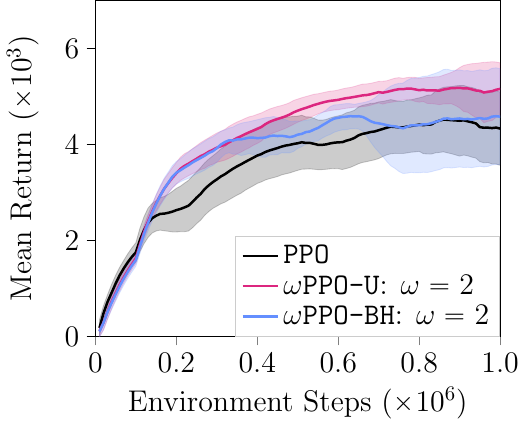}
    \end{subfigure}
    \hfill
    \begin{subfigure}[t]{0.32\linewidth}
        \centering
        \includegraphics[width=\linewidth]{figs/half/p-half-set3-w4.pdf}
        \caption{\emph{HalfCheetah-v5}.}
        \label{fig:set4-half}
    \end{subfigure}
    \hfill
    \begin{subfigure}[t]{0.32\linewidth}
        \centering
        \includegraphics[width=\linewidth]{figs/half/p-half-set3-w8.pdf}
    \end{subfigure}

    \vspace{0.3cm}

    \begin{subfigure}[t]{0.32\linewidth}
        \centering
        \includegraphics[width=\linewidth]{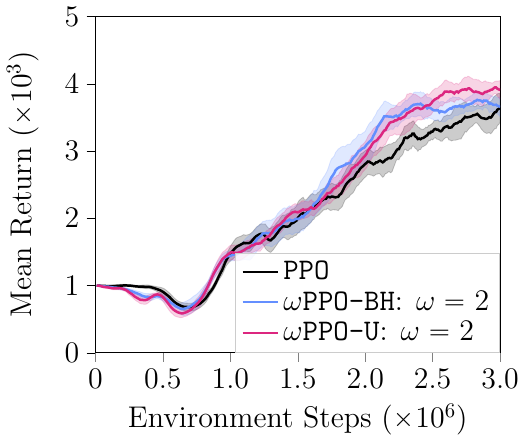}
    \end{subfigure}
    \hfill
    \begin{subfigure}[t]{0.32\linewidth}
        \centering
        \includegraphics[width=\linewidth]{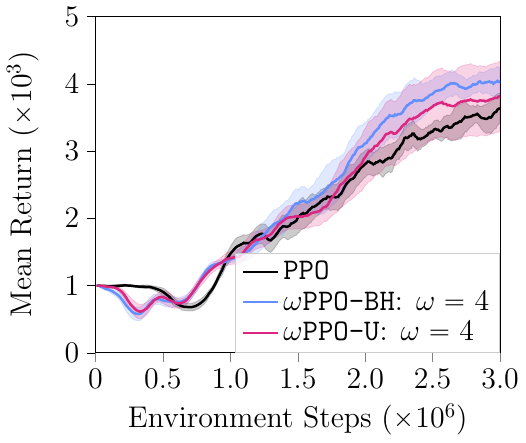}
        \caption{\emph{Ant-v5}.}
        \label{fig:set-3-ant}
    \end{subfigure}
    \hfill
    \begin{subfigure}[t]{0.32\linewidth}
        \centering
        \includegraphics[width=\linewidth]{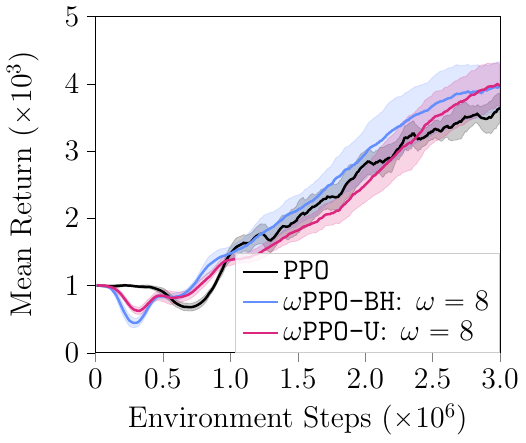}
    \end{subfigure}
    
    \caption{Comparison among \ppo, \bhppo, and \uppo in various MuJoCo environments under different window values in the \emph{Fixed $N_{\cB}$ setting} (10 runs, mean $\pm$ 95\%).}
    \label{fig:set-3-train-2}
\end{figure}

\subsection{Reduced Data Collection Setting} \label{apx:exp-set4}
\begin{table}[t]
    \centering
    \resizebox{\linewidth}{!}{\input{tables/set4-perf}}

    \vspace{0.1cm}
    
    \caption{Average of the last 10 final performance values of the curves shown in the \emph{$H/\omega$ steps} setting (Figures~\ref{fig:set-4-train} and~\ref{fig:set-4-train-2}). Each cell shows the mean $\pm 95\%$ C.I. edges (10 runs).}
    \label{tab:set4-perf}
\end{table}

Here, we analyze \uppo and \bhppo under the most aggressive sample efficiency setting: in contrast to Appendices~\ref{apx:exp-set2} and~\ref{apx:exp-set3}, we keep both $N_{\cB}$ and $B$ fixed at their \ppo values but allow the reuse variants to collect only $H/\omega$ new samples per iteration. The reuse variants therefore perform the same number of updates as \ppo over a composite dataset of the same total size, but with a fresh data fraction of $1/\omega$ per iteration.

Figures~\ref{fig:set-4-train} and~\ref{fig:set-4-train-2} show training evaluation curves and Table~\ref{tab:set4-perf} reports final performance. All experiments consider $\omega \in \{2, 4, 8\}$.
In this setting the results are unambiguous: the considered sample reuse does not help. We conjecture that the primary cause is the severely reduced amount of fresh data available to the critic. Since the critic is updated exclusively on the most recent dataset $\cD_k$, reducing new interactions to $H/\omega$ may directly hinder it, possibly leading to degraded advantage estimates that in turn corrupt actor updates. We anticipate that this is not mitigated by fitting the critic on the whole $\cD(\omega_{k})$ (see Appendix~\ref{apx:exp-on-off}). The single exception is \emph{Hopper-v5} at $\omega=2$ (Figure~\ref{fig:set-4-hopper}), where both \uppo and \bhppo still collect $1024$ fresh samples per iteration, appearing to be sufficient to maintain critic quality in this environment. In all other environments, where \ppo collects only $H=512$ samples, the reuse variants are forced to operate with as few as $64$ fresh samples at $\omega=8$, falling below what appears to be the minimum threshold for stable learning. 

\begin{figure}[t]
    \centering
    \begin{subfigure}[t]{0.32\linewidth}
        \centering
        \includegraphics[width=\linewidth]{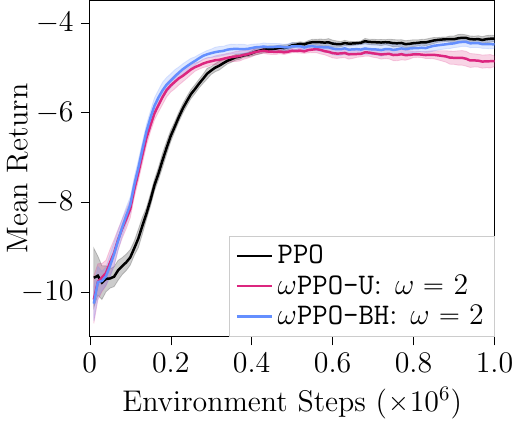}
    \end{subfigure}
    \hfill
    \begin{subfigure}[t]{0.32\linewidth}
        \centering
        \includegraphics[width=\linewidth]{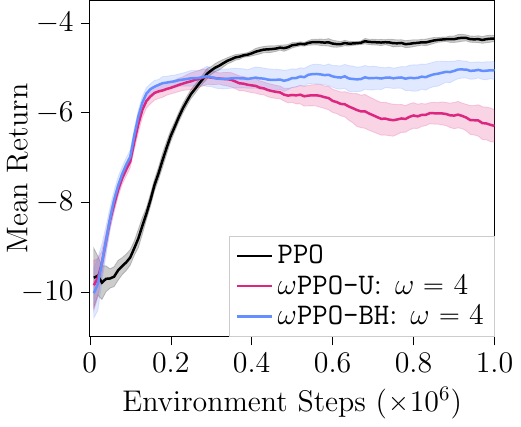}
        \caption{\emph{Reacher-v5}.}
        \label{fig:set-4-reacher}
    \end{subfigure}
    \hfill
    \begin{subfigure}[t]{0.32\linewidth}
        \centering
        \includegraphics[width=\linewidth]{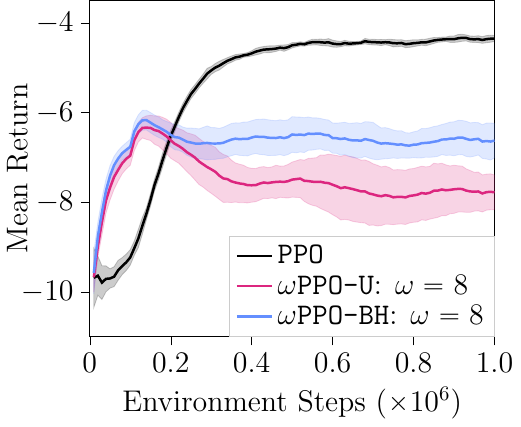}
    \end{subfigure}

    \vspace{0.3cm}

    \begin{subfigure}[t]{0.32\linewidth}
        \centering
        \includegraphics[width=\linewidth]{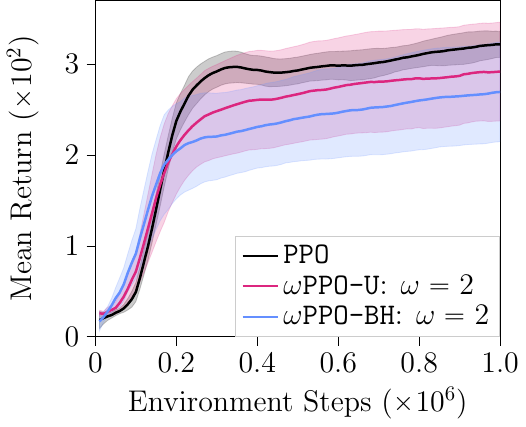}
    \end{subfigure}
    \hfill
    \begin{subfigure}[t]{0.32\linewidth}
        \centering
        \includegraphics[width=\linewidth]{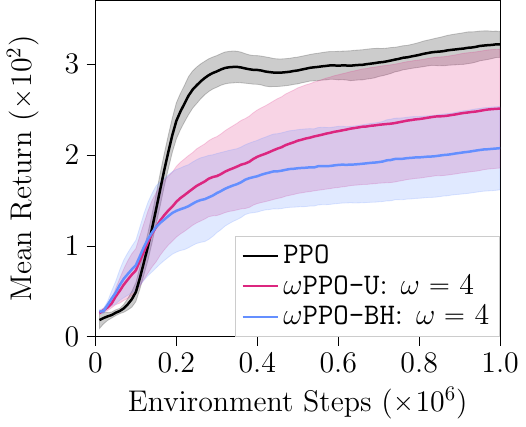}
        \caption{\emph{Swimmer-v5}.}
        \label{fig:set-4-swimmer}
    \end{subfigure}
    \hfill
    \begin{subfigure}[t]{0.32\linewidth}
        \centering
        \includegraphics[width=\linewidth]{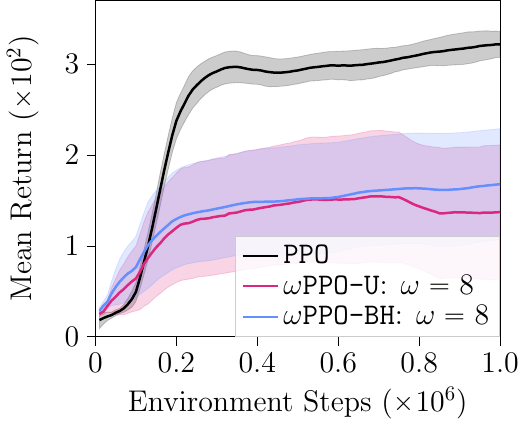}
    \end{subfigure}

    \vspace{0.3cm}

    \begin{subfigure}[t]{0.32\linewidth}
        \centering
        \includegraphics[width=\linewidth]{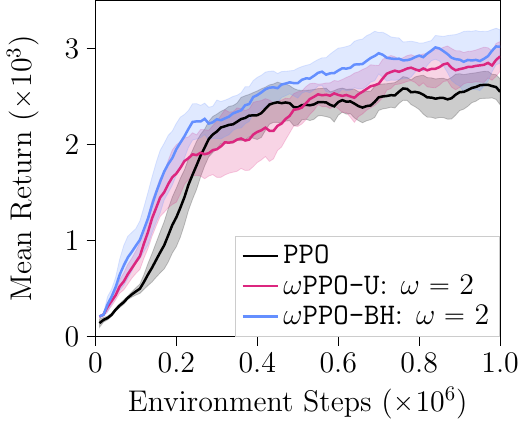}
    \end{subfigure}
    \hfill
    \begin{subfigure}[t]{0.32\linewidth}
        \centering
        \includegraphics[width=\linewidth]{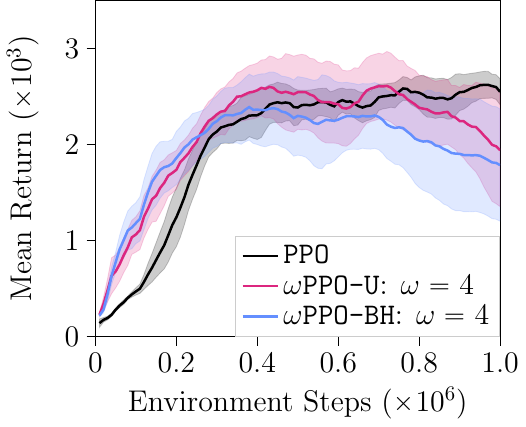}
        \caption{\emph{Hopper-v5}.}
        \label{fig:set-4-hopper}
    \end{subfigure}
    \hfill
    \begin{subfigure}[t]{0.32\linewidth}
        \centering
        \includegraphics[width=\linewidth]{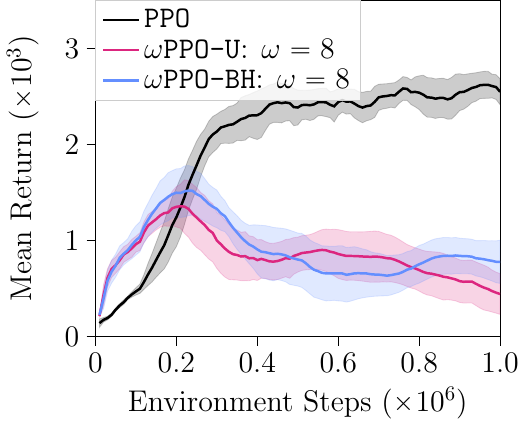}
    \end{subfigure}

    \vspace{0.3cm}

    \begin{subfigure}[t]{0.32\linewidth}
        \centering
        \includegraphics[width=\linewidth]{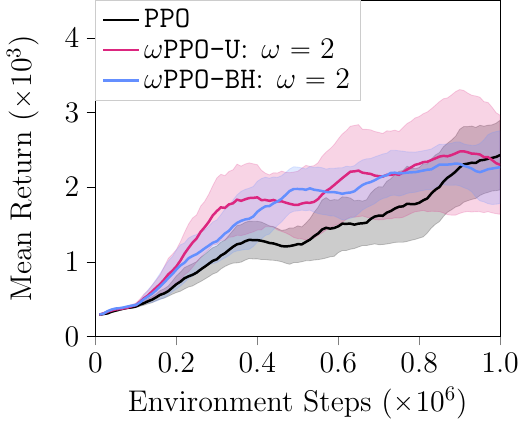}
    \end{subfigure}
    \hfill
    \begin{subfigure}[t]{0.32\linewidth}
        \centering
        \includegraphics[width=\linewidth]{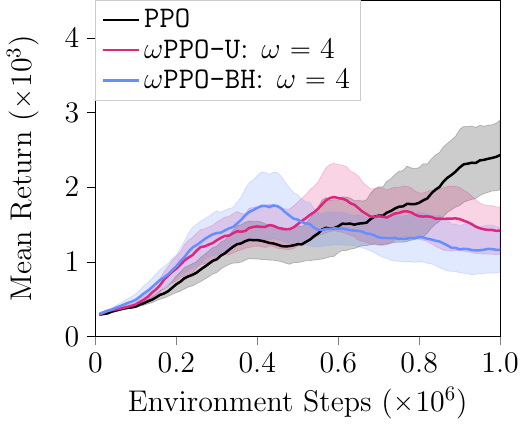}
        \caption{\emph{Walker2d-v5}.}
        \label{fig:set-4-walker}
    \end{subfigure}
    \hfill
    \begin{subfigure}[t]{0.32\linewidth}
        \centering
        \includegraphics[width=\linewidth]{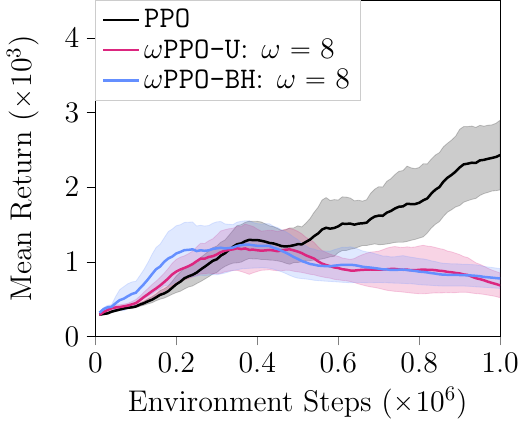}
    \end{subfigure}
    
    \caption{Comparison among \ppo, \bhppo, and \uppo in various MuJoCo environments under different window values in the \emph{$H/\omega$ steps} setting (10 runs, mean $\pm$ 95\%).}
    \label{fig:set-4-train}
\end{figure}

\begin{figure}[t]
    \centering
    
    \begin{subfigure}[t]{0.32\linewidth}
        \centering
        \includegraphics[width=\linewidth]{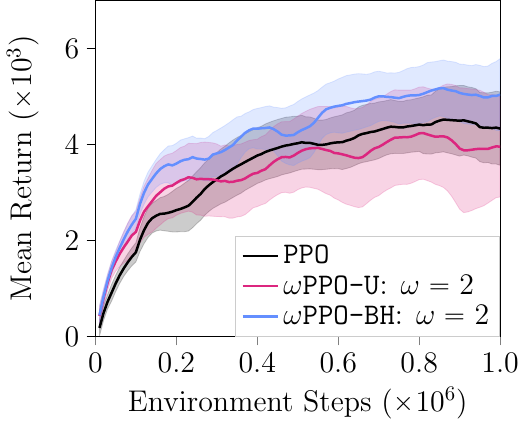}
    \end{subfigure}
    \hfill
    \begin{subfigure}[t]{0.32\linewidth}
        \centering
        \includegraphics[width=\linewidth]{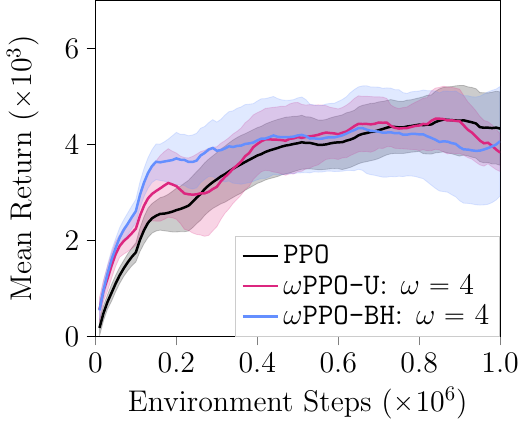}
        \caption{\emph{HalfCheetah-v5}.}
        \label{fig:set-4-half}
    \end{subfigure}
    \hfill
    \begin{subfigure}[t]{0.32\linewidth}
        \centering
        \includegraphics[width=\linewidth]{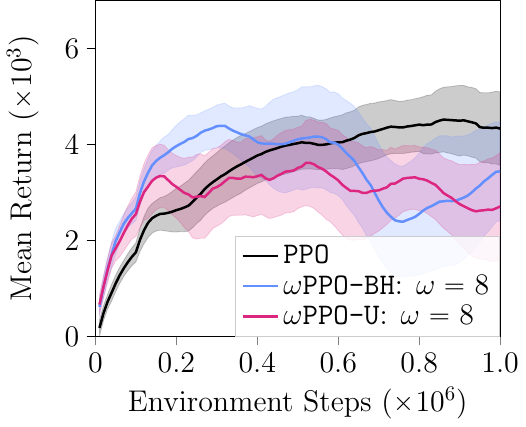}
    \end{subfigure}

    \vspace{0.3cm}

    \begin{subfigure}[t]{0.32\linewidth}
        \centering
        \includegraphics[width=\linewidth]{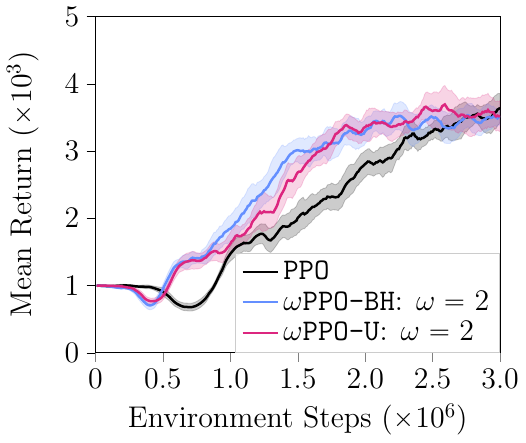}
    \end{subfigure}
    \hfill
    \begin{subfigure}[t]{0.32\linewidth}
        \centering
        \includegraphics[width=\linewidth]{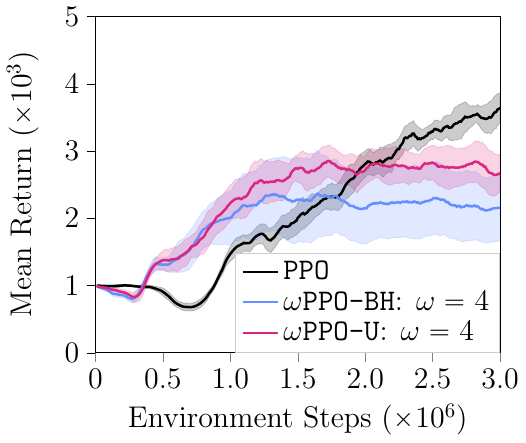}
        \caption{\emph{Ant-v5}.}
        \label{fig:set-4-ant}
    \end{subfigure}
    \hfill
    \begin{subfigure}[t]{0.32\linewidth}
        \centering
        \includegraphics[width=\linewidth]{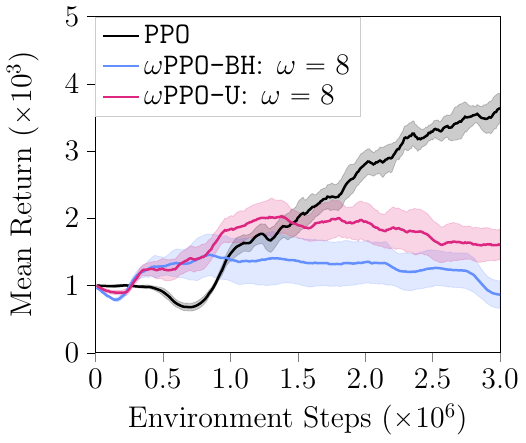}
    \end{subfigure}
    
    \caption{Comparison among \ppo, \bhppo, and \uppo in various MuJoCo environments under different window values in the \emph{$H/\omega$ steps} setting (10 runs, mean $\pm$ 95\%).}
    \label{fig:set-4-train-2}
\end{figure}

\subsection{\uppo vs.\ \bhppo}
In this experiment, we aim to understand how the two importance weighting schemes adopted by \uppo and \bhppo differ in practice in their treatment of on and off-policy data.
To this end, we present for all the considered environments and for all the comparison settings, the clip fraction and ESS diagnostics described in Appendix~\ref{apx:exp-details}.
Each metric is evaluated before the optimization phase at every iteration, allowing us to distinguish the distributional shift already present at the start of each iteration.

Before presenting the results, we recall that:
\begin{itemize}[noitemsep,leftmargin=*,topsep=-3pt]
    \item The clip fraction $\mathrm{CF}^{\Box}_i$ ($\Box \in \{\texttt{U}, \texttt{BH}\}$) captures how frequently the clipping mechanism is active for data from window index $i$. A high value indicates that a large fraction of transitions are effectively excluded from the gradient signal, suggesting a significant distributional shift between the current policy $\pi_{\vtheta}$ and the one at the denominator of the importance weight, i.e., either the policy under which samples have been collected or the uniform mixture policy.
    \item The normalized $\mathrm{ESS}^{\Box}_i \in [0, 1]$ ($\Box \in \{\texttt{U}, \texttt{BH}\}$) measures the statistical efficiency of the importance weights: it equals $1$ when the current policy matches the sampling distribution exactly, and decreases as the two distributions diverge. 
\end{itemize}
Together, these two metrics provide complementary views of how stale the reused data is and how each method manages the resulting off-policy correction.
Specifically, we group the plots of these metrics per environment:
\begin{itemize}[noitemsep,leftmargin=*,topsep=-3pt]
    \item Figure~\ref{fig:diag-reacher} shows the metrics for the \emph{Reacher-v5} environment;
    \item Figure~\ref{fig:diag-swimmer} shows the metrics for the \emph{Swimmer-v5} environment;
    \item Figure~\ref{fig:diag-hopper} shows the metrics for the \emph{Hopper-v5} environment;
    \item Figure~\ref{fig:diag-walker} shows the metrics for the \emph{Walker2d-v5} environment;
    \item Figure~\ref{fig:diag-half} shows the metrics for the \emph{HalfCheetah-v5} environment;
    \item Figure~\ref{fig:diag-ant} shows the metrics for the \emph{Ant-v5} environment.
\end{itemize}
Each figure reports the metric at hand averaged over the entire training and grouped by window index.

\textbf{Window Index Effect Under \uppo.}~~Under the vanilla uniform weighting scheme, older data  corresponding to larger window indices $i$ consistently exhibit higher $\mathrm{CF}^{\texttt{U}}_i$ and lower $\mathrm{ESS}^{\texttt{U}}_i$, before optimization.
This is expected: since the policy has already been updated on those samples in previous iterations, the importance ratio $w^{\texttt{U}}_{i,j}$ for older transitions deviates more from $1$ than for recently collected data, reflecting the growing distributional shift between the current policy and the one that generated them.
This is also reflected by the $\mathrm{ESS}^{\texttt{U}}_i$ which decreases monotonically with window index.

\textbf{The Effect of the BH Correction.}~~Since \bhppo weights all samples w.r.t. the mixture policy $\pi^{\mathrm{BH}}_{\omega_{k}}$, both $\mathrm{CF}^{\texttt{BH}}_i$ and $\mathrm{ESS}^{\texttt{BH}}_i$ are approximately constant across all window indices before optimization, effectively equalizing the distribution of the IWs across window indices (the advantages of older samples, however, remain stale).
Crucially, \bhppo also maintains a substantially higher and more uniform $\mathrm{ESS}^{\texttt{BH}}_i$ across all window indices compared to \uppo.
This may allow \bhppo to extract more reliable gradient signal from samples coming from older iterations.
It is worth noting, however, that this equalization comes at a cost: since the BH ratios of the most recent data (window index $i=0$) generally differ from $1$, a fraction of new transitions already falls outside the clipping range before any gradient update, conversely to what happens in both \ppo and \uppo.

\begin{figure}[t]
    \centering
    \begin{subfigure}[t]{0.49\linewidth}
        \includegraphics[width=\linewidth]{figs/reacher-diag/p-diag-reacher-set1-clip.pdf}
        \caption{Clip Fraction diagnostic: Fixed $B$.}
        \label{fig:diag-reacher-set1-clip}
    \end{subfigure}
    \hfill
    \begin{subfigure}[t]{0.49\linewidth}
        \includegraphics[width=\linewidth]{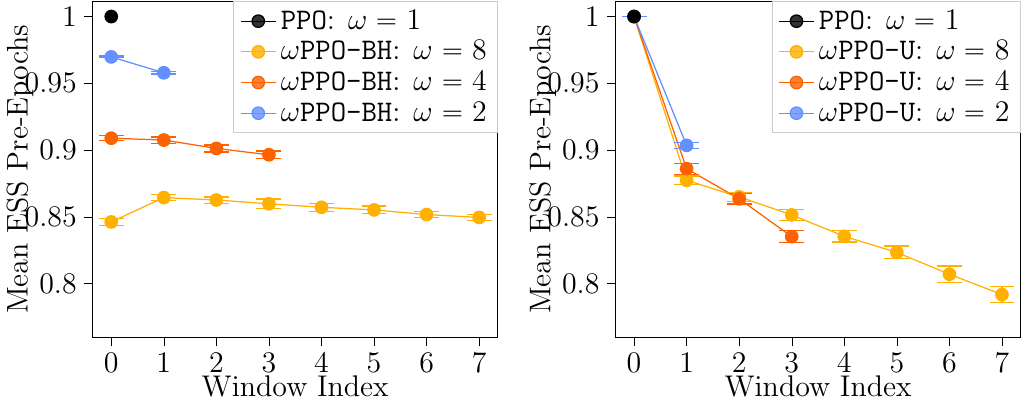}
        \caption{ESS diagnostic: Fixed $B$.}
        \label{fig:diag-reacher-set1-ess}
    \end{subfigure}
    \vspace{0.5cm}

    \begin{subfigure}[t]{0.49\linewidth}
        \includegraphics[width=\linewidth]{figs/reacher-diag/p-diag-reacher-set3-clip.pdf}
        \caption{Clip Fraction diagnostic: Fixed $N_{\cB}$.}
        \label{fig:diag-reacher-set3-clip}
    \end{subfigure}
    \hfill
    \begin{subfigure}[t]{0.49\linewidth}
        \includegraphics[width=\linewidth]{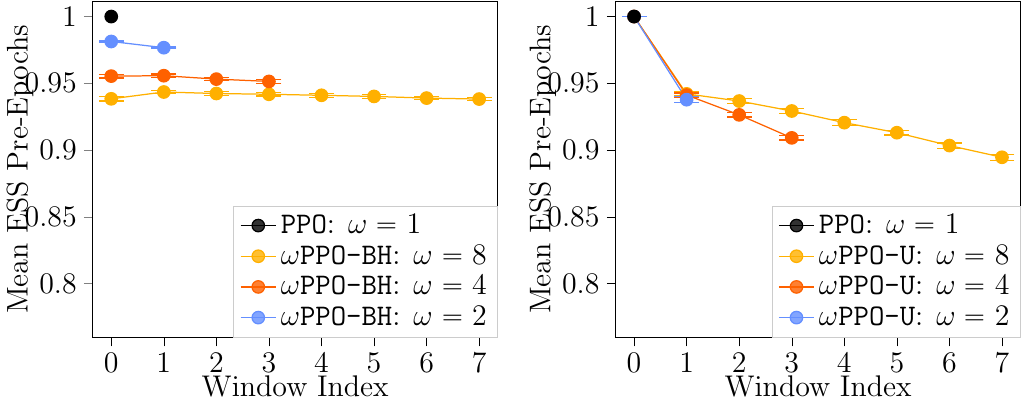}
        \caption{ESS diagnostic: Fixed $N_{\cB}$.}
        \label{fig:diag-reacher-set3-ess}
    \end{subfigure}
    \vspace{0.5cm}

    \begin{subfigure}[t]{0.49\linewidth}
        \includegraphics[width=\linewidth]{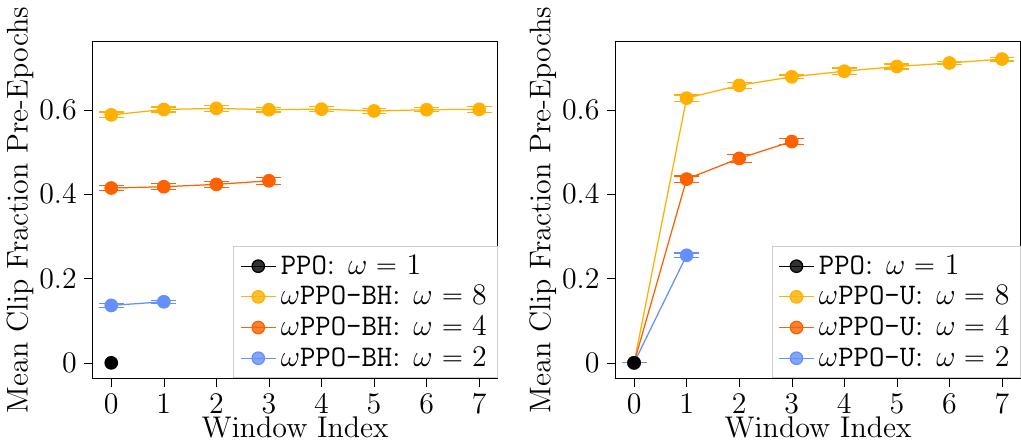}
        \caption{Clip Fraction diagnostic: Fixed $H/\omega$ Steps.}
        \label{fig:diag-reacher-set4-clip}
    \end{subfigure}
    \hfill
    \begin{subfigure}[t]{0.49\linewidth}
        \includegraphics[width=\linewidth]{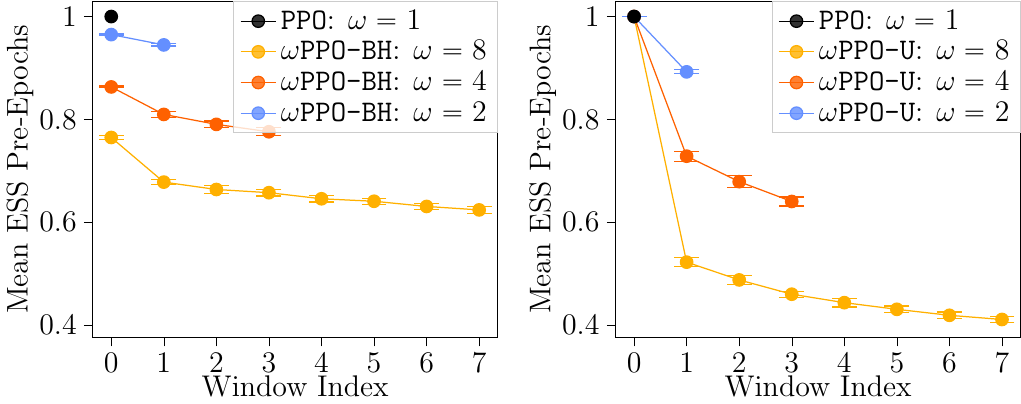}
        \caption{ESS diagnostic: Fixed $H/\omega$ Steps.}
        \label{fig:diag-reacher-set4-ess}
    \end{subfigure}

    \caption{Clip Fraction and ESS diagnostics in \emph{Reacher-v5} (10 runs, mean $\pm$ 95\% C.I.).}
    \label{fig:diag-reacher}
\end{figure}

\begin{figure}[t]
    \centering
    \begin{subfigure}[t]{0.49\linewidth}
        \includegraphics[width=\linewidth]{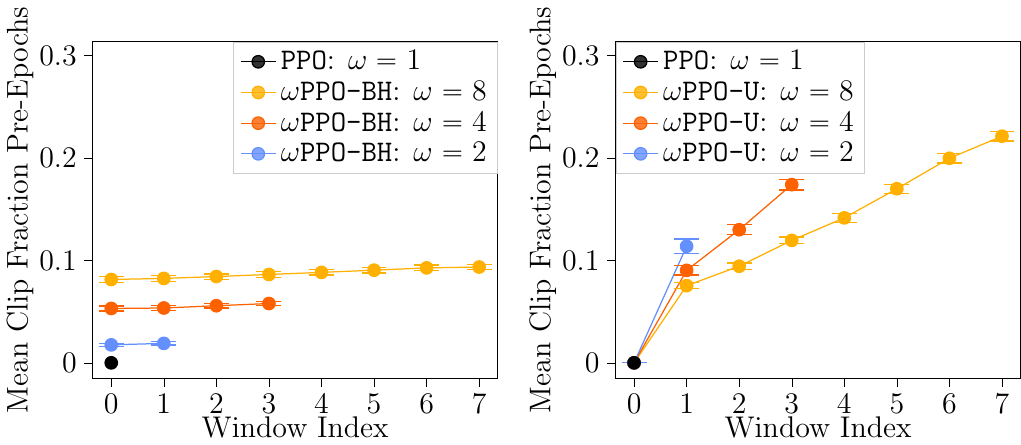}
        \caption{Clip Fraction diagnostic: Fixed $B$.}
        \label{fig:diag-swimmer-set1-clip}
    \end{subfigure}
    \hfill
    \begin{subfigure}[t]{0.49\linewidth}
        \includegraphics[width=\linewidth]{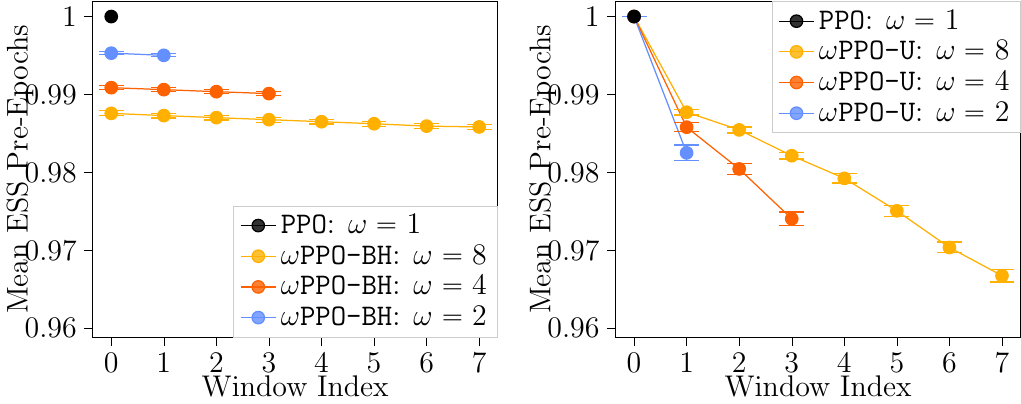}
        \caption{ESS diagnostic: Fixed $B$.}
        \label{fig:diag-swimmer-set1-ess}
    \end{subfigure}
    \vspace{0.5cm}

    \begin{subfigure}[t]{0.49\linewidth}
        \includegraphics[width=\linewidth]{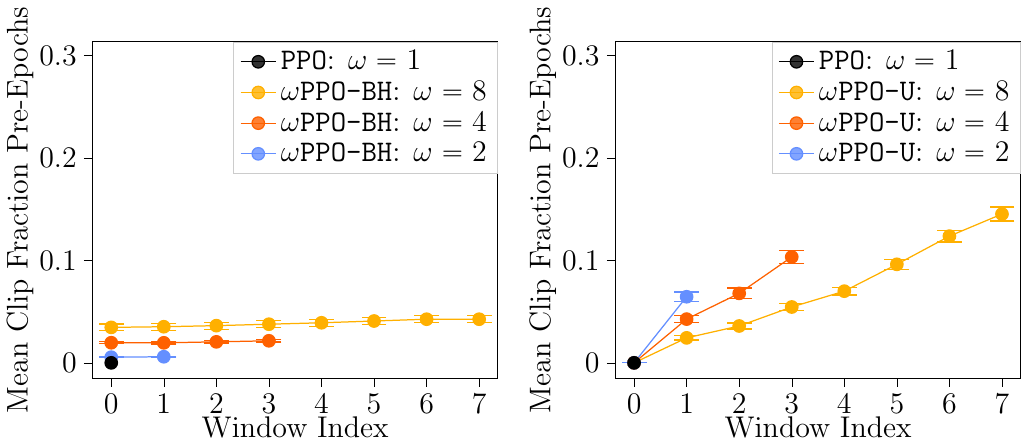}
        \caption{Clip Fraction diagnostic: Fixed $N_{\cB}$.}
        \label{fig:diag-swimmer-set3-clip}
    \end{subfigure}
    \hfill
    \begin{subfigure}[t]{0.49\linewidth}
        \includegraphics[width=\linewidth]{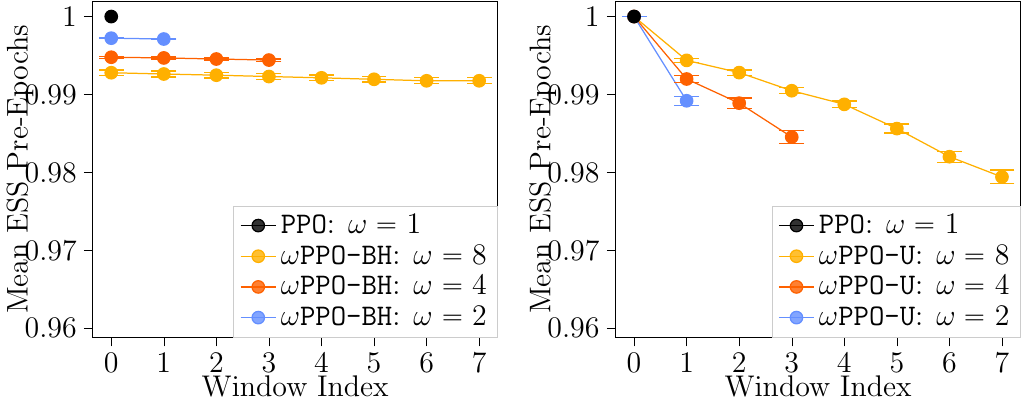}
        \caption{ESS diagnostic: Fixed $N_{\cB}$.}
        \label{fig:diag-swimmer-set3-ess}
    \end{subfigure}
    \vspace{0.5cm}

    \begin{subfigure}[t]{0.49\linewidth}
        \includegraphics[width=\linewidth]{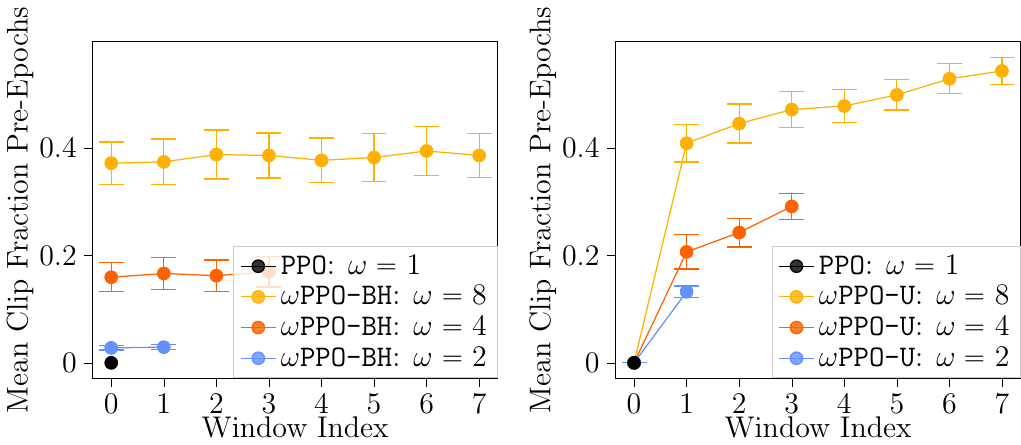}
        \caption{Clip Fraction diagnostic: Fixed $H/\omega$ Steps.}
        \label{fig:diag-swimmer-set4-clip}
    \end{subfigure}
    \hfill
    \begin{subfigure}[t]{0.49\linewidth}
        \includegraphics[width=\linewidth]{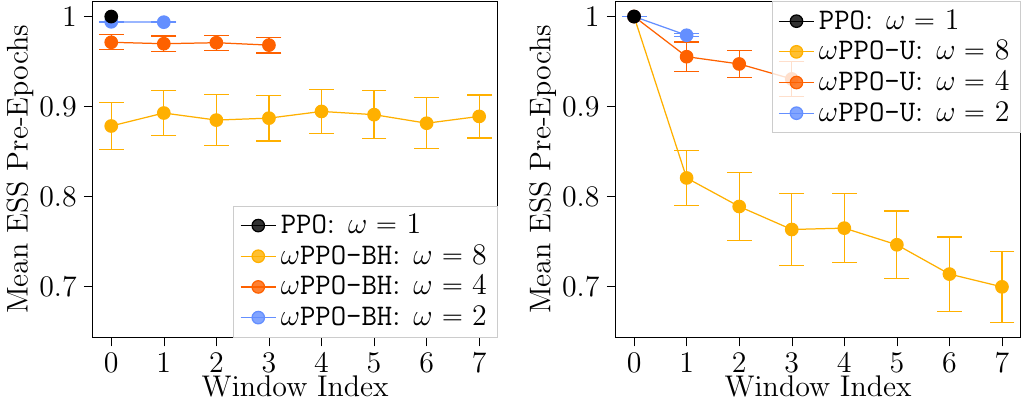}
        \caption{ESS diagnostic: Fixed $H/\omega$ Steps.}
        \label{fig:diag-swimmer-set4-ess}
    \end{subfigure}

    \caption{Clip Fraction and ESS diagnostics in \emph{Swimmer-v5} (10 runs, mean $\pm$ 95\% C.I.).}
    \label{fig:diag-swimmer}
\end{figure}

\begin{figure}[t]
    \centering
    \begin{subfigure}[t]{0.49\linewidth}
        \includegraphics[width=\linewidth]{figs/hopper-diag/p-diag-hopper-set1-clip.pdf}
        \caption{Clip Fraction diagnostic: Fixed $B$.}
        \label{fig:diag-hopper-set1-clip}
    \end{subfigure}
    \hfill
    \begin{subfigure}[t]{0.49\linewidth}
        \includegraphics[width=\linewidth]{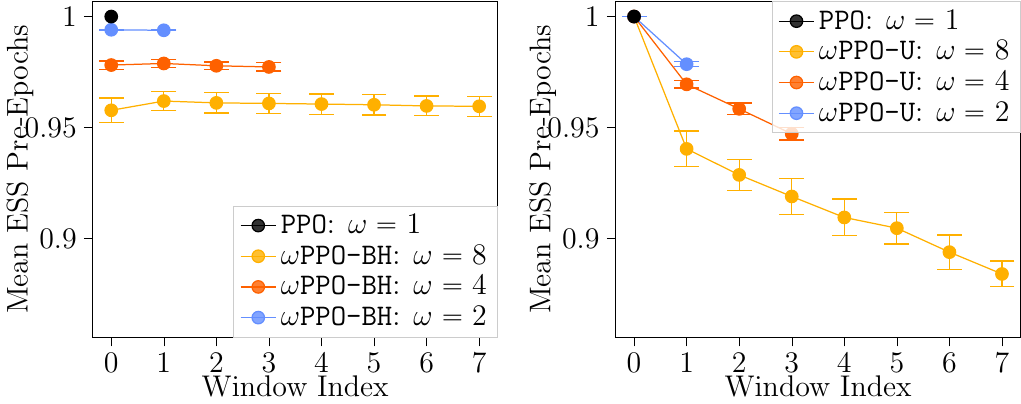}
        \caption{ESS diagnostic: Fixed $B$.}
        \label{fig:diag-hopper-set1-ess}
    \end{subfigure}
    \vspace{0.5cm}

    \begin{subfigure}[t]{0.49\linewidth}
        \includegraphics[width=\linewidth]{figs/hopper-diag/p-diag-hopper-set3-clip.pdf}
        \caption{Clip Fraction diagnostic: Fixed $N_{\cB}$.}
        \label{fig:diag-hopper-set3-clip}
    \end{subfigure}
    \hfill
    \begin{subfigure}[t]{0.49\linewidth}
        \includegraphics[width=\linewidth]{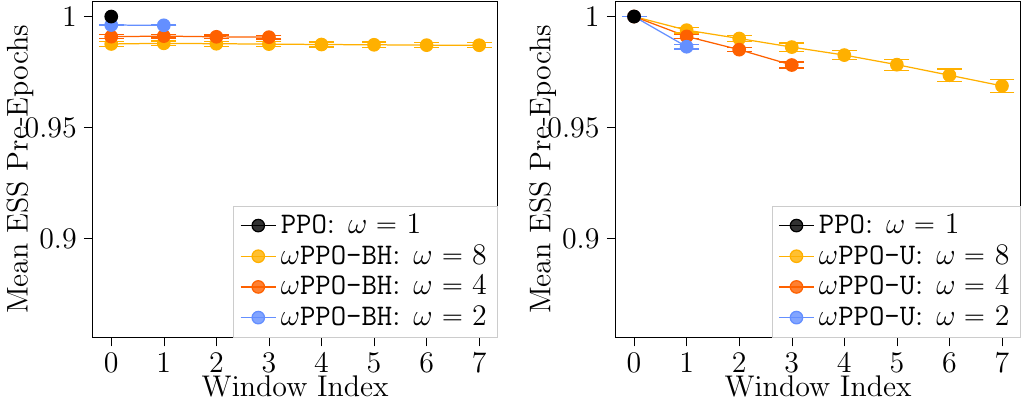}
        \caption{ESS diagnostic: Fixed $N_{\cB}$.}
        \label{fig:diag-hopper-set3-ess}
    \end{subfigure}
    \vspace{0.5cm}

    \begin{subfigure}[t]{0.49\linewidth}
        \includegraphics[width=\linewidth]{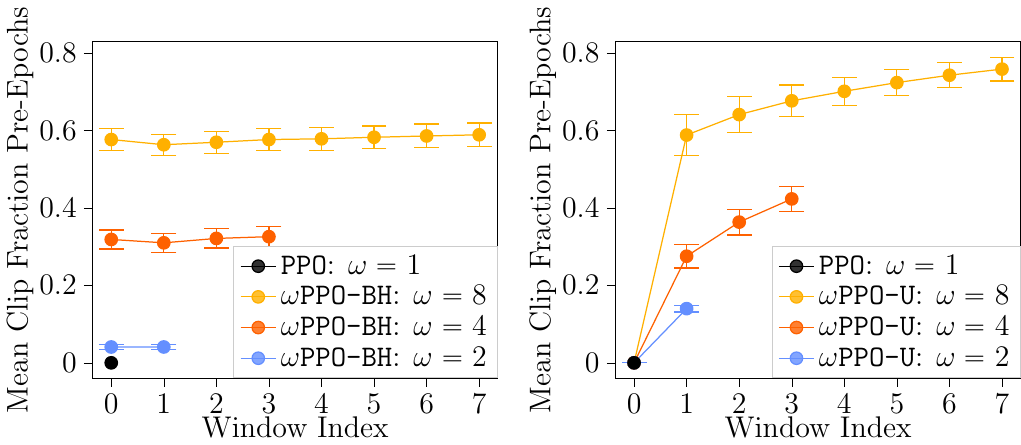}
        \caption{Clip Fraction diagnostic: Fixed $H/\omega$ Steps.}
        \label{fig:diag-hopper-set4-clip}
    \end{subfigure}
    \hfill
    \begin{subfigure}[t]{0.49\linewidth}
        \includegraphics[width=\linewidth]{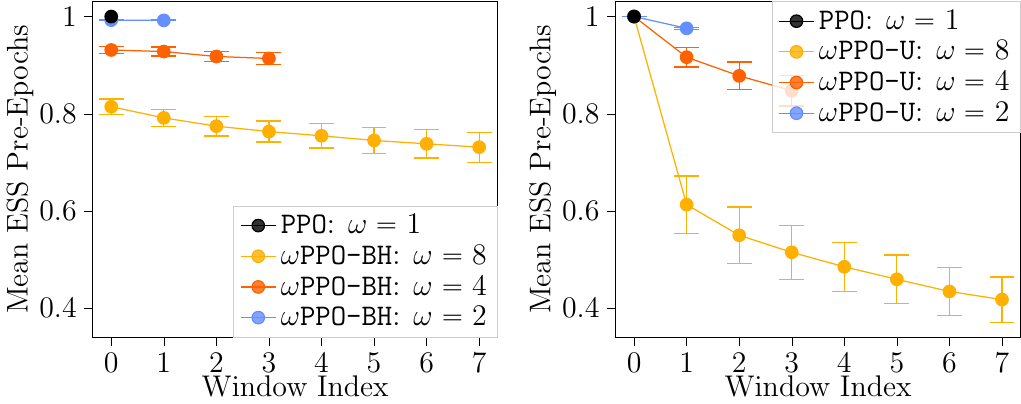}
        \caption{ESS diagnostic: Fixed $H/\omega$ Steps.}
        \label{fig:diag-hopper-set4-ess}
    \end{subfigure}

    \caption{Clip Fraction and ESS diagnostics in \emph{Hopper-v5} (10 runs, mean $\pm$ 95\% C.I.).}
    \label{fig:diag-hopper}
\end{figure}

\begin{figure}[t]
    \centering
    \begin{subfigure}[t]{0.49\linewidth}
        \includegraphics[width=\linewidth]{figs/walker-diag/p-diag-walker-set1-clip.pdf}
        \caption{Clip Fraction diagnostic: Fixed $B$.}
        \label{fig:diag-walker-set1-clip}
    \end{subfigure}
    \hfill
    \begin{subfigure}[t]{0.49\linewidth}
        \includegraphics[width=\linewidth]{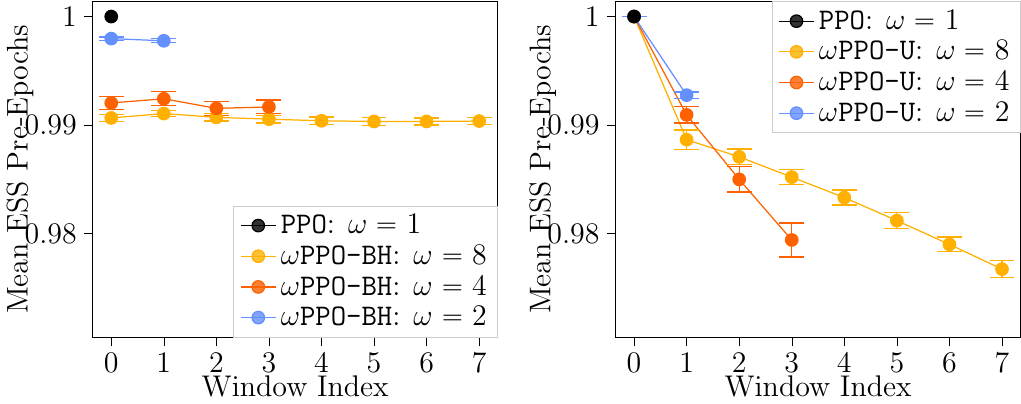}
        \caption{ESS diagnostic: Fixed $B$.}
        \label{fig:diag-walker-set1-ess}
    \end{subfigure}
    \vspace{0.5cm}

    \begin{subfigure}[t]{0.49\linewidth}
        \includegraphics[width=\linewidth]{figs/walker-diag/p-diag-walker-set3-clip.pdf}
        \caption{Clip Fraction diagnostic: Fixed $N_{\cB}$.}
        \label{fig:diag-walker-set3-clip}
    \end{subfigure}
    \hfill
    \begin{subfigure}[t]{0.49\linewidth}
        \includegraphics[width=\linewidth]{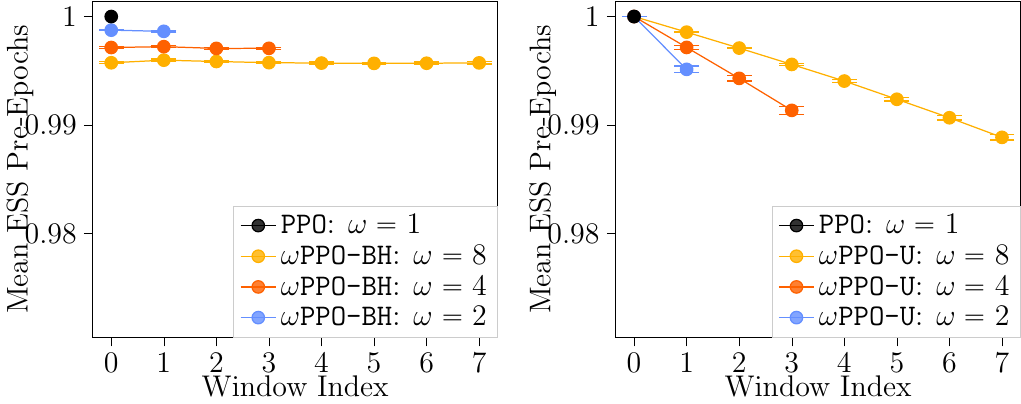}
        \caption{ESS diagnostic: Fixed $N_{\cB}$.}
        \label{fig:diag-walker-set3-ess}
    \end{subfigure}
    \vspace{0.5cm}

    \begin{subfigure}[t]{0.49\linewidth}
        \includegraphics[width=\linewidth]{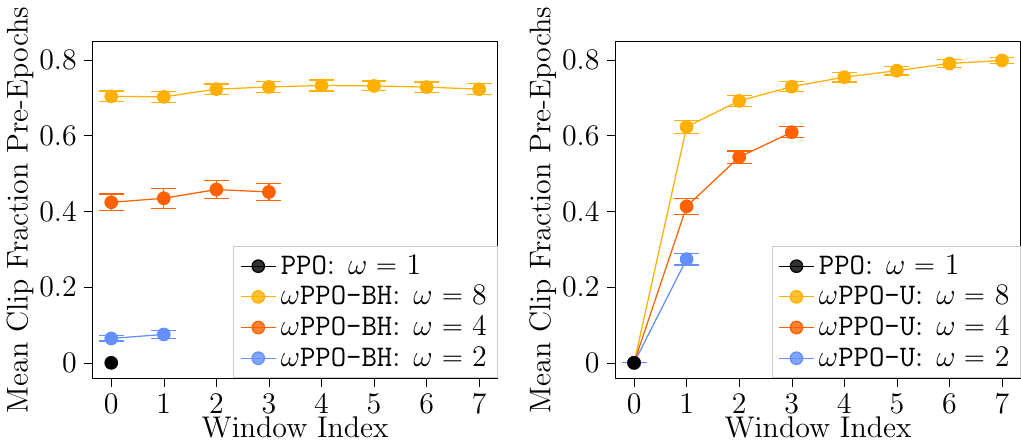}
        \caption{Clip Fraction diagnostic: Fixed $H/\omega$ Steps.}
        \label{fig:diag-walker-set4-clip}
    \end{subfigure}
    \hfill
    \begin{subfigure}[t]{0.49\linewidth}
        \includegraphics[width=\linewidth]{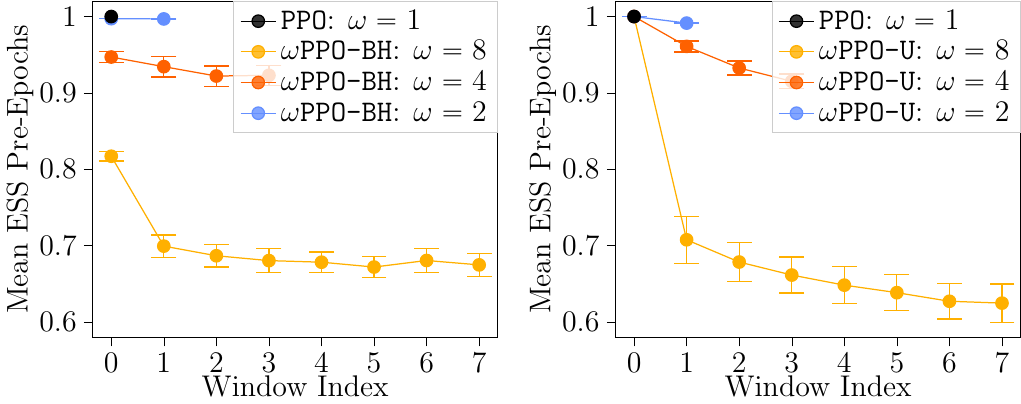}
        \caption{ESS diagnostic: Fixed $H/\omega$ Steps.}
        \label{fig:diag-walker-set4-ess}
    \end{subfigure}

    \caption{Clip Fraction and ESS diagnostics in \emph{Walker2d-v5} (10 runs, mean $\pm$ 95\% C.I.).}
    \label{fig:diag-walker}
\end{figure}

\begin{figure}[t]
    \centering
    \begin{subfigure}[t]{0.49\linewidth}
        \includegraphics[width=\linewidth]{figs/half-diag/p-diag-half-set1-clip.pdf}
        \caption{Clip Fraction diagnostic: Fixed $B$.}
        \label{fig:diag-half-set1-clip}
    \end{subfigure}
    \hfill
    \begin{subfigure}[t]{0.49\linewidth}
        \includegraphics[width=\linewidth]{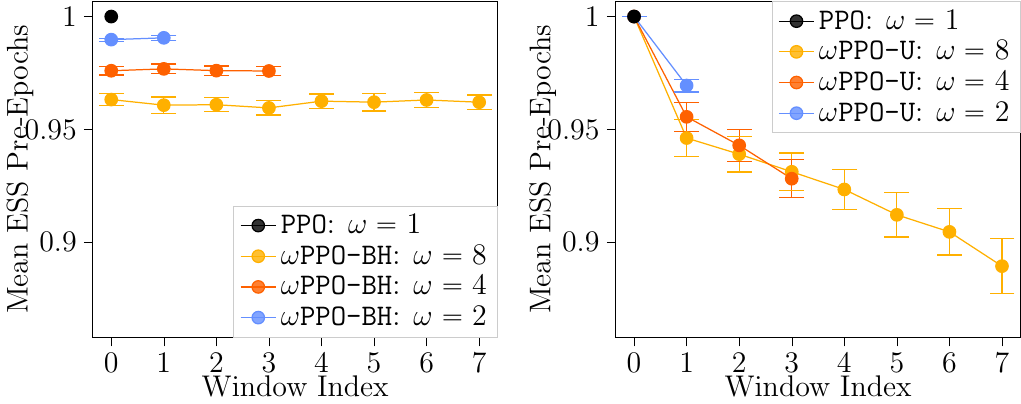}
        \caption{ESS diagnostic: Fixed $B$.}
        \label{fig:diag-half-set1-ess}
    \end{subfigure}
    \vspace{0.5cm}

    \begin{subfigure}[t]{0.49\linewidth}
        \includegraphics[width=\linewidth]{figs/half-diag/p-diag-half-set3-clip.pdf}
        \caption{Clip Fraction diagnostic: Fixed $N_{\cB}$.}
        \label{fig:diag-half-set3-clip}
    \end{subfigure}
    \hfill
    \begin{subfigure}[t]{0.49\linewidth}
        \includegraphics[width=\linewidth]{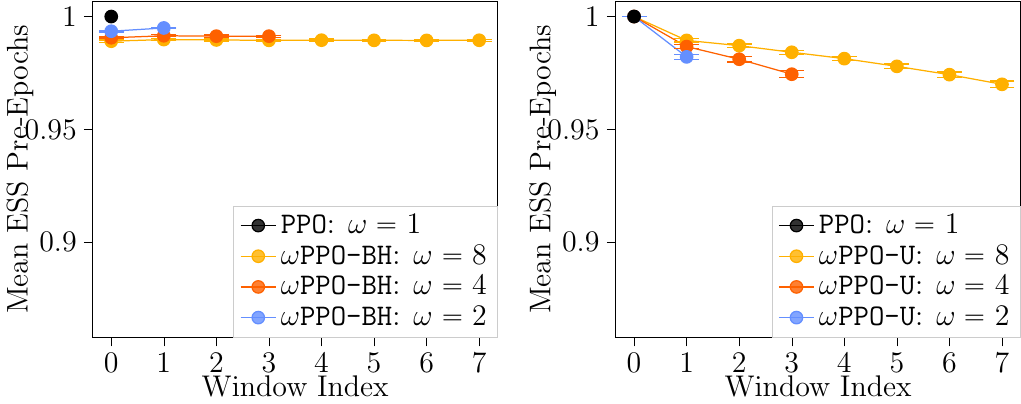}
        \caption{ESS diagnostic: Fixed $N_{\cB}$.}
        \label{fig:diag-half-set3-ess}
    \end{subfigure}
    \vspace{0.5cm}

    \begin{subfigure}[t]{0.49\linewidth}
        \includegraphics[width=\linewidth]{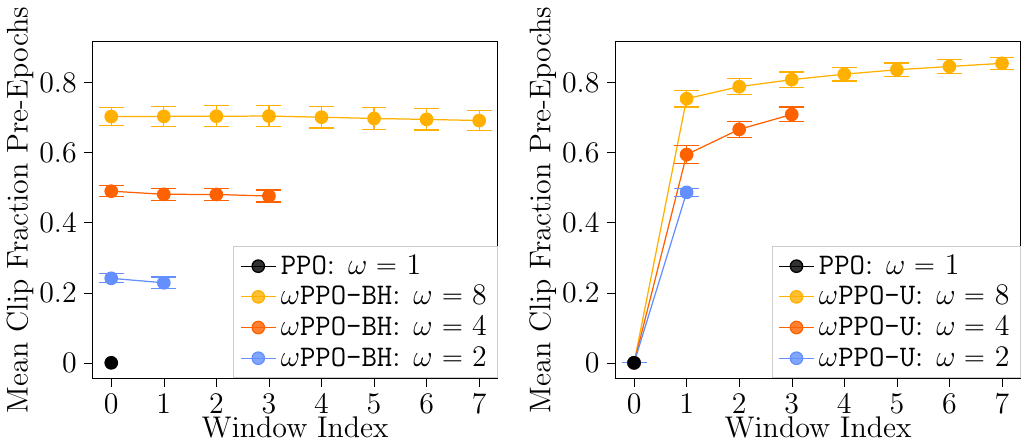}
        \caption{Clip Fraction diagnostic: Fixed $H/\omega$ Steps.}
        \label{fig:diag-half-set4-clip}
    \end{subfigure}
    \hfill
    \begin{subfigure}[t]{0.49\linewidth}
        \includegraphics[width=\linewidth]{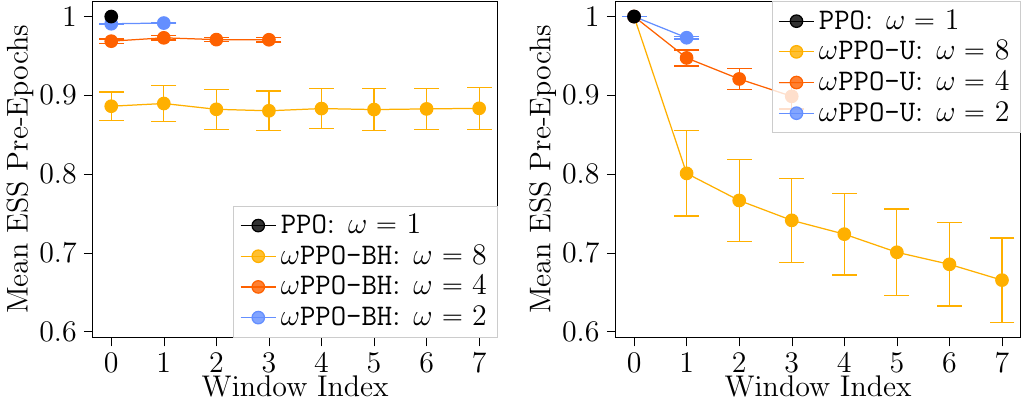}
        \caption{ESS diagnostic: Fixed $H/\omega$ Steps.}
        \label{fig:diag-half-set4-ess}
    \end{subfigure}

    \caption{Clip Fraction and ESS diagnostics in \emph{HalfCheetah-v5} (10 runs, mean $\pm$ 95\% C.I.).}
    \label{fig:diag-half}
\end{figure}

\begin{figure}[t]
    \centering
    \begin{subfigure}[t]{0.49\linewidth}
        \includegraphics[width=\linewidth]{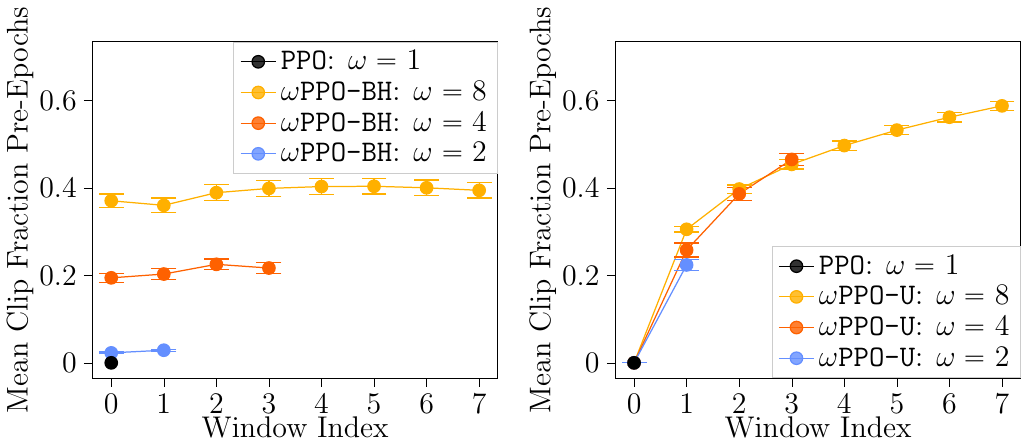}
        \caption{Clip Fraction diagnostic: Fixed $B$.}
        \label{fig:diag-ant-set1-clip}
    \end{subfigure}
    \hfill
    \begin{subfigure}[t]{0.49\linewidth}
        \includegraphics[width=\linewidth]{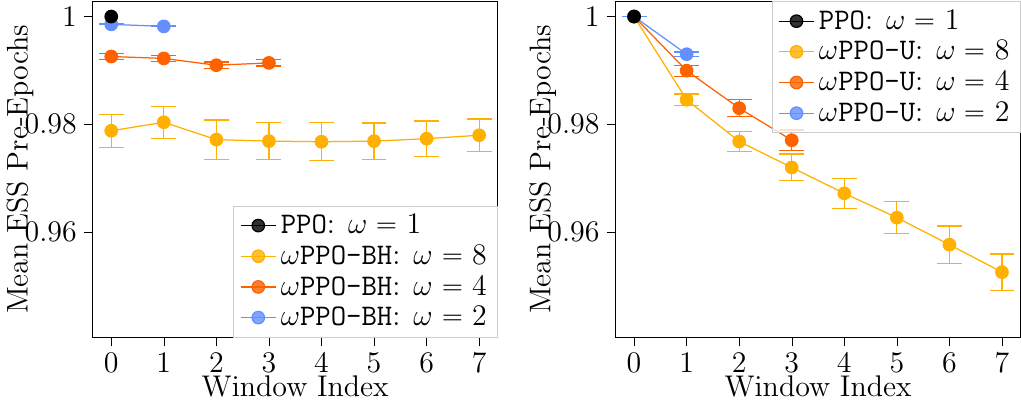}
        \caption{ESS diagnostic: Fixed $B$.}
        \label{fig:diag-ant-set1-ess}
    \end{subfigure}
    \vspace{0.5cm}

    \begin{subfigure}[t]{0.49\linewidth}
        \includegraphics[width=\linewidth]{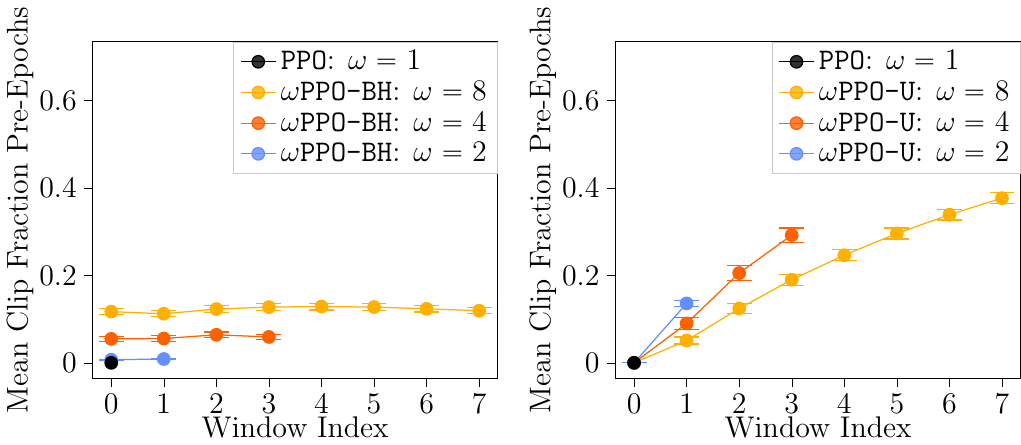}
        \caption{Clip Fraction diagnostic: Fixed $N_{\cB}$.}
        \label{fig:diag-ant-set3-clip}
    \end{subfigure}
    \hfill
    \begin{subfigure}[t]{0.49\linewidth}
        \includegraphics[width=\linewidth]{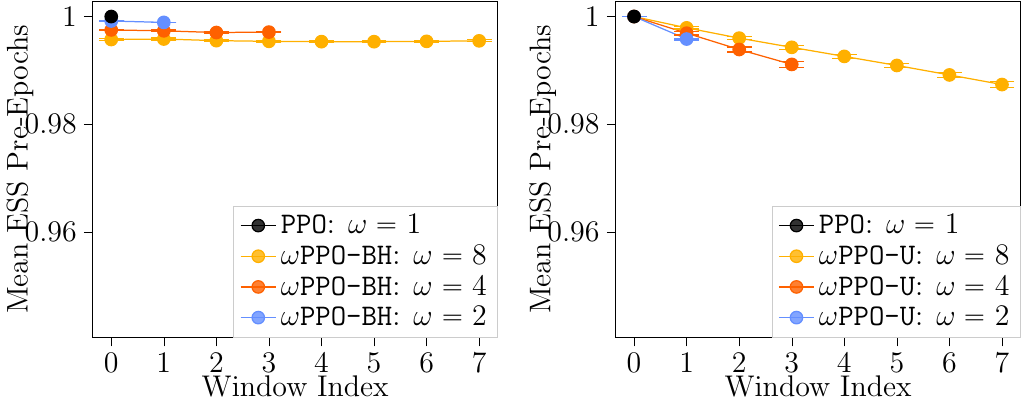}
        \caption{ESS diagnostic: Fixed $N_{\cB}$.}
        \label{fig:diag-ant-set3-ess}
    \end{subfigure}
    \vspace{0.5cm}

    \begin{subfigure}[t]{0.49\linewidth}
        \includegraphics[width=\linewidth]{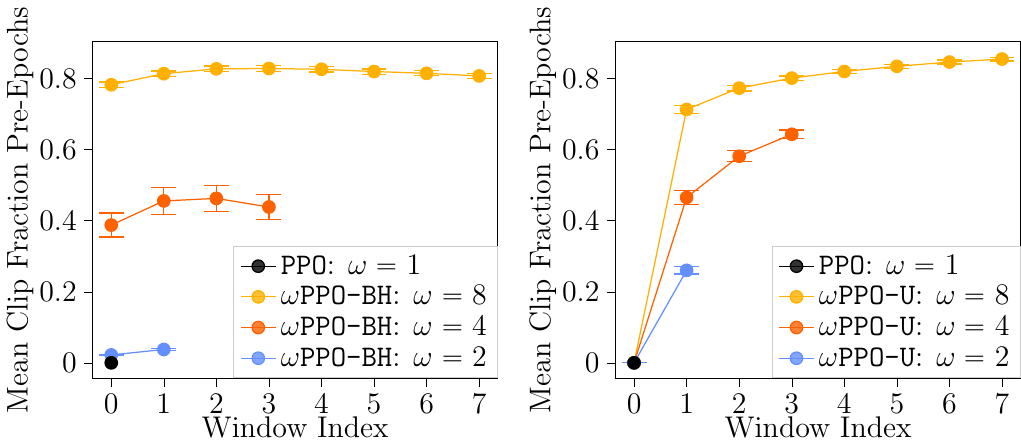}
        \caption{Clip Fraction diagnostic: Fixed $H/\omega$ Steps.}
        \label{fig:diag-ant-set4-clip}
    \end{subfigure}
    \hfill
    \begin{subfigure}[t]{0.49\linewidth}
        \includegraphics[width=\linewidth]{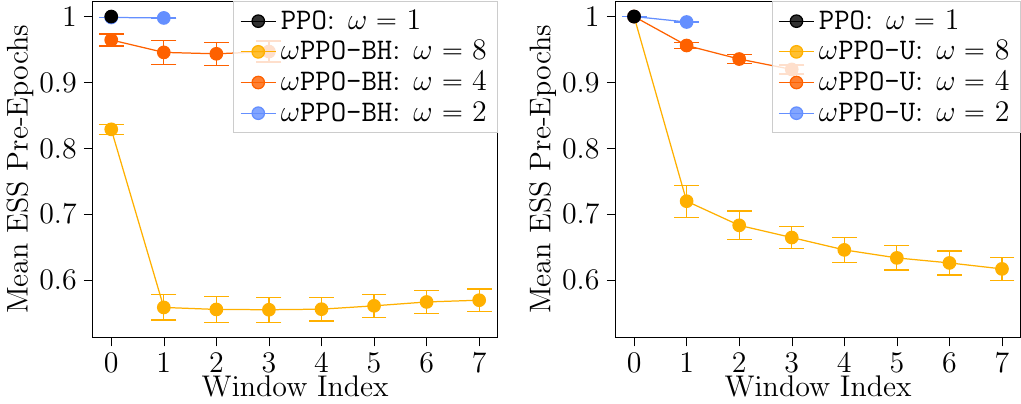}
        \caption{ESS diagnostic: Fixed $H/\omega$ Steps.}
        \label{fig:diag-ant-set4-ess}
    \end{subfigure}

    \caption{Clip Fraction and ESS diagnostics in \emph{Ant-v5} (10 runs, mean $\pm$ 95\% C.I.).}
    \label{fig:diag-ant}
\end{figure}

\subsection{Ablation: Training the Critic on $\cD_{k}$ vs.\ on $\cD(\omega_{k})$} \label{apx:exp-on-off}

In this section, we conduct an ablation on whether to fit the critic on the full composite buffer $\cD(\omega_{k})$ rather than on the most recent dataset $\cD_{k}$ alone, as done in all other experiments and motivated in Section~\ref{sec:sample-reuse}. We consider \emph{Reacher-v5}, \emph{Swimmer-v5}, and \emph{Hopper-v5} across the three comparison settings. Training evaluation curves are reported in Figures~\ref{fig:on-off-reacher}, \ref{fig:on-off-swimmer}, and~\ref{fig:on-off-hopper}, with overall performance in Table~\ref{tab:off-vs-on}. We denote by \uppo Off and \bhppo Off the variants fitting the critic on $\cD(\omega_k)$.

In general, by looking at Table~\ref{tab:off-vs-on} only and if the aim is just to find the best-performing configuration of \uppo and \bhppo (relative to our controlled scenarios), then one could treat the fitting of the critic of the whole buffer $\cD(\omega_{k})$ as an additional hyperparameter to explore. However, results appear to be tied to the specific starting hyperparameter configuration, considered control task, and considered comparison setting. For instance, in \emph{Reacher-v5} fitting the critic on $\cD(\omega_{k})$ delivers the best performance across all the comparison settings, as well as in \emph{Hopper-v5}, but for the sole \emph{fixed $N_{\cB}$} scenario. Otherwise, results may vary and often there is not a clear distinction in achieved final performance and in many cases, as $\omega$ enlarges, final and training performances tends to degrade when fitting the critic on $\cD(\omega_{k})$. Next, we comment the comparison settings separately, considering to group them by $\omega$.

In the \emph{fixed $B$} setting (Figures~\ref{fig:on-off-reacher-base}, \ref{fig:on-off-swimmer-base}, and~\ref{fig:on-off-hopper-base}), results are nearly identical at $\omega=2$, but as $\omega$ increases, fitting on $\cD(\omega_{k})$ either worsens performance or leaves it unchanged across all environments. The only exception is \emph{Reacher-v5}, where past data yields a marginal improvement in final performance, particularly for \bhppo, though the gain is negligible relative to fitting on $\cD_{k}$ only.

In the \emph{fixed $N_{\cB}$} setting (Figures~\ref{fig:on-off-reacher-batch}, \ref{fig:on-off-swimmer-batch}, and~\ref{fig:on-off-hopper-batch}), performance in \emph{Reacher-v5} is comparable across both strategies for all window sizes, while in \emph{Swimmer-v5} fitting on the full $\cD(\omega_{k})$ appears to hinder training, particularly at $\omega \in \{4,8\}$. 
A notable exception is \emph{Hopper-v5}, where fitting on $\cD(\omega_k)$ improves final performance for both variants and all window sizes.

In the \emph{$H/\omega$ steps} setting (Figures~\ref{fig:on-off-reacher-data}, \ref{fig:on-off-swimmer-data}, and~\ref{fig:on-off-hopper-data}), the choice of critic fitting strategy has no discernible effect across all window sizes, as all methods perform poorly regardless.

Overall, fitting the critic on, as $\omega$ enlarges, $\cD(\omega_{k})$ either shows no clear advantage or actively worsens learning, with the exceptions of \emph{Reacher-v5} (all settings) and \emph{Hopper-v5} in the \emph{fixed $N_{\cB}$} setting only. As discussed in Section~\ref{sec:sample-reuse}, past windows in $\cD(\omega_{k})$ contain return estimates computed under the critic available at collection time; without re-evaluation under the current critic~\citep[e.g.,][]{queeney2021generalized,meng2023off,zhao2023policy}, stale returns may introduce a bias accumulating with the window size. This ablation confirms that such bias may be harmful, though its magnitude appears to depend on both the environment and the comparison setting and, as previously said, when the goal is just to find the best reuse configuration, then fitting the critic on the whole $\cD(\omega_k)$ could be considered as an additional exploring dimension. However, for the purpose of this work and given this ablation, we decided to conduct the controlled study on reusing fitting the critic on $\cD_k$ only.

\begin{figure}[t]
    \centering

    \begin{subfigure}{\linewidth}
        \centering
        \begin{minipage}[t]{0.32\linewidth}
            \includegraphics[width=\linewidth]{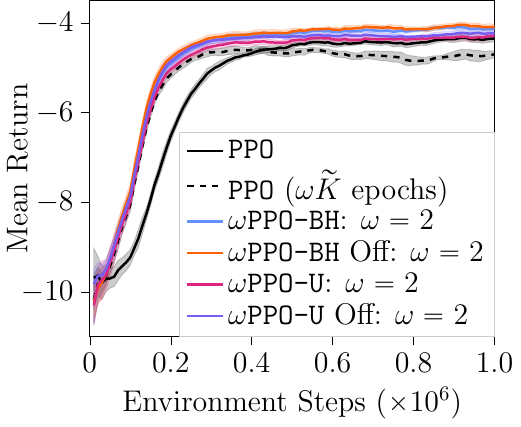}
        \end{minipage}
        \hfill
        \begin{minipage}[t]{0.32\linewidth}
            \includegraphics[width=\linewidth]{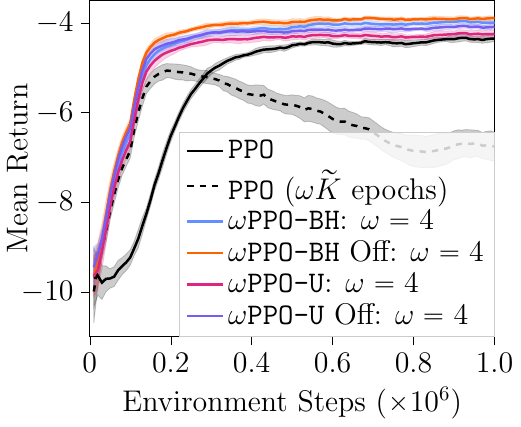}
        \end{minipage}
        \hfill
        \begin{minipage}[t]{0.32\linewidth}
            \includegraphics[width=\linewidth]{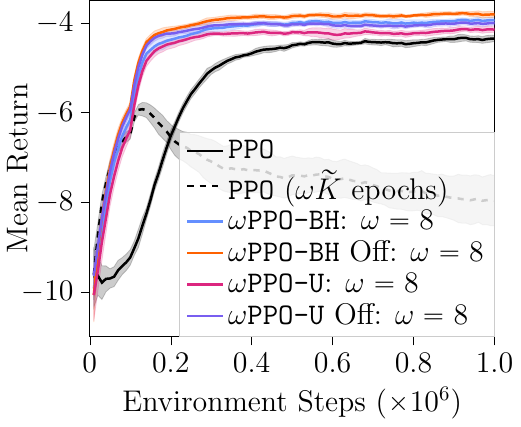}
        \end{minipage}
        \caption{\emph{Fixed $B$} setting.}
        \label{fig:on-off-reacher-base}
    \end{subfigure}

    \vspace{0.3cm}

    \begin{subfigure}{\linewidth}
        \centering
        \begin{minipage}[t]{0.32\linewidth}
            \includegraphics[width=\linewidth]{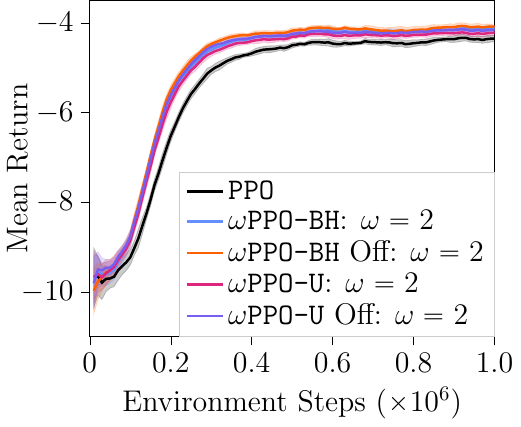}
        \end{minipage}
        \hfill
        \begin{minipage}[t]{0.32\linewidth}
            \includegraphics[width=\linewidth]{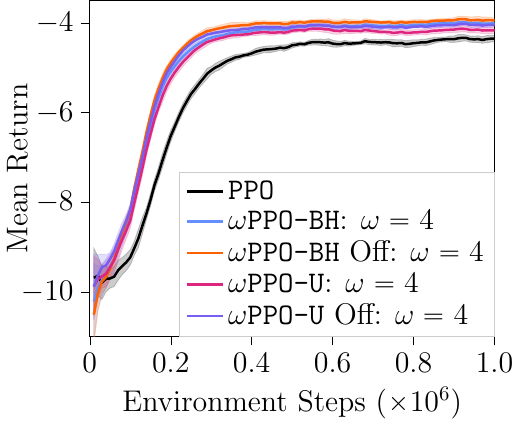}
        \end{minipage}
        \hfill
        \begin{minipage}[t]{0.32\linewidth}
            \includegraphics[width=\linewidth]{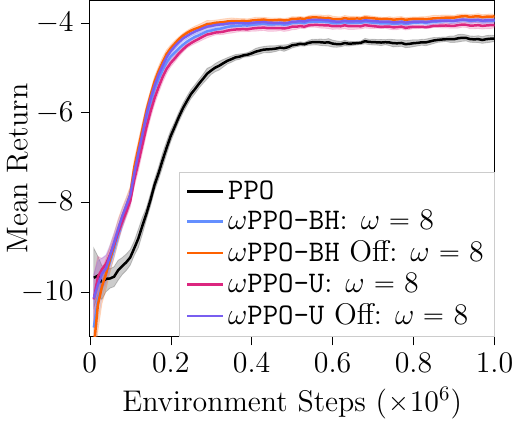}
        \end{minipage}
        \caption{\emph{Fixed $N_{\cB}$} setting.}
        \label{fig:on-off-reacher-batch}
    \end{subfigure}

    \vspace{0.3cm}

    \begin{subfigure}{\linewidth}
        \centering
        \begin{minipage}[t]{0.32\linewidth}
            \includegraphics[width=\linewidth]{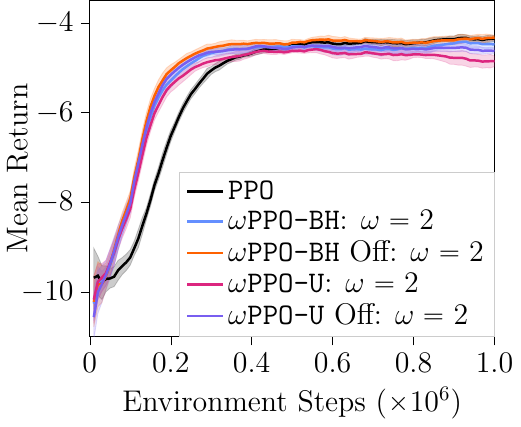}
        \end{minipage}
        \hfill
        \begin{minipage}[t]{0.32\linewidth}
            \includegraphics[width=\linewidth]{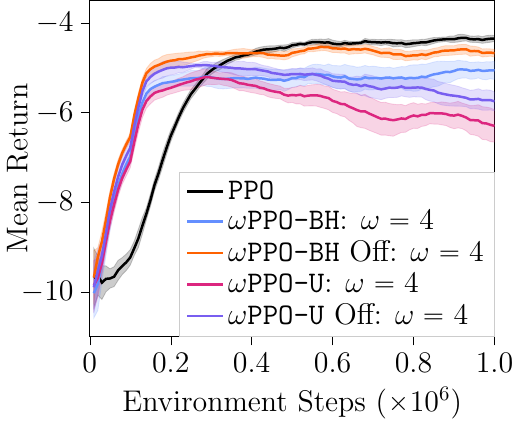}
        \end{minipage}
        \hfill
        \begin{minipage}[t]{0.32\linewidth}
            \includegraphics[width=\linewidth]{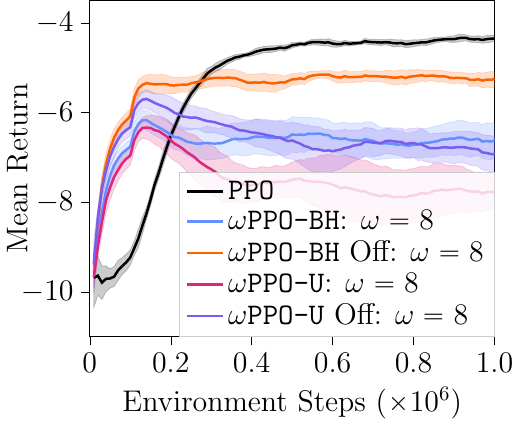}
        \end{minipage}
        \caption{\emph{$H/\omega$ steps} setting.}
        \label{fig:on-off-reacher-data}
    \end{subfigure}

    \caption{Ablation on training the critic on on-policy data only vs.\ on off-policy data too. Evaluation conducted in the \emph{Reacher-v5} environment, considering all the presented comparison settings (10 runs, mean $\pm$ 95\%).}
    \label{fig:on-off-reacher}
\end{figure}

\begin{figure}[t]
    \centering

    \begin{subfigure}{\linewidth}
        \centering
        \begin{minipage}[t]{0.32\linewidth}
            \includegraphics[width=\linewidth]{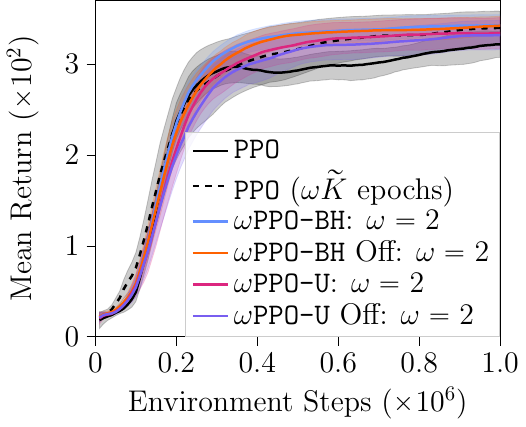}
        \end{minipage}
        \hfill
        \begin{minipage}[t]{0.32\linewidth}
            \includegraphics[width=\linewidth]{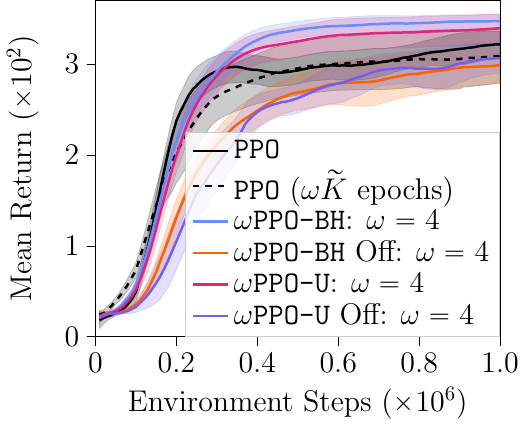}
        \end{minipage}
        \hfill
        \begin{minipage}[t]{0.32\linewidth}
            \includegraphics[width=\linewidth]{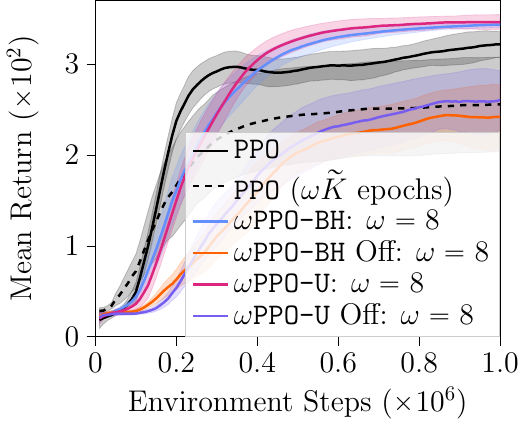}
        \end{minipage}
        \caption{\emph{Fixed $B$} setting.}
        \label{fig:on-off-swimmer-base}
    \end{subfigure}

    \vspace{0.3cm}

    \begin{subfigure}{\linewidth}
        \centering
        \begin{minipage}[t]{0.32\linewidth}
            \includegraphics[width=\linewidth]{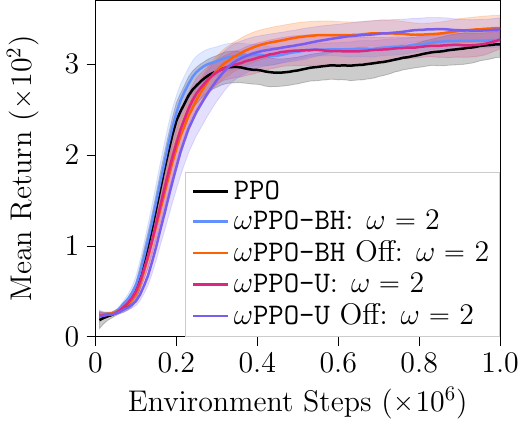}
        \end{minipage}
        \hfill
        \begin{minipage}[t]{0.32\linewidth}
            \includegraphics[width=\linewidth]{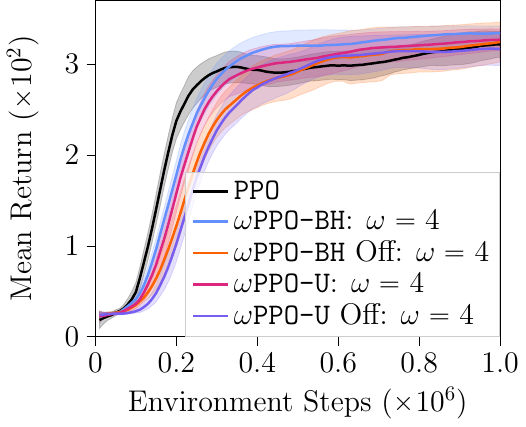}
        \end{minipage}
        \hfill
        \begin{minipage}[t]{0.32\linewidth}
            \includegraphics[width=\linewidth]{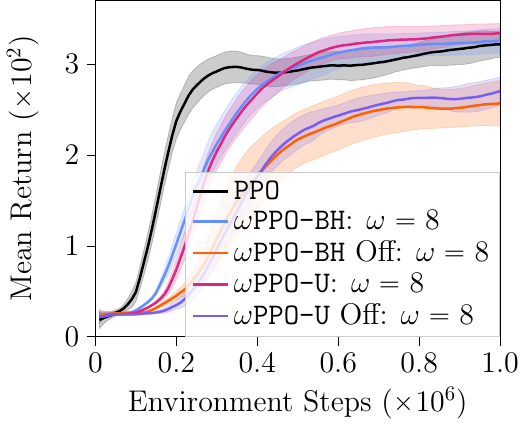}
        \end{minipage}
        \caption{\emph{Fixed $N_{\cB}$} setting.}
        \label{fig:on-off-swimmer-batch}
    \end{subfigure}

    \vspace{0.3cm}

    \begin{subfigure}{\linewidth}
        \centering
        \begin{minipage}[t]{0.32\linewidth}
            \includegraphics[width=\linewidth]{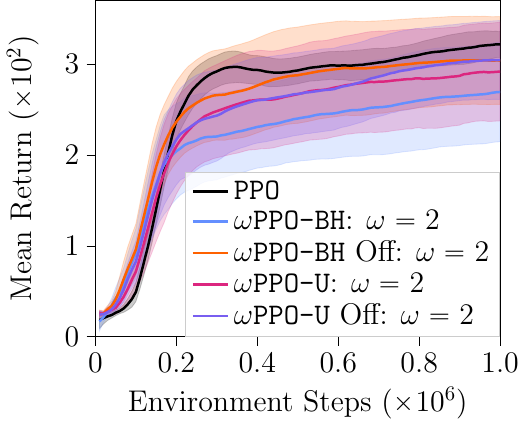}
        \end{minipage}
        \hfill
        \begin{minipage}[t]{0.32\linewidth}
            \includegraphics[width=\linewidth]{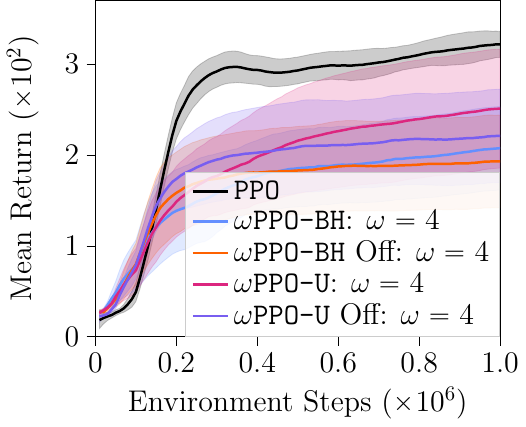}
        \end{minipage}
        \hfill
        \begin{minipage}[t]{0.32\linewidth}
            \includegraphics[width=\linewidth]{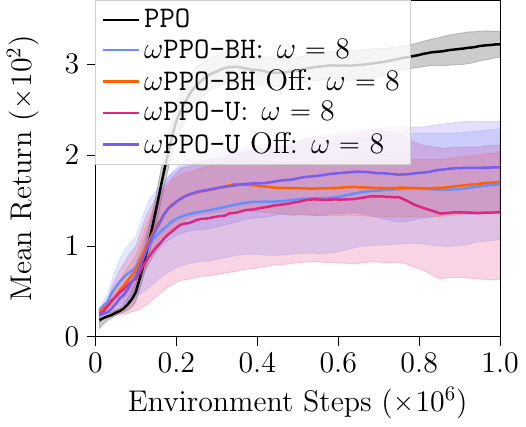}
        \end{minipage}
        \caption{\emph{$H/\omega$ steps} setting.}
        \label{fig:on-off-swimmer-data}
    \end{subfigure}

    \caption{Ablation on training the critic on on-policy data only vs.\ on off-policy data too. Evaluation conducted in the \emph{Swimmer-v5} environment, considering all the presented comparison settings (10 runs, mean $\pm$ 95\%).}
    \label{fig:on-off-swimmer}
\end{figure}

\begin{figure}[t]
    \centering

    \begin{subfigure}{\linewidth}
        \centering
        \begin{minipage}[t]{0.32\linewidth}
            \includegraphics[width=\linewidth]{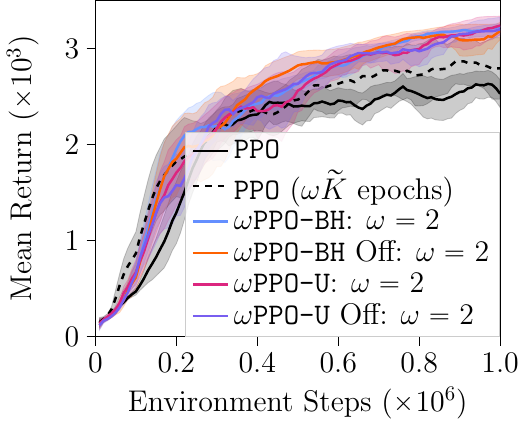}
        \end{minipage}
        \hfill
        \begin{minipage}[t]{0.32\linewidth}
            \includegraphics[width=\linewidth]{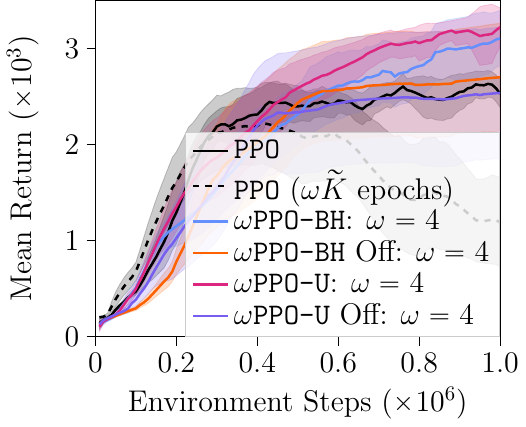}
        \end{minipage}
        \hfill
        \begin{minipage}[t]{0.32\linewidth}
            \includegraphics[width=\linewidth]{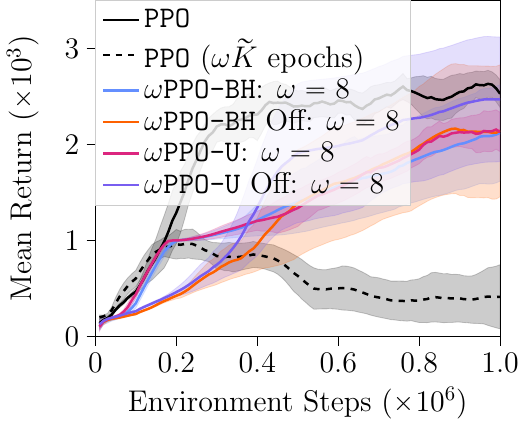}
        \end{minipage}
        \caption{\emph{Fixed $B$} setting.}
        \label{fig:on-off-hopper-base}
    \end{subfigure}
    
    \vspace{0.3cm}

    \begin{subfigure}{\linewidth}
        \centering
        \begin{minipage}[t]{0.32\linewidth}
            \includegraphics[width=\linewidth]{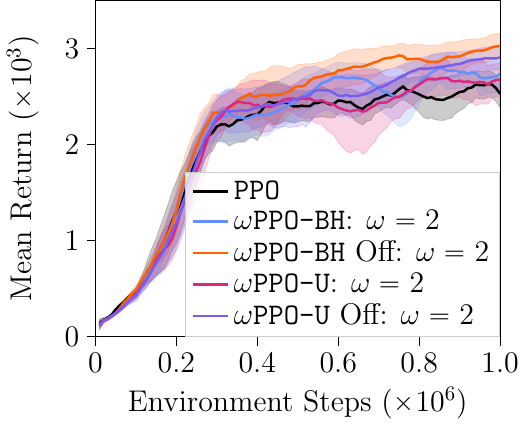}
        \end{minipage}
        \hfill
        \begin{minipage}[t]{0.32\linewidth}
            \includegraphics[width=\linewidth]{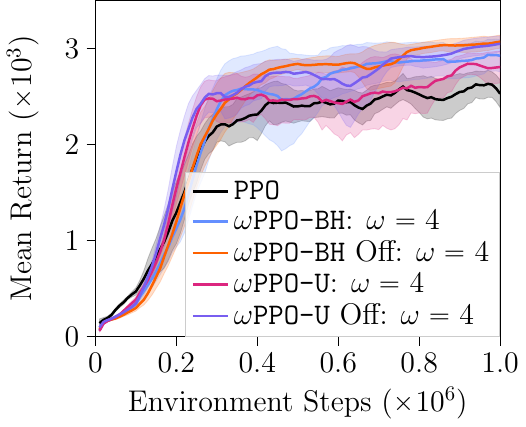}
        \end{minipage}
        \hfill
        \begin{minipage}[t]{0.32\linewidth}
            \includegraphics[width=\linewidth]{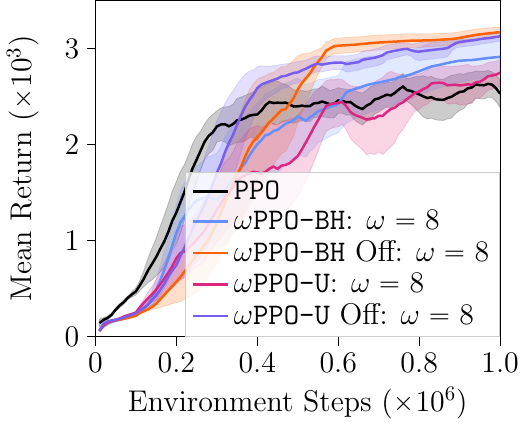}
        \end{minipage}
        \caption{\emph{Fixed $N_{\cB}$} setting.}
        \label{fig:on-off-hopper-batch}
    \end{subfigure}

    \vspace{0.3cm}

    \begin{subfigure}{\linewidth}
        \centering
        \begin{minipage}[t]{0.32\linewidth}
            \includegraphics[width=\linewidth]{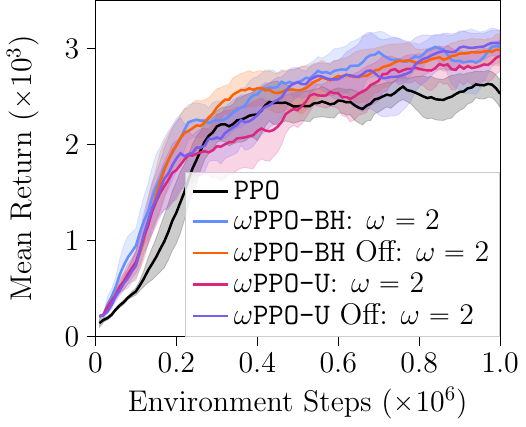}
        \end{minipage}
        \hfill
        \begin{minipage}[t]{0.32\linewidth}
            \includegraphics[width=\linewidth]{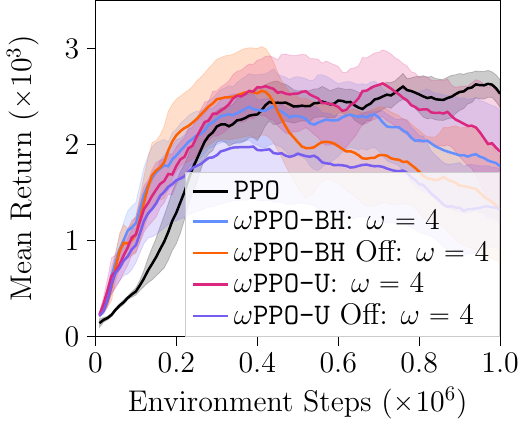}
        \end{minipage}
        \hfill
        \begin{minipage}[t]{0.32\linewidth}
            \includegraphics[width=\linewidth]{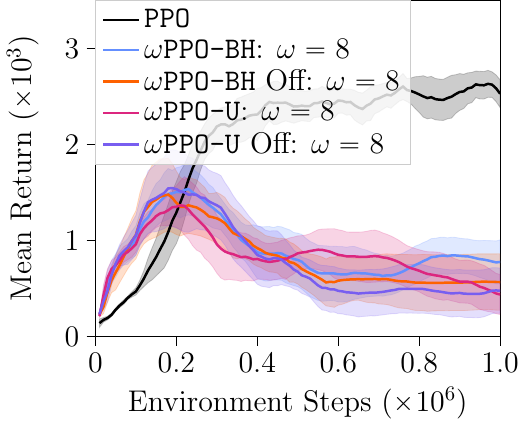}
        \end{minipage}
        \caption{\emph{$H/\omega$ steps} setting.}
        \label{fig:on-off-hopper-data}
    \end{subfigure}

    \caption{Ablation on training the critic on on-policy data only vs.\ on off-policy data too. Evaluation conducted in the \emph{Hopper-v5} environment, considering all the presented comparison settings (10 runs, mean $\pm$ 95\%).}
    \label{fig:on-off-hopper}
\end{figure}

\begin{table}[t]
    \centering
    \resizebox{\linewidth}{!}{\input{tables/off-vs-on}}

    \vspace{0.1cm}
    
    \caption{Average of the last 10 final performance values of the curves shown 
    in the \emph{critic fitting ablation} 
    (Figures~\ref{fig:on-off-reacher}, \ref{fig:on-off-swimmer}, 
    and~\ref{fig:on-off-hopper}). 
    Each cell shows the mean $\pm 95\%$ C.I. edges (10 runs).}
    \label{tab:off-vs-on}  
\end{table}

%% file: tables/set1-perf.tex
{\renewcommand{\arraystretch}{1.2}
\begin{tabular}{ll cccccc}
    \toprule
    \textbf{Algorithm} & &
        \emph{Reacher-v5} & \emph{Walker2d-v5} & \emph{Hopper-v5} &
        \emph{HalfCheetah-v5} & \emph{Swimmer-v5} & \emph{Ant-v5} \\
    \midrule
    \ppo & &
        $-4.35 \pm 0.063$ & $2430 \pm 443$ & $2530 \pm 137$ &
        $4330 \pm 725$ & $322 \pm 13.6$ & $\mathbf{3640 \pm 204}$ \\
    \midrule
    \multirow{3}{*}{\ppo ($\omega \widetilde{K}$ epochs)}
        & $2\times$ &
            $-4.70 \pm 0.069$ & $2510 \pm 556$ & $2790 \pm 305$ &
            $4500 \pm 1120$ & $\mathbf{340 \pm 17.5}$ & $\mathbf{3610 \pm 203}$ \\
        & $4\times$ &
            $-6.76 \pm 0.308$ & $682 \pm 138$ & $1190 \pm 442$ &
            $4840 \pm 861$ & $309 \pm 28.3$ & $2740 \pm 266$ \\
        & $8\times$ &
            $-7.97 \pm 0.531$ & $392 \pm 32.4$ & $415 \pm 317$ &
            $3940 \pm 810$ & $256 \pm 49.3$ & $1890 \pm 236$ \\
    \midrule
    \multirow{3}{*}{\uppo}
        & $\omega=2$ &
            $-4.29 \pm 0.065$ &
            $\mathbf{3350 \pm 556}$ &
            \cellcolor{ibmMagenta!20}$\mathbf{3240 \pm 94.3}$ &
            $\mathbf{5560 \pm 335}$ &
            $\mathbf{334 \pm 16.2}$ &
            $\mathbf{3650 \pm 227}$ \\
        & $\omega=4$ &
            $-4.25 \pm 0.075$ &
            \cellcolor{ibmMagenta!20}$\mathbf{3740 \pm 546}$ &
            $\mathbf{3220 \pm 192}$ &
            $\mathbf{5960 \pm 729}$ &
            $\mathbf{339 \pm 15.0}$ &
            $\mathbf{3840 \pm 254}$ \\
        & $\omega=8$ &
            $-4.15 \pm 0.057$ &
            $\mathbf{3650 \pm 426}$ &
            $2130 \pm 195$ &
            $\mathbf{6010 \pm 484}$ &
            $\mathbf{346 \pm 7.73}$ &
            $\mathbf{ 3650 \pm 481}$ \\
    \midrule
    \multirow{3}{*}{\bhppo}
        & $\omega=2$ &
            $-4.14 \pm 0.047$ &
            $\mathbf{3360 \pm 491}$ &
            $\mathbf{3220 \pm 94.6}$ &
            $\mathbf{5530 \pm 550}$ &
            $\mathbf{343 \pm 11.2}$ &
            $\mathbf{3800 \pm 189}$ \\
        & $\omega=4$ &
            $\mathbf{-3.99 \pm 0.045}$ &
            $\mathbf{3230 \pm 735}$ &
            $\mathbf{3110 \pm 274}$ &
            \cellcolor{ibmUltramarine!20}$\mathbf{6170 \pm 444}$ &
            \cellcolor{ibmUltramarine!20}$\mathbf{347 \pm 6.76}$ &
            \cellcolor{ibmUltramarine!20}$\mathbf{3930 \pm 210}$ \\
        & $\omega=8$ &
            \cellcolor{ibmUltramarine!20}$\mathbf{-3.94 \pm 0.052}$ &
            $\mathbf{2660 \pm 815}$ &
            $2120 \pm 476$ &
            $\mathbf{5720 \pm 586}$ &
            $\mathbf{343 \pm 3.17}$ &
            $\mathbf{3760 \pm 576}$ \\
    \bottomrule
\end{tabular}}

%% file: tables/set3-perf.tex
{\renewcommand{\arraystretch}{1.2}
\begin{tabular}{ll cccccc}
\toprule
\textbf{Algorithm} & &
  \emph{Reacher-v5} &
  \emph{Walker2d-v5} &
  \emph{Hopper-v5} &
  \emph{HalfCheetah-v5} &
  \emph{Swimmer-v5} &
  \emph{Ant-v5} \\
\midrule
%
\ppo & &
  $-4.35 \pm 0.063$ &
  $2430 \pm 443$ &
  $2530 \pm 137$ &
  $4330 \pm 725$ &
  $\mathbf{322 \pm 13.6}$ &       
  $\mathbf{3640 \pm 204}$ \\      
\midrule
%
\multirow{3}{*}{\uppo}
  & $\omega=2$ &
    $-4.22 \pm 0.060$ &             
    $\mathbf{3340 \pm 535}$ &
    $\mathbf{2680 \pm 163}$ &
    $\mathbf{5160 \pm 515}$ &
    $\mathbf{327 \pm 9.56}$ &
    $\mathbf{3910 \pm 130}$ \\
  & $\omega=4$ &
    $-4.17 \pm 0.063$ &             
    $\mathbf{2820 \pm 485}$ &
    $\mathbf{2810 \pm 157}$ &
    $\mathbf{5760 \pm 569}$ &
    $\mathbf{327 \pm 11.2}$ &
    $\mathbf{3820 \pm 491}$ \\
  & $\omega=8$ &
    $\mathbf{-4.05 \pm 0.057}$ &    
    $\mathbf{3030 \pm 430}$ &
    $\mathbf{2750 \pm 131}$ &
    $\mathbf{5360 \pm 384}$ &
    \cellcolor{ibmMagenta!20}$\mathbf{334 \pm 10.5}$ &   
    $\mathbf{3970 \pm 301}$ \\
\midrule
%
\multirow{3}{*}{\bhppo}
  & $\omega=2$ &
    $-4.12 \pm 0.052$ &             
    $\mathbf{3250 \pm 530}$ &
    $\mathbf{2740 \pm 128}$ &
    $\mathbf{4580 \pm 949}$ &
    $\mathbf{327 \pm 7.69}$ &
    $3650 \pm 136$ \\               
  & $\omega=4$ &
    $\mathbf{-4.01 \pm 0.045}$ &    
    \cellcolor{ibmUltramarine!20}$\mathbf{3810 \pm 642}$ &   
    \cellcolor{ibmUltramarine!20}$\mathbf{2920 \pm 159}$ &   
    \cellcolor{ibmUltramarine!20}$\mathbf{5780 \pm 558}$ &   
    \cellcolor{ibmUltramarine!20}$\mathbf{334 \pm 7.51}$ &   
    \cellcolor{ibmUltramarine!20}$\mathbf{4030 \pm 197}$ \\  
  & $\omega=8$ &
    \cellcolor{ibmUltramarine!20}$\mathbf{-3.96 \pm 0.054}$ & 
    $\mathbf{2990 \pm 375}$ &
    $\mathbf{2910 \pm 177}$ &
    $\mathbf{5240 \pm 447}$ &
    $\mathbf{326 \pm 12.6}$ &
    $\mathbf{3960 \pm 357}$ \\
\bottomrule
\end{tabular}}

%% file: tables/set4-perf.tex
%
%
%

{\renewcommand{\arraystretch}{1.2}
\begin{tabular}{ll cccccc}
  \toprule
  \textbf{Algorithm} & &
    \emph{Reacher-v5} &
    \emph{Walker2d-v5} &
    \emph{Hopper-v5} &
    \emph{HalfCheetah-v5} &
    \emph{Swimmer-v5} &
    \emph{Ant-v5} \\
  \midrule
  \ppo & &
    \cellcolor{gray!25}$\mathbf{-4.35 \pm 0.063}$ &
    \cellcolor{gray!25}$\mathbf{2430 \pm 443}$ &
    $2530 \pm 137$ &
    $\mathbf{4330 \pm 725}$ &
    \cellcolor{gray!25}$\mathbf{322 \pm 13.6}$ &
    \cellcolor{gray!25}$\mathbf{3640 \pm 204}$ \\
  \midrule
  \multirow{3}{*}{\uppo}
    & $\omega=2$ &
      $-4.86 \pm 0.113$ &
      $\mathbf{2300 \pm 629}$ &
      $\mathbf{2920 \pm 154}$ &
      $\mathbf{3960 \pm 994}$ &
      $\mathbf{292 \pm 51.6}$ &
      $\mathbf{3520 \pm 222}$ \\
    & $\omega=4$ &
      $-6.30 \pm 0.342$ &
      $1420 \pm 300$ &
      $1920 \pm 553$ &
      $3830 \pm 373$ &
      $\mathbf{251 \pm 62.1}$ &
      $2670 \pm 262$ \\
    & $\omega=8$ &
      $-7.78 \pm 0.366$ &
      $686 \pm 157$ &
      $438 \pm 198$ &
      $2710 \pm 750$ &
      $137 \pm 70.1$ &
      $1600 \pm 216$ \\
  \midrule
  \multirow{3}{*}{\bhppo}
    & $\omega=2$ &
      $\mathbf{-4.48 \pm 0.084}$ &
      $\mathbf{2270 \pm 463}$ &
      \cellcolor{ibmUltramarine!20}$\mathbf{3020 \pm 163}$ &
      \cellcolor{ibmUltramarine!20}$\mathbf{5040 \pm 711}$ &
      $\mathbf{269 \pm 52.0}$ &
      $\mathbf{3500 \pm 116}$ \\
    & $\omega=4$ &
      $-5.06 \pm 0.191$ &
      $1160 \pm 286$ &
      $1770 \pm 555$ &
      $\mathbf{4070 \pm 1080}$ &
      $207 \pm 43.5$ &
      $2150 \pm 459$ \\
    & $\omega=8$ &
      $-6.63 \pm 0.366$ &
      $778 \pm 126$ &
      $777 \pm 218$ &
      $\mathbf{3440 \pm 976}$ &
      $168 \pm 57.9$ &
      $865 \pm 193$ \\
  \bottomrule
\end{tabular}}

%% file: tables/off-vs-on.tex

{\renewcommand{\arraystretch}{1.2}
\begin{tabular}{ll ccc ccc ccc}
\toprule
& & \multicolumn{3}{c}{\emph{Reacher}}
  & \multicolumn{3}{c}{\emph{Hopper}}
  & \multicolumn{3}{c}{\emph{Swimmer}} \\
\cmidrule(lr){3-5} \cmidrule(lr){6-8} \cmidrule(lr){9-11}
\textbf{Algorithm} & &
  \textbf{Fixed $B$} & \textbf{Fixed $N_{\mathcal{B}}$} & \textbf{$H/\omega$ Steps} &
  \textbf{Fixed $B$} & \textbf{Fixed $N_{\mathcal{B}}$} & \textbf{$H/\omega$ Steps} &
  \textbf{Fixed $B$} & \textbf{Fixed $N_{\mathcal{B}}$} & \textbf{$H/\omega$ Steps} \\
\midrule
\ppo & &
  $-4.35 \pm 0.063$ & $-4.35 \pm 0.063$ & $\mathbf{-4.35 \pm 0.063}$ &
  $2530 \pm 137$    & $2530 \pm 137$    & $2530 \pm 137$    &
  $322 \pm 13.6$    & $\mathbf{322 \pm 13.6}$ & \cellcolor{ibmGrey!20}$\mathbf{322 \pm 13.6}$ \\
\midrule
\multirow{3}{*}{\ppo ($\omega\widetilde{K}$ epochs)}
& $2\times$
  & $-4.70 \pm 0.069$ & --- & ---
  & $2790 \pm 305$    & --- & ---
  & $\mathbf{340 \pm 17.5}$ & --- & --- \\
& $4\times$
  & $-6.76 \pm 0.308$ & --- & ---
  & $1190 \pm 442$    & --- & ---
  & $309 \pm 28.3$ & --- & --- \\
& $8\times$
  & $-7.97 \pm 0.531$ & --- & ---
  & $415 \pm 317$     & --- & ---
  & $256 \pm 49.3$ & --- & --- \\
\midrule
\multirow{3}{*}{\uppo}
& $\omega=2$
  & $-4.29 \pm 0.065$ & $-4.22 \pm 0.060$ & $-4.86 \pm 0.113$
  & \cellcolor{ibmMagenta!20}$\mathbf{3240 \pm 94.3}$ & $2680 \pm 163$ & $\mathbf{2920 \pm 154}$
  & $\mathbf{334 \pm 16.2}$ & $\mathbf{327 \pm 9.56}$ & $\mathbf{292 \pm 51.6}$ \\
& $\omega=4$
  & $-4.25 \pm 0.075$ & $-4.17 \pm 0.063$ & $-6.30 \pm 0.342$
  & $\mathbf{3220 \pm 192}$ & $2810 \pm 157$ & $1920 \pm 553$
  & $\mathbf{339 \pm 15.0}$ & $\mathbf{327 \pm 11.2}$ & $\mathbf{251 \pm 62.1}$ \\
& $\omega=8$
  & $-4.15 \pm 0.057$ & $-4.05 \pm 0.057$ & $-7.78 \pm 0.366$
  & $2130 \pm 195$ & $2750 \pm 131$ & $438 \pm 198$
  & $\mathbf{346 \pm 7.73}$ & $\mathbf{334 \pm 10.5}$ & $137 \pm 70.1$ \\
\midrule
\multirow{3}{*}{\uppo~Off}
& $\omega=2$
  & $-4.23 \pm 0.070$ & $-4.16 \pm 0.057$ & $-4.63 \pm 0.064$
  & $\mathbf{3200 \pm 52.5}$ & $2910 \pm 113$ & \cellcolor{ibmIndigo!20}$\mathbf{3050 \pm 129}$
  & $\mathbf{332 \pm 15.3}$ & $\mathbf{338 \pm 12.1}$ & $\mathbf{304 \pm 41.2}$ \\
& $\omega=4$
  & $-4.09 \pm 0.045$ & $-4.04 \pm 0.040$ & $-5.75 \pm 0.201$
  & $\mathbf{2540 \pm 649}$ & $\mathbf{3050 \pm 77.9}$ & $1320 \pm 395$
  & $306 \pm 16.0$ & $\mathbf{317 \pm 17.4}$ & $221 \pm 48.8$ \\
& $\omega=8$
  & $-4.01 \pm 0.046$ & $\mathbf{-3.94 \pm 0.045}$ & $-6.93 \pm 0.380$
  & $2470 \pm 616$ & $\mathbf{3130 \pm 45.0}$ & $478 \pm 223$
  & $260 \pm 30.8$ & $270 \pm 14.8$ & $186 \pm 47.8$ \\
\midrule
\multirow{3}{*}{\bhppo}
& $\omega=2$
  & $-4.14 \pm 0.047$ & $-4.12 \pm 0.052$ & $-4.48 \pm 0.084$
  & $\mathbf{3220 \pm 94.6}$ & $2740 \pm 128$ & $\mathbf{3020 \pm 163}$
  & $\mathbf{343 \pm 11.2}$ & $\mathbf{327 \pm 7.69}$ & $\mathbf{269 \pm 52.0}$ \\
& $\omega=4$
  & $-3.99 \pm 0.045$ & $-4.01 \pm 0.045$ & $-5.06 \pm 0.191$
  & $\mathbf{3110 \pm 274}$ & $2920 \pm 159$ & $1770 \pm 555$
  & \cellcolor{ibmUltramarine!20}$\mathbf{347 \pm 6.76}$ & $\mathbf{334 \pm 7.51}$ & $207 \pm 43.5$ \\
& $\omega=8$
  & $-3.94 \pm 0.052$ & $-3.96 \pm 0.054$ & $-6.63 \pm 0.366$
  & $2120 \pm 476$ & $2910 \pm 177$ & $777 \pm 218$
  & $\mathbf{343 \pm 3.17}$ & $\mathbf{326 \pm 12.6}$ & $168 \pm 57.9$ \\
\midrule
\multirow{3}{*}{\bhppo~Off}
& $\omega=2$
  & $-4.10 \pm 0.051$ & $-4.09 \pm 0.034$ & \cellcolor{ibmOrange!20}$\mathbf{-4.33 \pm 0.049}$
  & $\mathbf{3180 \pm 85.9}$ & $\mathbf{3020 \pm 126}$ & $\mathbf{2980 \pm 159}$
  & $\mathbf{342 \pm 9.95}$ & \cellcolor{ibmOrange!20}$\mathbf{340 \pm 13.4}$ & $\mathbf{304 \pm 45.9}$ \\
& $\omega=4$
  & $\mathbf{-3.90 \pm 0.036}$ & $\mathbf{-3.94 \pm 0.043}$ & $-4.68 \pm 0.082$
  & $\mathbf{2700 \pm 532}$ & $\mathbf{3070 \pm 64.8}$ & $1330 \pm 521$
  & $299 \pm 18.3$ & $\mathbf{324 \pm 20.7}$ & $193 \pm 48.3$ \\
& $\omega=8$
  & \cellcolor{ibmOrange!20}$\mathbf{-3.81 \pm 0.056}$ & \cellcolor{ibmOrange!20}$\mathbf{-3.86 \pm 0.044}$ & $-5.26 \pm 0.162$
  & $2140 \pm 652$ & \cellcolor{ibmOrange!20}$\mathbf{3170 \pm 50.9}$ & $569 \pm 279$
  & $242 \pm 33.2$ & $257 \pm 23.5$ & $170 \pm 30.9$ \\
\bottomrule
\end{tabular}}